\documentclass[nohyperref]{article}

\usepackage{microtype}
\usepackage{graphicx}
\usepackage{subfigure}
\usepackage{booktabs} 
\usepackage{enumitem}
\usepackage{hyperref}

\usepackage[accepted]{icml2022}

\usepackage{amsmath}
\usepackage{amssymb}
\usepackage{mathtools}
\usepackage{amsthm}

\usepackage[capitalize,noabbrev]{cleveref}

\theoremstyle{plain}
\newtheorem{theorem}{Theorem}[section]

\newtheorem{lemma}[theorem]{Lemma}
\newtheorem{corollary}[theorem]{Corollary}
\theoremstyle{definition}

\newtheorem{assumption}[theorem]{Assumption}
\theoremstyle{remark}
\newtheorem{remark}[theorem]{Remark}

\usepackage{smile}

\icmltitlerunning{Communication-Efficient Adaptive  Federated Learning}

\def \vb {\mathbf{v}}

\def \Vbh {\hat{\mathbf{V}}}

\usepackage{color}
\usepackage{xcolor}
\usepackage{soul}
\definecolor{carnelian}{rgb}{0.7, 0.11, 0.11}
\newcommand{\ifcomments}{\iftrue}

\begin{document}

\twocolumn[
\icmltitle{Communication-Efficient Adaptive Federated Learning}




\icmlsetsymbol{equal}{*}

\begin{icmlauthorlist}
    \icmlauthor{Yujia Wang}{psu}
    \icmlauthor{Lu Lin}{uva}
    \icmlauthor{Jinghui Chen}{psu}
\end{icmlauthorlist}

\icmlaffiliation{psu}{College of Information Sciences and Technology, Pennsylvania State University, State College, PA, United States}
\icmlaffiliation{uva}{Department of Computer Science, University of Virginia, Charlottesville, VA, United States}

\icmlcorrespondingauthor{Jinghui Chen}{jzc5917@psu.edu}

\icmlkeywords{Machine Learning, ICML}

\vskip 0.3in
]



\printAffiliationsAndNotice{}  


\begin{abstract}
Federated learning is a machine learning training paradigm that enables clients to jointly train models without sharing their own localized data. However, the implementation of federated learning in practice still faces numerous challenges, such as the large communication overhead due to the repetitive server-client synchronization and the lack of adaptivity by SGD-based model updates.  Despite that various methods have been proposed for reducing the communication cost by gradient compression or quantization, and the federated versions of adaptive optimizers such as FedAdam are proposed to add more adaptivity, the current federated learning framework still cannot solve the aforementioned challenges all at once. In this paper, we propose a novel communication-efficient adaptive federated learning method (FedCAMS) with theoretical convergence guarantees. We show that in the nonconvex stochastic optimization setting, our proposed FedCAMS achieves the same convergence rate of $\cO(\frac{1}{\sqrt{TKm}})$ as its non-compressed counterparts. Extensive experiments on various benchmarks verify our theoretical analysis.
\end{abstract}

\section{Introduction}
Federated learning (FL) \citep{konevcny2016federated,mcmahan2017communication} has recently become a popular machine learning training paradigm where multiple clients cooperate to jointly learn a machine learning model. In the federated learning setting, training data is distributed across a large number of clients, or edge devices, such as smartphones, personal computers, or IoT devices. These clients own valuable data for training a variety of machine learning models, yet those raw client data is not allowed to share with the server or other clients due to privacy and regulation concerns. 
Federated Learning \citep{konevcny2016federated,mcmahan2017communication} works by having each client train the ML model locally based on its own data, while having the clients iteratively exchanging and synchronizing their local ML model parameters with each other through a central server.
\citet{mcmahan2017communication} proposed FedAvg algorithm, whose global model is updated by averaging multiple steps of local stochastic gradient descent (SGD) updates, and it has become one of the most popular FL methods. 

Despite the ability to jointly train the model without directly sharing the data, the implementation of FL in practice still faces several major challenges such as (1) \textit{large communication overhead} due to the repetitive synchronization between the server and the clients; and (2) \textit{lack of adaptivity} as SGD-based update may not be suitable for heavy-tail stochastic gradient noise distributions, which often arise in training large-scale models such as BERT \citep{devlin2018bert}, GPT-3 \citep{brown2020language}, GAN \citep{goodfellow2014generative} or ViT \citep{dosovitskiy2021an}.

Note that various attempts have been made to solve the aforementioned challenges individually but not all of them at once. In terms of reducing communication costs, one can avoid transmitting the complete model updates when synchronizing. Several works, including \citep{reisizadeh2020fedpaq, jin2020stochastic, jhunjhunwala2021adaptive, chen2021communication}, have studied the compressed and quantized federated learning optimization method based on FedAvg. Another way is to reduce the number of participating clients such that only part of the clients participate in the model training at each round \citep{yang2021achieving,li2019convergence,nishio2019client, li2019fedmd}. Besides, the network resources allocation also plays an important role in communication-efficient federated learning problems \citep{li2019fair,yang2020energy}. 
For adaptivity concerns, recently, FedAdam \citep{reddi2020adaptive} and other variants, such as FedYogi \citep{reddi2020adaptive} and FedAMSGrad \citep{tong2020effective} were proposed to introduce adaptive gradient methods \citep{kingma2014adam,j.2018on} into federated learning framework and provided provable convergence guarantees. However, it is still an open problem how to achieve communication efficient adaptive federated optimization while still providing rigorous convergence guarantees.

In this paper, we aim to develop a new compressed federated adaptive gradient optimization method that is communication-efficient while also provably convergences. Specifically, we first propose FedAMS, a variant of FedAdam, with an improved convergence analysis over the original FedAdam. Based on FedAMS, we propose FedCAMS, a \textbf{Fed}erated \textbf{C}ommunication-compressed \textbf{A}MSGrad with \textbf{M}ax  \textbf{S}tablization (FedCAMS), which addresses both the communication and adaptivity challenges within one training framework.
We summarize our contributions as follows:
\vspace{-10pt}
\begin{itemize}[itemsep=0ex,leftmargin=1em]
    \item We provide an improved analysis on the convergence behaviour of FedAMS, a variant of the existing federated adaptive gradient method FedAdam \citep{reddi2020adaptive}, whose analysis is simplified for only considering the case where no momentum is been used. In particular, we prove that FedAMS (with momentum) can achieve the same convergence rate of $\cO(\frac{1}{\sqrt{TKm}})$ w.r.t total iterations $T$, the number of local updates $K$, and the number of workers $m$ for both full participation and partial participation schemes.
    \item We propose a new communication-efficient adaptive federated optimization method, FedCAMS, which to the best of our knowledge, for the first time, achieves both communication efficiency and adaptivity in federated learning with one single learning framework. FedCAMS largely reduces the communication cost by error feedback and compression strategy and it is compatible with various commonly-used compressors in practice. We prove that FedCAMS achieves the same convergence rate of $\cO(\frac{1}{\sqrt{TKm}})$, as its uncompressed counterpart FedAMS. 
    \item We conduct experiments on various benchmarks and show that our proposed FedAMS and FedCAMS achieve good adaptivity in training real-world machine learning models. Furthermore, we show that FedCAMS effectively reduced the communication cost (number of bits for communication) by orders of magnitude while sacrificing little in terms of prediction accuracy.
\end{itemize}
\vspace{-5pt}
\textbf{Notation:} For vectors $\xb,\yb \in \RR^d$, $\sqrt{\xb},\xb^2, \xb/\yb$ denote the element-wise square root, square, and division of the vectors. For vector $\xb$ and matrix $A$, $\|\cdot\|$ denotes the $\ell_2$ norm of vector/matrix, i.e., $\|\xb\| = \|\xb\|_2$ and $\|A\| = \|A\|_2$. 

\section{Related Work}

\noindent\textbf{SGD and Adaptive Gradient Methods:} Stochastic gradient descent (SGD) \citep{robbins1951stochastic} has been widely applied in training machine learning models for decades. Although SGD is straightforward to implement, it is known to be sensitive to parameters and relatively slow to converge when facing heavy-tail stochastic gradient noise. Adaptive gradient methods were proposed to overcome these issues of SGD, including AdaGrad \citep{duchi2011adaptive}, RMSProp \citep{tieleman2012lecture}, AdaDelta \citep{zeiler2012adadelta}. Adam \citep{kingma2014adam} and its variant AMSGrad \citep{j.2018on}, are tremendously used in training deep neural networks, and other variants \citep{luo2019adaptive,loshchilov2017decoupled,chen2020closing} also play important roles in improving adaptive gradient methods through different aspects.    
\\
\textbf{Federated Learning: }As the demand of locally data storing and training models at edge devices, Federated Learning \citep{konevcny2016federated,li2020federated} rapidly attracts growing interest in recent years. 
Federated Averaging method (FedAvg) \citep{mcmahan2017communication} works by periodically averaging local SGD updates. \citet{stich2018local} provided a concise theoretical convergence guarantee for local SGD. \citet{lin2018don} proposed a variant of local SGD with empirical improvements. There are many works based on FedAvg such as FedProx \citep{li2020federated}, FedNova \citep{wang2020tackling}, SCAFFOLD \citep{karimireddy2020scaffold}, and other work discussed the variants of FedAvg
\citep{yang2021achieving, li2019convergence, hsu2019measuring, wang2019slowmo}. \citet{reddi2020adaptive} recently proposed several adaptive federated optimization methods including FedAdagrad, FedYogi and FedAdam to overcome the existing convergence issues of FedAvg. \citet{chen2020toward} proposed Local AMSGrad and \citet{tong2020effective} proposed a family of federated adaptive gradient methods with calibrations. 
Another line of research focused on addressing data heterogeneity issues or the network resource allocation issues \citep{ghosh2019robust,li2019fedmd, yang2020energy}.
\\
\textbf{Communication-Compressed Federated Learning:}
Various strategies have been proposed for reducing communication costs in distributed learning for SGD based algorithms \citep{bernstein2018signsgd,seide20141,alistarh2017qsgd,basu2019qsparse,stich2018sparsified,stich2019error,karimireddy2019error} and also adaptive gradient methods \citep{tang20211,wang2022communication}.
In terms of federated learning, many studies have tried to apply the aforementioned methods to FedAvg and have attracted growing interest recently, e.g., FedPAQ \citep{reisizadeh2020fedpaq}, FedCOM \citep{ haddadpour2021federated}, sign SGD in federated learning \citep{jin2020stochastic}, communication-efficient federated learning \citep{ chen2021communication}, AdaQuantFL \citep{jhunjhunwala2021adaptive}. However, there are fewer attempts to develop communication-efficient adaptive gradient methods in federated learning, which is our key focus in this work. 

\section{Proposed Method}
\vspace{-5pt}
In this paper, we aim to study the following federated learning nonconvex optimization problem: 
\vspace{-5pt}
\begin{align}\label{eq:prob_setup}
    \min_{x \in \RR^d} f(\xb)= \frac{1}{m} \sum_{i=1}^m F_i(\xb),
\end{align}
where $m$ is the total amount of local clients, $d$ denotes the dimension of the model parameters, $F_i(\xb) = \EE_{\xi \sim \cD_i} F_i(\xb,\xi_i)$ is the local nonconvex loss function on client $i$ associated with a local distribution $\cD_i$. In the stochastic setting, we can only obtain the unbiased estimator of $F_i(\xb)$, i.e., the stochastic gradient $\gb_t^i = \nabla F_i(\xb,\xi_i)$. In the non i.i.d setting, distributions $\cD_i, \cD_j$ can vary from each other, i.e., $\cD_i\neq \cD_j$, $\forall i \neq j$.

FedAvg \citep{mcmahan2017communication} is a commonly used optimization approach to solve \eqref{eq:prob_setup}. Let $\xb_{t}$ denotes the global model parameters before the $t$-th iteration. Now at iteration $t$, the participating client $i$ from the selected subset $\cS_{t}$ (with size $n$) receives the model $\xb_t$ from the server, conducts $K$ steps of local SGD updates with local learning rate $\eta_l$, obtains the local model $\xb_{t,K}^i$. Client $i$ then sends the model difference $\Delta_t^i = \xb_{t,K}^i - \xb_t$ to the server. And the server updates the global model difference $\Delta_t$ by simply averaging the local model differences $\Delta_t^i$. The server then updates the global model $\xb_{t+1}$ by $\xb_{t+1} = \xb_t + \Delta_t$, which is the same\footnote{The global update of FedAvg is equivalent to perform one step SGD update with the pseudo gradient $\Delta_t$ and learning rate $\eta = 1$.} as directly averaging the local model $\xb_{t,K}^i$, i.e., $\xb_{t+1} = \xb_t + \frac{1}{n}\sum_{i\in \cS_t} (\xb_{t,K}^i - \xb_t) = \frac{1}{n}\sum_{i\in \cS_t} \xb_{t,K}^i$.

FedAdam was then proposed among several adaptive optimization methods in federated learning \citep{reddi2020adaptive}. FedAdam changes the global update rule of FedAvg from one-step SGD to one-step adaptive gradient optimization. Specifically, after gathering local differences $\Delta_t^i$ and averaging to $\Delta_t$, the server updates the global model by Adam optimizer:
\begin{align}
    & \mb_{t} = \beta_1\mb_{t-1} + (1-\beta_1) \Delta_t, \label{eq:adam-m} \\
    & \vb_{t} = \beta_2 \vb_{t-1} + (1-\beta_2) \Delta_t^2, \label{eq:adam-v} \\
    & \xb_{t+1} = \xb_{t} + \eta \frac{\mb_t}{\sqrt{\vb_t}+\epsilon}, \label{eq:adam} 
\end{align}
where $\Delta_t$ acts as pseudo gradient, and the global update can be viewed as one step Adam update using $\Delta_t$. 
Several variants were also proposed will slight changes in the variance term $\vb_t$, such as FedAdagrad and FedYogi \citep{reddi2020adaptive} and FedAMSGrad \citep{tong2020effective}.
Note that the $\epsilon$ in \eqref{eq:adam} is used for numerical stabilization purpose as the $\vb_t$ term can be quite small and cause unstable optimization behaviours.

\subsection{Federated AMSGrad with Max Stabilization}

In this section, we propose a general adaptive federated optimization framework, \textbf{Fed}erated \textbf{A}MSGrad with \textbf{M}ax \textbf{S}tabilization (FedAMS), where the server conducts one additional max stabilization step before the final update. 

Algorithm \ref{alg:fedams} summarize the details of general FedAMS framework. At the beginning of global round $t$, we first select a subset of clients $\cS_t$, each participating client $i \in \cS_t$ obtains the local model $\xb_{t,K}^i$ after $K$ steps of local SGD updates with learning rate $\eta_l$. The model difference $\Delta_t^i$ is the difference between the local updated model $\xb_{t,K}^i$ and the current global model $\xb_t$, i.e., $\Delta_t^i = \xb_{t,K}^i - \xb_t$. The server aggregates $\Delta_t^i$ and gets the global difference $\Delta_t$, this $\Delta_t$ acts as a pseudo gradient to calculate momentum $\mb_t$ and variance $\vb_t$ following \eqref{eq:adam-m} and \eqref{eq:adam-v}. 
Now for updating $\xb_{t+1}$, our general FedAMS framework provides two options for max stabilization:
\begin{align}
    & \text{Option 1: }\hat{\vb}_{t} = \max(\hat{\vb}_{t-1}, \vb_t, \epsilon), \xb_{t+1} = \xb_{t} + \eta \frac{\mb_t}{\sqrt{\hat{\vb}_t}}. \notag\\ 
    & \text{Option 2: }\hat{\vb}_{t} = \max(\hat{\vb}_{t-1}, \vb_t), \xb_{t+1} = \xb_{t} + \eta \frac{\mb_t}{\sqrt{\hat{\vb}_t}+\epsilon}. \notag
\end{align}
Note that Option 2 is the same as the AMSGrad \citep{j.2018on} update rule, which 
brings a non-decreasing $\vb_t$ to solve a non-convergence issue in Adam \citep{kingma2014adam}. 
For Option 1, FedAMS directly adopts $\sqrt{\hat{\vb}_t}$ as the denominator where $\epsilon$ is token as the part of the max operation in $\hat{\vb}_t$. Intuitively, the unstable behaviour of the denominator (small value in the $\vb_t$) usually only happens for a small set of dimensions. Therefore, the max stabilization strategy in Option 1 only affects those dimensions with small $\vb_t$ values, while the traditional adding strategy as in Option 2 will affect the accuracy on all dimensions.

Moreover, we want to emphasize that although the theoretical analysis in \citet{reddi2020adaptive} assumes $\beta_1 = 0$ and only considers the impact of variance $\vb_t$, thus the non-decreasing variance is not necessary for the analysis in \citet{reddi2020adaptive}. While the non-decreasing variance is indeed necessary for us to obtain the complete proof with a positive $\beta_1$ (see Appendix for details).


\begin{algorithm}[t]
\caption{FedAMS}
  \label{alg:fedams}
  \begin{flushleft}
        \textbf{Input:} initial point $\xb_1$, local step size $\eta_l$, global stepsize $\eta$, $\beta_1, \beta_2, \epsilon$.
        \end{flushleft}
  \begin{algorithmic}[1]
        \STATE $\mb_0\gets 0$, $\vb_0\gets 0$ 
      \FOR{$t=1$ to $T$}
        \STATE Random sample a subset $\cS_t$ of clients
        \STATE Server sends $\xb_t$ to the subset $\cS_t$ of clients
        \STATE $\xb_{t,0}^i= \xb_t$
        \FOR{ each client $i \in \cS_t$ in parallel}
            \FOR{$k = 0,...,K-1$}
                \STATE Compute local stochastic gradient: $\gb_{t,k}^i= \nabla F_i(\xb_{t,k}^i; \xi_{t,k}^i)$
                \STATE $\xb_{t,k+1}^i = \xb_{t,k}^i -\eta_l \gb_{t,k}^i$
            \ENDFOR
            \STATE $\Delta_t^i = \xb_{t,K}^i-\xb_t$
        \ENDFOR
        \STATE Server aggregates local update: $\Delta_t = \frac{1}{|\cS_t|} \sum_{i\in\cS_t} \Delta_t^i$
        \STATE Update $\mb_t=\beta_1 \mb_{t-1}+(1-\beta_1)\Delta_t\;$
        \STATE Update $\vb_t=\beta_2 \vb_{t-1}+(1-\beta_2)\Delta_t^2\;$
        \\
        {//Option 1:}\\
        \STATE
        \colorbox{blue!20}{$\hat{\vb}_t=\max(\hat{\vb}_{t-1},\vb_t, \epsilon)$, update $\xb_{t+1}=\xb_t+ \eta \frac{\mb_t}{\sqrt{\hat{\vb}_t}}$}
        {//Option 2:}\\
        \STATE \colorbox{red!20}{ $\hat{\vb}_t = \max(\hat{\vb}_{t-1},\vb_t)$,  update $\xb_{t+1}=\xb_t+ \eta \frac{\mb_t}{\sqrt{\hat{\vb}_t}+\epsilon}$}
      \ENDFOR
  \end{algorithmic}
\end{algorithm}

\subsection{Federated Communication-Compressed AMSGrad}

In order to reduce the communication costs between synchronization, we propose \textbf{Fed}erated \textbf{C}ommunication-compressed \textbf{A}MSGrad with \textbf{M}ax \textbf{S}tabilization (FedCAMS), which is summarized in Algorithm \ref{alg:compfedams}. The main difference lies in that after the client $i$ obtains the model differences $\Delta_t^i$ via local SGD, FedCAMS will compress $\Delta_t^i$ to $\hat{\Delta}_t^i$ via error feedback compression strategy, and then send $\hat{\Delta}_t^i$ to the central server. In details, at round $t$, the client $i$ will apply the compressor on the summation of model differences $\Delta_t^i$ together with the cumulative compression error $\eb_t^i$ to obtain $\hat{\Delta}_t^i$. After that, the client will update term $\eb_{t+1}^i$ by calculating the new cumulative compression error, i.e., $\eb_{t+1}^i = \Delta_t + \eb_t^i - \hat{\Delta}_t^i$, which will be useful for next round's computation. The rest part of FedCAMS is similar to FedAMS: the server aggregates $\hat{\Delta}_t^i$ and obtains $\hat{\Delta}_t$, which will participate in the global update. 

To summarize, FedCAMS is indeed a communication-efficient with the following features.

\textbf{Error Feedback Compression:} Although error-feedback strategy \citep{karimireddy2019error,stich2018sparsified,stich2019error} has been widely used in various distributed learning settings, there is much less use of error-feedback in the federated settings, especially for adaptive federated optimization. Note that combining error-feedback with adaptive federated optimization is not a trivial task at all, instead, it is actually quite complicated.
Specifically, the theoretical analysis of the adaptive gradient method in the typical nonconvex setting relies on the construction of the Lyapunov function. Compared to directly analyzing the model parameter $\xb$, this causes extra difficulty as applying the error feedback strategy on the smoothness-expanded terms from the Lyapunov function will result in an accumulation of the compression error which leads to divergence\footnote{Similar divergence issue has also been discussed in \citet{tang20211,wang2022communication} in the distributed setting.}. In our theoretical analysis, we have to modify the original construction of the Lyapunov function and introduce a new auxiliary sequence about the compression error which eliminates the accumulation of compression error. 
Unlike those direct compression strategies such as simple quantization or direct compression strategies \citep{haddadpour2021federated,reisizadeh2020fedpaq,jin2020stochastic} which usually require an unbiased compressor to work, error feedback allows for various biased compressors such as commonly used scaled sign compressor or top-$k$ compressors. Furthermore, the design of the error feedback strategy is well-known for reducing unnecessary compression error, which leads to a more precise model update. 

\textbf{Support for Partial Participation: }Here we also make error feedback compatible with partial participation settings by keeping the stale cumulative compression error for clients who were not selected for the current round training (see Lines 14-16 in Algorithm \ref{alg:compfedams}). 
Such design makes FedCAMS more practical and communication efficient. Note that the default client sampling strategy in FedCAMS is to randomly select the participating clients (without replacement) in each round, i.e., $p_i = \PP\{i \in \cS_t\} = n/m$.
This can be easily extended to the weighted sampling strategy with probability $p_i = w_i$, even with varying numbers of participating workers $n$.

\begin{algorithm}[h!]
\caption{FedCAMS}
  \label{alg:compfedams}
  \begin{flushleft}
        \textbf{Input:} initial point $\xb_1$, local step size $\eta_l$, global stepsize $\eta$, $\beta_1,  \beta_2, \epsilon$, compressor $\cC(\cdot)$.
        \end{flushleft}
  \begin{algorithmic}[1]
        \STATE $\mb_0\gets 0$, $\vb_0\gets 0, \eb_1^i=0$
      \FOR{$t=1$ to $T$}
        \STATE Random sample a subset $\cS_t$ of clients
        \STATE Server sends $\xb_t$ to the subset $\cS_t$ of clients
        \STATE $\xb_{t,0}^i= \xb_t$
        \FOR{ each client $i \in \cS_t$ in parallel}
            \FOR{$k = 0,...,K-1$}
                \STATE Compute local stochastic gradient: $\gb_{t,k}^i= \nabla F_i(\xb_{t,k}^i; \xi_{t,k}^i)$
                \STATE $\xb_{t,k+1}^i = \xb_{t,k}^i -\eta_l \gb_{t,k}^i$
            \ENDFOR
            \STATE $\Delta_t^i = \xb_{t,K}^i-\xb_t$
            \STATE Compress $\hat{\Delta}_t^i = \cC(\Delta_t^i+\eb_t^i)$, send $\hat{\Delta}_t^i$ to the server and update $\eb_{t+1}^i = \Delta_t^i+\eb_t^i-\hat{\Delta}_t^i $
        \ENDFOR
        \FOR{ each client $j \notin \cS_t$ in parallel}
            \STATE client $j$ maintains the stale compression error $\eb_{t+1}^j = \eb_t^j$
        \ENDFOR
        \STATE Server aggregates local update $\hat{\Delta}_t = \frac{1}{|\cS_t|} \sum_{i\in\cS_t} \hat{\Delta}_t^i$
        \STATE Server updates $\xb_{t+1}$ using $\hat{\Delta}_t$ in the same way as in Algorithm \ref{alg:fedams} (Line 14-17)
      \ENDFOR
  \end{algorithmic}
\end{algorithm}

\section{Convergence Analysis}
In this section, we present the theoretical convergence results of our proposed FedAMS and FedCAMS in Algorithm \ref{alg:fedams} and \ref{alg:compfedams}. We first introduce some assumptions needed for the proof.
\begin{assumption}[Smoothness]\label{as:smooth}
Each loss function on the $i$-th worker $F_i(\xb)$ is $L$-smooth, i.e., $\forall \xb,\yb \in \RR^d$, 
\vspace{-5pt}
\begin{align*}
    \big|F_i(\xb)-F_i(\yb)-\langle\nabla F_i(\yb), \xb-\yb\rangle\big| \leq \frac{L}{2}\|\xb-\yb\|^2.
\end{align*}
\end{assumption}
This also implies the $L$-gradient Lipschitz condition, i.e., $\|\nabla F_i(\xb) - \nabla F_i(\yb) \| \leq L \|\xb - \yb \|$. Assumption \ref{as:smooth} is a standard assumption in nonconvex optimization problems, which has been also adopted in \citet{kingma2014adam, j.2018on,li2019convergence,yang2021achieving}.

\begin{assumption}[Bounded Gradient]\label{as:bounded-g}
Each loss function on the $i$-th worker $F_i(\xb)$ has $G$-bounded stochastic gradient on $\ell_2$, i.e., for all $\xi$, we have $\|\nabla f_i(\xb,\xi)\|\leq G$.
\end{assumption}
The assumption of bounded gradient is usually adopted in adaptive gradient methods \citep{kingma2014adam,j.2018on,zhou2018convergence,chen2020closing}

\begin{assumption}[Bounded Variance]\label{as:bounded-v} Each stochastic gradient on the $i$-th worker has a bounded local variance, i.e., for all $\xb, i \in [m]$,we have
    $\EE \big[ \|\nabla f_i(\xb,\xi)- \nabla F_i(\xb)\|^2\big] \leq \sigma_l^2$,
and the loss function on each worker has a global variance bound,
$\frac{1}{m}\sum_{i=1}^m \|\nabla F_i(\xb)-\nabla f(\xb)\|^2 \leq \sigma_g^2$.
\end{assumption}
Assumption \ref{as:bounded-v} is widely used in federated optimization problems \citep{li2019convergence,reddi2020adaptive,yang2021achieving}. The bounded local variance represents the randomness of stochastic gradients, and the bounded global variance represents data heterogeneity between clients. Note that $\sigma_g = 0$ corresponds to the \textit{i.i.d} setting, in which datasets from each client have the same distribution. 

In the following, we will show the convergence results of FedAMS\footnote{For simplicity, we will only present the convergence guarantee with Option 1. Note that the theoretical analysis can be easily extended to Option 2 with constant-only changes.} and FedCAMS. 

\subsection{Convergence Analysis for FedAMS} \label{subsec:fedams}

\noindent\textbf{Full Participation:}
For the full participation scheme, all workers participate in the communication rounds and model update, i.e., $|\cS_t| = m, \forall t \in [t]$.
\begin{theorem}\label{thm:full-noncomp}
Under Assumptions~\ref{as:smooth}-\ref{as:bounded-v}, if the local learning rate $\eta_l$ satisfies the following condition: $\eta_l \leq \min \Big\{\frac{1}{8KL}, \frac{\epsilon}{K \sqrt{\beta_2 K^2 G^2 + \epsilon} [(3+C_1^2) \eta L + 2\sqrt{2(1-\beta_2)}G]}\Big\}$, then the iterates of FedAMS in Algorithm \ref{alg:fedams} under full participation scheme satisfy
\begin{align}\label{eq:fedams-thm}
    & \min_{t\in [T]} \EE \big[\|\nabla f(\xb_t)\|^2\big] \notag\\
    & \leq 4 \sqrt{\beta_2 \eta_l^2 K^2 G^2 + \epsilon} \bigg[\frac{f_0-f_*}{\eta\eta_l K T} + \frac{\Psi}{T} + \Phi \bigg],
\end{align}
where $\Psi = \frac{C_1 G^2 d}{\sqrt{\epsilon}}+ \frac{2C_1^2 \eta\eta_l K L G^2 d}{\epsilon}$, $\Phi = \frac{5 \eta_l^2 K L^2}{\sqrt{2\epsilon}} (\sigma_l^2+6K \sigma_g^2)+ [(3+C_1^2)\eta L + 2\sqrt{2(1-\beta_2)} G] \frac{\eta_l}{2m \epsilon} \sigma_l^2$ and $C_1 = \frac{\beta_1}{1-\beta_1}$.
\end{theorem}
\begin{remark}
The upper bound for $\min_{t\in [T]} \EE [\|\nabla f(\xb_t)\|^2]$ contains three parts: the first two items that directly related to the total number of step $T$ are vanishing as $T \to \infty$. The last term in \eqref{eq:fedams-thm} relates to the local stochastic variance $\sigma_l$ and global variance $\sigma_g$. In the \textit{i.i.d} setting where each worker has the same data distribution, we have zero global variance, i.e., $\sigma_g = 0$, and the variance term $\Phi$ will be smaller and less dependent on the number of local steps $K$.
\end{remark}

\begin{corollary}\label{cor:fedams-full}
Suppose we choose the global learning rate $\eta= \Theta(\sqrt{Km})$ and local learning rate $\eta_l= \Theta (\frac{1}{\sqrt{T} K})$, when $T$ is sufficient large, i.e., $T\geq Km$, the convergence rate for FedAMS in Algorithm \ref{alg:fedams} under full participation scheme satisfies
\begin{align}\label{eq:fedams-full-cor}
    & \min_{t\in [T]} \EE \big[\|\nabla f(\xb_t)\|^2\big] =  \cO\bigg(\frac{1}{\sqrt{TKm}}\bigg).
\end{align}
\end{corollary}
\begin{remark}
Corollary \ref{cor:fedams-full} suggests that with sufficient large $T$, FedAMS achieves a convergence rate of  $\cO(\frac{1}{\sqrt{TKm}})$, which matches the result for general federated non-convex optimization methods such as SCAFFOLD \citep{karimireddy2020scaffold}, FedAdam \citep{reddi2020adaptive}.  
\end{remark}
\begin{remark}
Note that compared with FedAdam \citep{reddi2020adaptive}, our theoretical analysis on FedAMS makes improvements in completing the proof. The analysis in \citet{reddi2020adaptive} can only consider the case when $\beta_1 = 0$ which largely simplifies the proof in their theoretical analysis, while we provide a full analysis on FedAMS with non-zero momentum term.
\end{remark}

\noindent\textbf{Partial Participation:}
In the partial participation scheme, we assume that only $n$ of $m$ workers participate the local update and communicate with the central server on each step $t$, i.e., $|\cS_t| = n, \forall t \in [1,T]$. The partial participation includes the randomness of sampling, and the coefficient varies for different sampling methods. Here we consider the random sampling without replacement. At the $t$-th iteration, we randomly sample a subset $\cS_t$ contains $n$ workers for local updating, for any two workers $i,j \in \cS_t$, the probability of being sampled to participate in model update are $\PP\{i \in \cS_t \} = \frac{n}{m}$ and $\PP\{i,j \in \cS_t \} = \frac{n(n-1)}{m(m-1)}$. 

\begin{theorem} \label{thm:part-fedams}
Under Assumption~\ref{as:smooth}-\ref{as:bounded-v}, if the local learning rate $\eta_l$ satisfies:
$\eta_l \leq \min\Big\{ \frac{1}{8KL},$ $\frac{n(m-1)\epsilon}{24 m(n-1) K \sqrt{\beta_2 K^2 G^2 + \epsilon} \cdot [3\eta L + C_1^2 \eta L + 2\sqrt{2(1-\beta_2)}G]} \Big\}$, then the iterates of FedAMS in Algorithm \ref{alg:fedams} under partial participation scheme satisfy
\begin{align}\label{eq:part-fedams}
    & \min_{t\in [T]} \EE \big[\|\nabla f(\xb_t)\|^2\big] \notag\\ 
    & \leq 8 \sqrt{\beta_2 \eta_l^2 K^2 G^2 + \epsilon} \bigg[\frac{f_0-f_*}{\eta\eta_l K T} + \frac{\Psi}{T} + \Phi \bigg],
\end{align}
where $\Psi = \frac{C_1 G^2 d}{\sqrt{\epsilon}}+ \frac{2C_1^2 \eta\eta_l K L G^2 d}{\epsilon}$ and $\Phi = \frac{5 \eta_l^2 K L^2}{\sqrt{2\epsilon}} (\sigma_l^2+6K \sigma_g^2)+ [(3 + C_1^2) \eta L+ 2\sqrt{2(1-\beta_2)} G] \frac{\eta_l}{2n\epsilon} \sigma_l^2 + [(3 + C_1^2) \eta L+ 2\sqrt{2(1-\beta_2)} G] \frac{\eta_l (m-n)}{2n(m-1)\epsilon} [15 K^2L^2 \eta^2 (\sigma_l^2 + 6K\sigma_g^2) + 3K\sigma_g^2]$ and $C_1 = \frac{\beta_1}{1-\beta_1}$.
\end{theorem}
\begin{remark}
The upper bound for $\min_{t\in [T]} \EE [\|\nabla f(\xb_t)\|^2]$ of partial participation is similar to full participation case but with a larger variance term $\Phi$. This is due to the fact that random sampling of participating workers introduces an additional variance during sampling. In the \textit{i.i.d} setting where we have zero global variance, i.e., $\sigma_g = 0$, the variance term will get smaller and less dependent on the number of local steps $K$ as well.
\end{remark}

\begin{corollary}\label{cor:fedams-part}
Suppose we choose the global learning rate $\eta= \Theta(\sqrt{Kn})$ and local learning rate $\eta_l = \Theta (\frac{1}{\sqrt{T} K})$, the convergence rate for FedAMS in Algorithm \ref{alg:fedams} under partial participation scheme without replacement sampling is
\begin{align}\label{eq:cor-fedams-part}
    & \min_{t\in [T]} \EE \big[\|\nabla f(\xb_t)\|^2\big] = \cO\bigg(\frac{\sqrt{K}}{\sqrt{Tn}}\bigg).
\end{align}
\end{corollary}
\begin{remark}
Note that Corollary \ref{cor:fedams-part} suggests that the dominant term in \eqref{eq:part-fedams} is $\cO(\frac{\sqrt{K}}{\sqrt{Tn}})$, which directly relates to the global variance $\sigma_g^2$. Such convergence rate is consistent with the partial participation result of FedAvg in the non i.i.d case in \citep{yang2021achieving}. It shows that the global variance has more impact on convergence behaviour in partial participation cases, especially in highly \textit{non i.i.d.} cases where $\sigma_g$ is large. Corollary \ref{cor:fedams-part} also suggests that larger number of participating clients $n$ would accelerate the convergence. Note that although FedAdam \citep{reddi2020adaptive} did not provide explicit conclusions on the partial participation setting, its appendix introduced the necessary steps for analyzing such setting, which cannot imply such desired relationship with $n$. 
\end{remark}
\begin{remark}
The impact of the number of local updates $K$ is complicated. In partial participation settings, it shows that larger $K$ slows down the convergence while full participation suggests the opposite. Similar slow-down result has also been mentioned in \citet{li2019convergence}, while some others \citep{stich2018local,mcmahan2017communication} showed that larger $K$ would increase the convergence rate. We will leave this as future work.
\end{remark}

\subsection{Convergence Analysis for FedCAMS} \label{subsec:cfedams}
Let us first introduce the assumption for the compressor.
\begin{assumption}[Biased Compressor]\label{as:compressor}Consider a biased operator $\cC: \RR^d \to \RR^d$: for $\forall \xb \in \RR^d$, there exists constant $0\leq q \leq 1$ such that 
\begin{align*}
    \|\cC(\xb)-\xb\| \leq q \|\xb\|, \forall \xb\in \RR^d.
\end{align*}
Note that $q = 0$ leads to $\cC(\xb)=\xb$ which means no compression to $\xb$. Assumption \ref{as:compressor} is a standard assumption for biased compressors \citep{karimireddy2019error,alistarh2018convergence}. There are several widely used compressors satisfying \ref{as:compressor} such as scaled-sign compressor and top-$k$ compressor.

\noindent\textbf{Top-$k$} \citep{stich2018sparsified}: For $1\leq k \leq d$ and $\forall \xb \in \RR^d$, the coordinate of $\xb$ is ordered by the magnitude $|x_{(1)}|\leq |x_{(2)}|\leq \cdots \leq |x_{(d)}|$. Denote $\alpha_1, \alpha_2,...,\alpha_d$ as standard unit basis vectors in $\RR^d$. The compressor $\cC_{\text{top}}: \RR^d \to \RR^d$ is defined as:
$\cC_{\text{top}}(\xb) = \sum_{i=d-k+1}^d \xb_{(i)} \alpha_{(i)}$.
\begin{remark}
Let us define the compression ratio as $r = k/d$. It can be shown that $\|\cC_{\text{top}}(\xb)- \xb\|^2 \leq (1-k/d) \|\xb\|^2 = (1-r) \|\xb\|^2$, and thus we have $q = \sqrt{1-r}$. 
\end{remark}

\noindent\textbf{Scaled sign} \citep{karimireddy2019error}:For $1\leq k \leq d$ and $\forall \xb \in \RR^d$, the compressor $\cC_{\sign}: \RR^d \to \RR^d$ is defined as
\begin{align*}
    \cC_{\sign}(\xb) = \|\xb\|_1 \cdot \sign(\xb)/d.
\end{align*}

\end{assumption}
\begin{remark}
For scaled sign compressor, we have $\|\cC_{\sign}(\xb)-\xb\|^2 = (1-\|\xb\|_1^2 /d \|\xb\|^2) \|\xb\|^2$, thus we have $q = \sqrt{1-\|\xb\|_1^2 /d \|\xb\|^2}$.
\end{remark}

\begin{assumption}[Compression Dissimilarity] \label{as:compressor-adapt} For the biased compressor satisfies \ref{as:compressor}, there exists a constant $\xi$ such that, for each iteration $t \geq 0$, we have 
{\small
\begin{align*}
    \bigg\|\cC\Big( \frac{1}{m} \sum_{i=1}^m [\Delta_t^i + \eb_t^i] \Big) - \frac{1}{m} \sum_{i=1}^m \cC(\Delta_t^i + \eb_t^i)\bigg\|   \leq \gamma \bigg\| \frac{1}{m} \sum_{i=1}^m \Delta_t^i \bigg\|.
\end{align*}
}
\end{assumption}
\vspace{-3pt}
Assumption \ref{as:compressor-adapt} bounds the difference between the average of compression and compression of average. Similar assumptions have been adopted in \citet{alistarh2018convergence}\footnote{We further discuss Assumption \ref{as:compressor-adapt} in Appendix \ref{appendix:fedcams-assumption}.}.

Next, we show the convergence analysis for FedCAMS. Due to the space limit, we only show the full participation setting and leave the partial participation setting in Appendix \ref{sec:ext}.

\begin{theorem} \label{thm:full-cfedams}
Under Assumptions~\ref{as:smooth}-\ref{as:bounded-v} \ref{as:compressor}, and \ref{as:compressor-adapt}, if the local learning rate $\eta_l$ satisfies the following condition: 
$\eta_l \leq \min \Big\{ \frac{1}{8KL}, \frac{\epsilon}{K C_{\beta,q}  [3\eta L + 2C_2 \eta L + 2\sqrt{2(1-\beta_2)}G]}\Big\}$, where $C_{\beta,q} = \sqrt{4\beta_2(1+q^2)^3(1-q^2)^{-2} K^2 G^2 + \epsilon}$, then the iterates of FedCAMS in Algorithm \ref{alg:compfedams} satisfy
\begin{align}
    & \min_{t\in [T]} \EE \big[\|\nabla f(\xb_t)\|^2\big] \notag\\
    & \leq 4 \sqrt{\beta_2 \frac{4(1+q^2)^3}{(1-q^2)^2} \eta_l^2 K^2 G^2 + \epsilon} \bigg[\frac{f_0-f_*}{\eta\eta_l K T} + \frac{\Psi}{T} + \Phi \bigg],
\end{align}
where $\Psi = \frac{C_1 G^2 d}{\sqrt{\epsilon}}+ \frac{2C_1^2 \eta\eta_l K L G^2 d}{\epsilon}$ and $\Phi = \frac{5 \eta_l^2 K L^2}{\sqrt{2\epsilon}} (\sigma_l^2+6K \sigma_g^2)+ [(3+2C_2)\eta L + 2\sqrt{2(1-\beta_2)} G] \frac{\eta_l}{2m \epsilon} \sigma_l^2$, $C_1 = \frac{\beta_1}{1-\beta_1} + \frac{2q}{1-q^2}$ and $C_2 = \frac{\beta_1^2}{(1-\beta_1)^2} + \frac{4(q+\gamma)^2}{(1-q^2)^2}$.
\end{theorem}

\begin{remark}
The convergence rate in Theorem \ref{thm:full-cfedams} contains three parts as well, the first two parts are related to total iterates $T$, and they vanish as $T$ increases. The last term $\Phi$ shows no direct dependency on $T$, but on local and global variances. In the \textit{i.i.d} case where $\sigma_g = 0$, the variance $\Phi$ will decrease and show less dependency on the number of local steps $K$.
\end{remark}

\begin{corollary}\label{cor:cfedams-full}
Suppose we choose the global learning rate $\eta = \Theta(\sqrt{Km})$ and local learning rate $\eta_l = \Theta (\frac{1}{\sqrt{T} K})$, when $T$ is sufficient large, i.e., $T\geq Km$, the convergence rate for FedCAMS in Algorithm \ref{alg:compfedams} under full participation scheme satisfies
\begin{align}
    & \min_{t\in [T]} \EE \big[\|\nabla f(\xb_t)\|^2\big] =  \cO\bigg(\frac{1}{\sqrt{TKm}} \bigg).
\end{align}
\end{corollary}

\begin{remark}
Corollary \ref{cor:cfedams-full} suggests that with sufficient large $T$, FedCAMS achieves the desired $\cO(\frac{1}{\sqrt{TKm}})$ convergence rate which matches the result for its uncompressed counterpart, FedAMS. This suggests that FedCAMS can indeed achieve better communication efficiency without sacrificing much on the accuracy.
\end{remark}

\begin{remark}
The constants $C_1$ and $C_2$ in Theorem \ref{thm:full-cfedams} are related to the compression constant $q$. 
Specifically, if we track the dependency on $q$ in the convergence rate, we have $O(1/(1-q) \sqrt{TKm})$ under the full participation scheme. A larger $q$ ($q \to 1$) corresponds to a stronger compression we applied, leading to worse convergence due to heavier information losses. Note that this $q$-dependency is common for the adaptive gradient method since the convergence proof of adaptive gradient methods heavily relies on the bounded gradient assumption and thus the compressed gradient bound is related to $q$. Similar type of $q$-dependency also occurs in other communication compressed distributed Adam methods such as \citet{chen2021quantized}.\footnote{ Note that the $q$-dependency in 1-bit Adam \cite{tang20211} is actually different due to the use of variance-freezed Adam update, i.e., freeze the variance term $\vb_t$ of Adam update as a constant after a few epochs, which make it resembles momentum SGD.
}
\end{remark}

\section{Experiments}\label{sec:experiments}
In this section, we present empirical validations toward the effectiveness of our proposed algorithms. Firstly, we provide comparisons between FedAMS and other first-order federated optimization baselines. Secondly, we provide experimental results of our proposed communication-efficient adaptive federated learning method, FedCAMS, to show its effectiveness in achieving communication-efficient adaptive federated learning. 

\noindent\textbf{Experimental Setup:} We test all federated learning baselines, including ours on CIFAR10 and CIFAR100 datasets \citep{krizhevsky2009learning} using the following two models: (1) ResNet-18 \citep{he2016deep}, a widely used convolutional neural network model which is commonly trained by SGD; and (2) ConvMixer model \citep{trockman2022patches}, which shares similar ideas to vision transformer \citep{dosovitskiy2021an} to use patch embeddings to preserve locality and similarly is trained via adaptive gradient methods by default. We set in total $100$ clients for all federated training experiments. We set the partial participation ratio as $0.1$, i.e., in each round, the server picks $10$ out of $100$ clients to participate in the communication and model update. In each round, the client will perform $3$ local epochs of local training with batch size $20$. We search for the best training hyper-parameters for each baseline, including ours. Due to the space limit, we leave all the hyper-parameter details as well as the CIFAR-100 experiments in the Appendix.

\subsection{FedAMS and Adaptive Federated Optimization}
We compare two options of the FedAMS framework with several state-of-the-art adaptive federated learning optimization methods, including: (1) FedAdam \citep{reddi2020adaptive} (2) FedYogi \citep{reddi2020adaptive} as well as standard federated baselines: (3) FedAvg \citep{mcmahan2017communication}. Note that the Option 2 for FedAMS is same as FedAMSGrad \citep{tong2020effective}. Thus in this section, we denote FedAMS for Option 1 and FedAMSGrad for Option 2 in the general FedAMS framework. 


Figure \ref{fig:cifar10} shows the convergence result of FedAMS and other federated learning baselines on training CIFAR-10 dataset with ResNet-18 model and ConvMixer-256-8 model. We compare the training loss and test accuracy against global rounds for each model. For the ResNet-18 model, FedAMS and FedYogi achieve quite similar performances, which are significantly better than the other three baselines. In particular, FedAMS performs the best in terms of the final training loss and test accuracy. On the other hand, FedAMSGrad and FedAdam obtain quite similar results on test accuracy and training loss. FedAvg achieves a slightly better training loss to FedAdam and FedAMSGrad but much higher test accuracy which is close to FedYogi and FedAMS. For the ConvMixer-256-8 model, which is typically trained via adaptive gradient methods, we observe that all adaptive federated optimization methods (FedAdam, FedYogi, FedAMSGrad and FedAMS) achieve much better performance in terms of both training loss and test accuracy than FedAvg. In detail, FedAMS again achieves a significantly better result than other baselines. Other adaptive methods, including FedAdam, FedYogi, and FedAMSGrad, have similar convergence behaviour when training the ConvMixer-256-8 model. Such results empirically show the effectiveness of our proposed FedAMS method with max stabilization.

Figure \ref{fig:cifar10_kn_resnet} shows the effect of parameter $n$ on the convergence rate by choosing different number of $n$ from $\{5, 10, 20\}$. From Figure \ref{fig:cifar10_kn_resnet} we can observe that a larger number of participating clients $n$ in general achieves a faster convergence rate. This verified our theoretical results in Section \ref{subsec:fedams}.

\begin{figure}[ht!]
\centering
\subfigure[ResNet-18]{\includegraphics[width=0.23\textwidth]{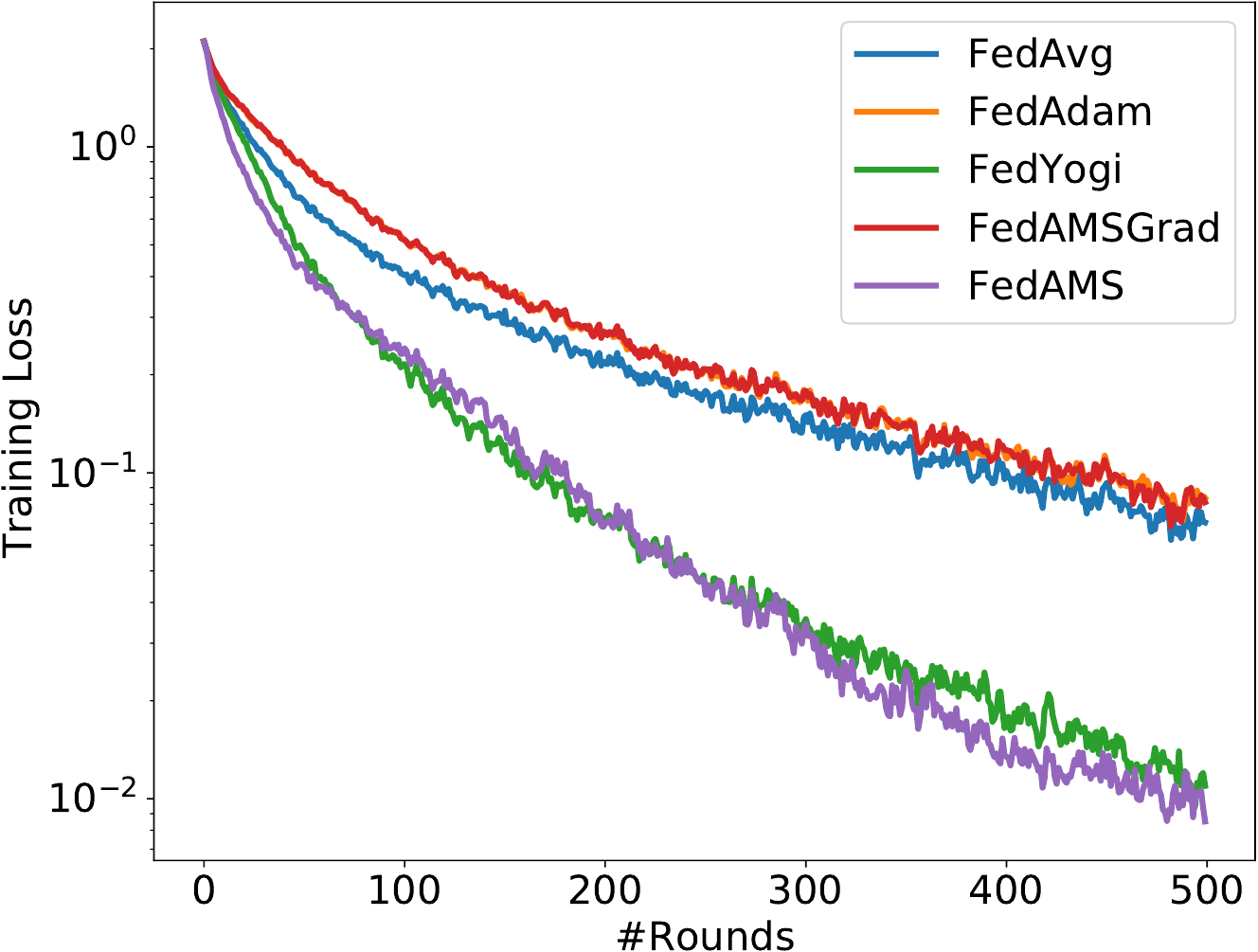}}
\subfigure[ResNet-18]{\includegraphics[width=0.23\textwidth]{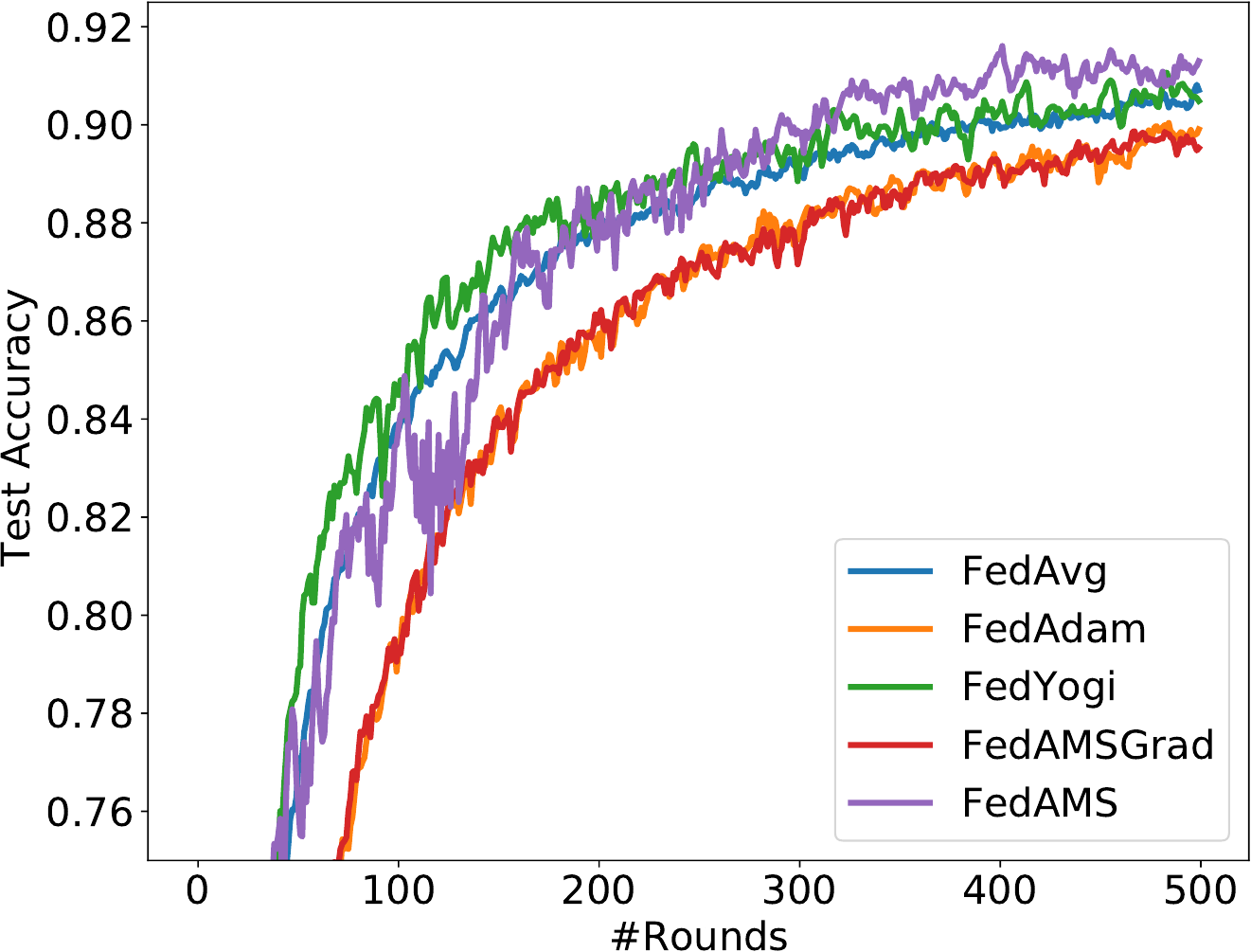}}
\subfigure[ConvMixer-256-8]{\includegraphics[width=0.23\textwidth]{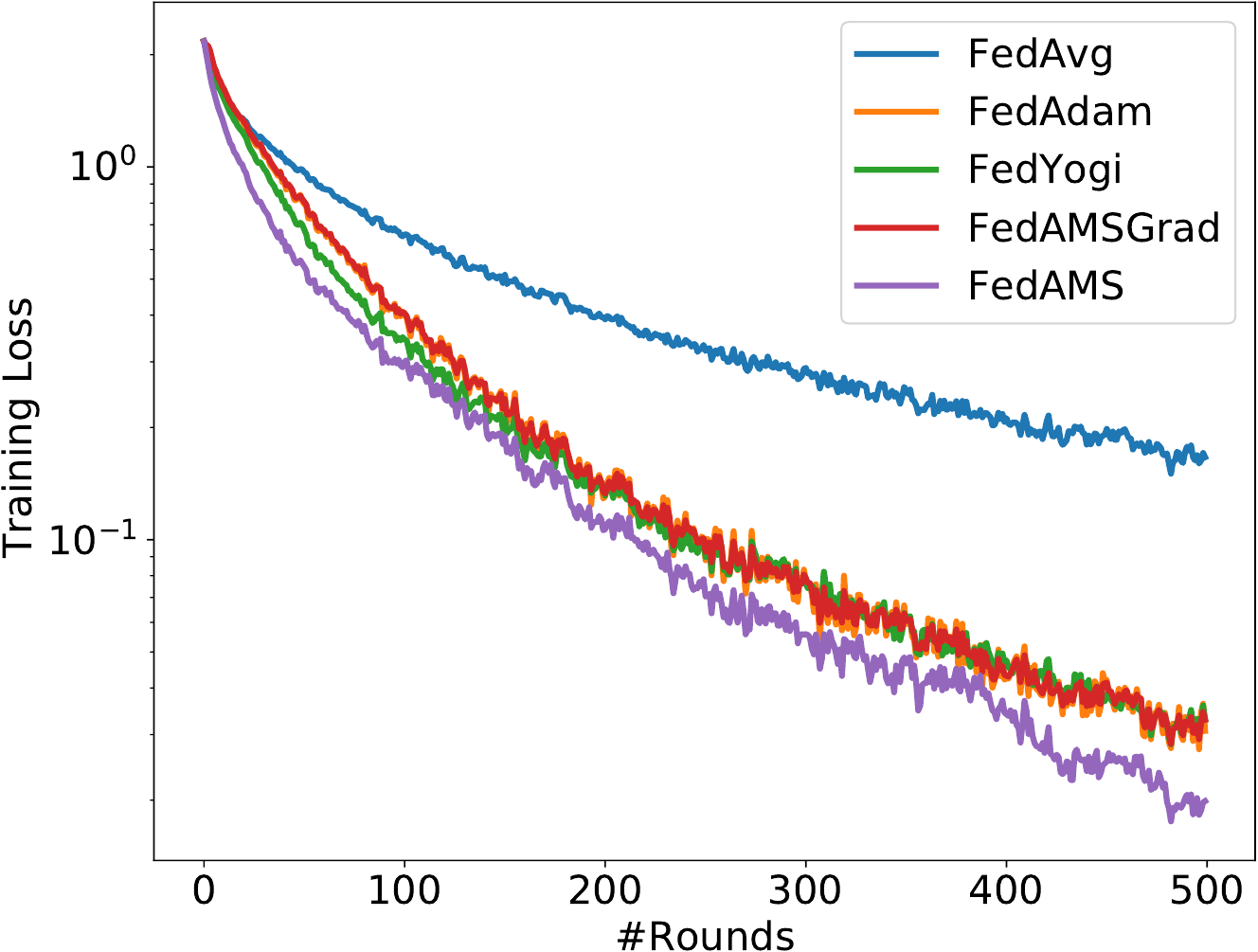}}
\subfigure[ConvMixer-256-8]{\includegraphics[width=0.23\textwidth]{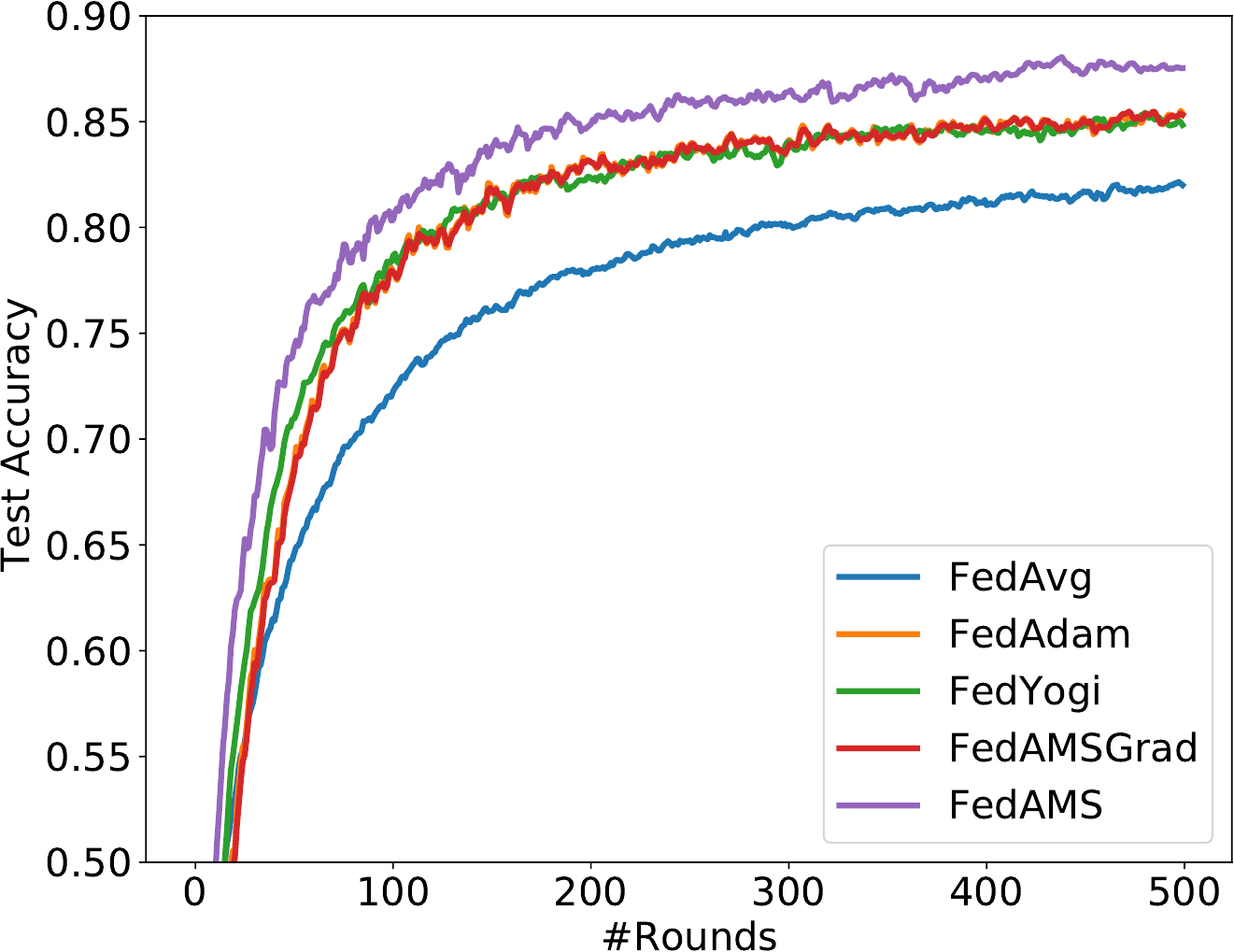}}

\caption{The learning curves for FedAMS and other federated learning baselines on training CIFAR-10 data (a)(b) show the results for the ResNet-18 model and (c)(d) show the results for the ConvMixer-256-8 model. }
\label{fig:cifar10}
\end{figure}

\begin{figure}[ht!]
\centering
\subfigure[ResNet-18]{\includegraphics[width=0.23\textwidth]{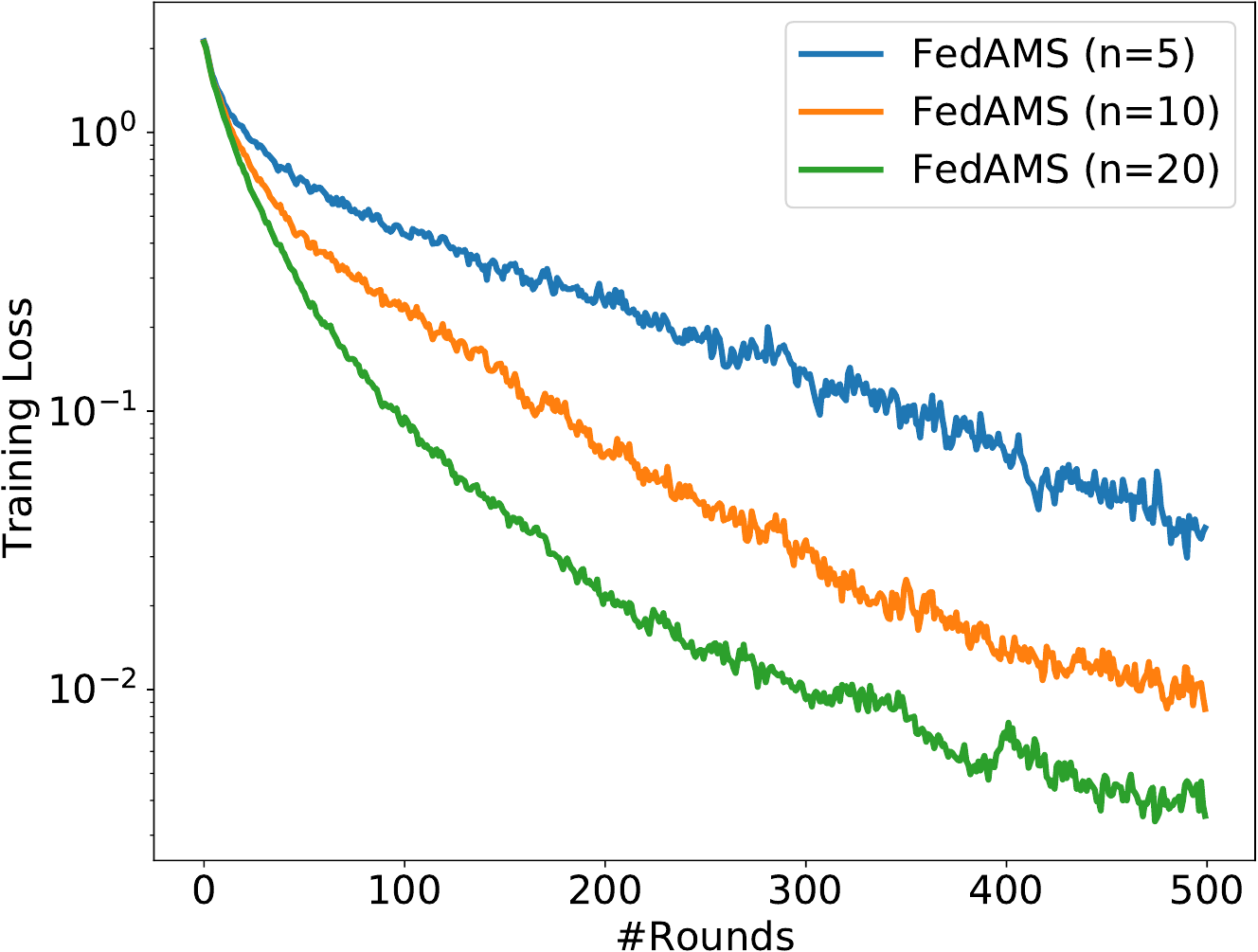}}
\subfigure[ConvMixer-256-8]{\includegraphics[width=0.23\textwidth]{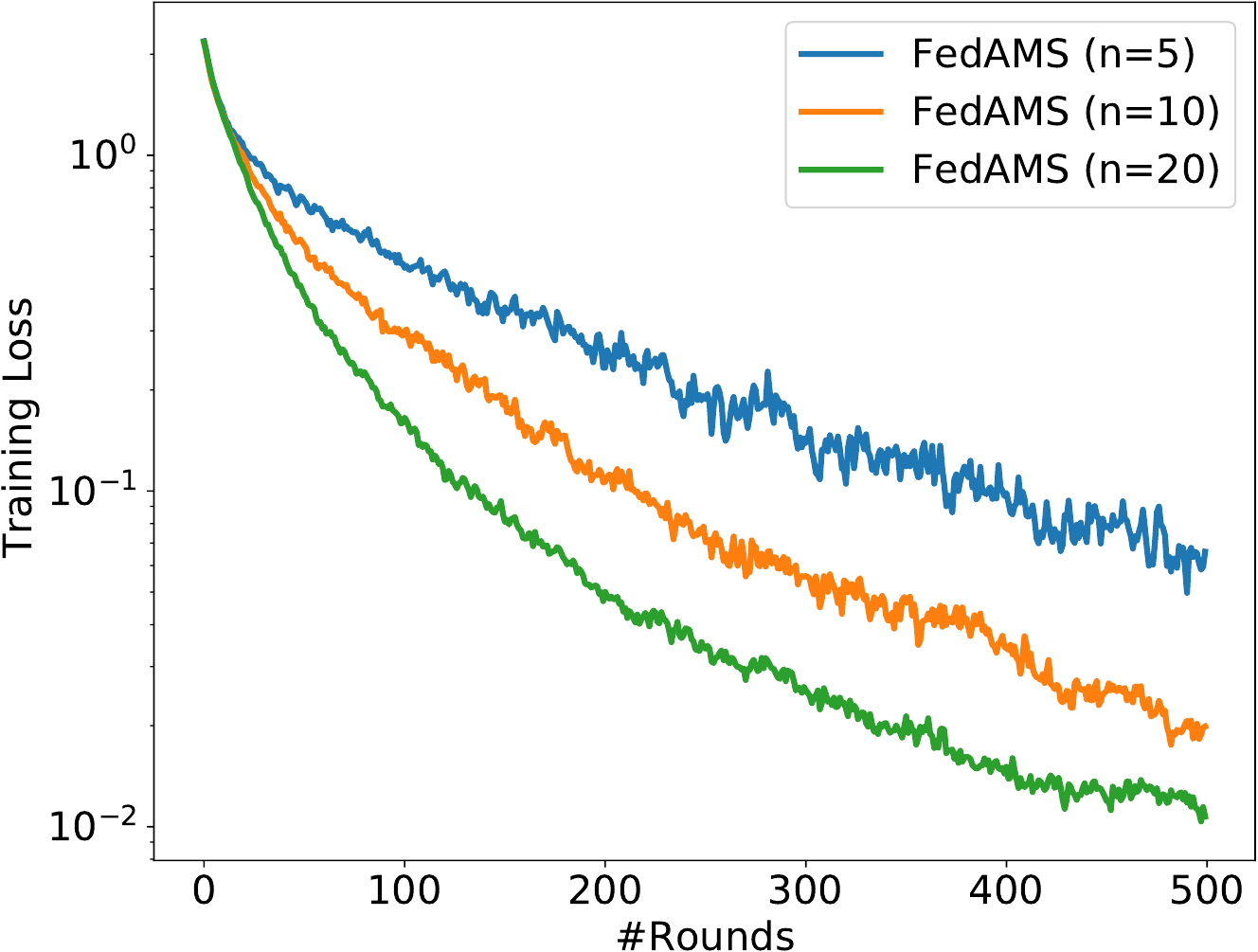}}
\caption{The learning curves for FedAMS with different participating number of clients $n$ in training CIFAR-10 data on the ResNet-18 and the ConvMixer-256-8 models.}
\label{fig:cifar10_kn_resnet}
\end{figure}

Figure \ref{fig:abla-local-ep} shows the ablation study with different local epochs by choosing different number of local epochs $E$ from $\{3, 10, 30,100\}$ when training CIFAR-10 data on ResNet-18 with FedAMS optimizer. We observe that larger $E$ leads to faster convergence, but larger $E$ does not show a significant advantage in achieving a higher test accuracy. We follow FedAvg \citep{mcmahan2017communication} and FedAdam \citep{reddi2020adaptive} and set local epoch $E = 3$ by default unless otherwise specified. 
\vspace{-10pt}
\begin{figure}[ht!]
    \centering
    \subfigure[Training Loss]{\includegraphics[width=0.23\textwidth]{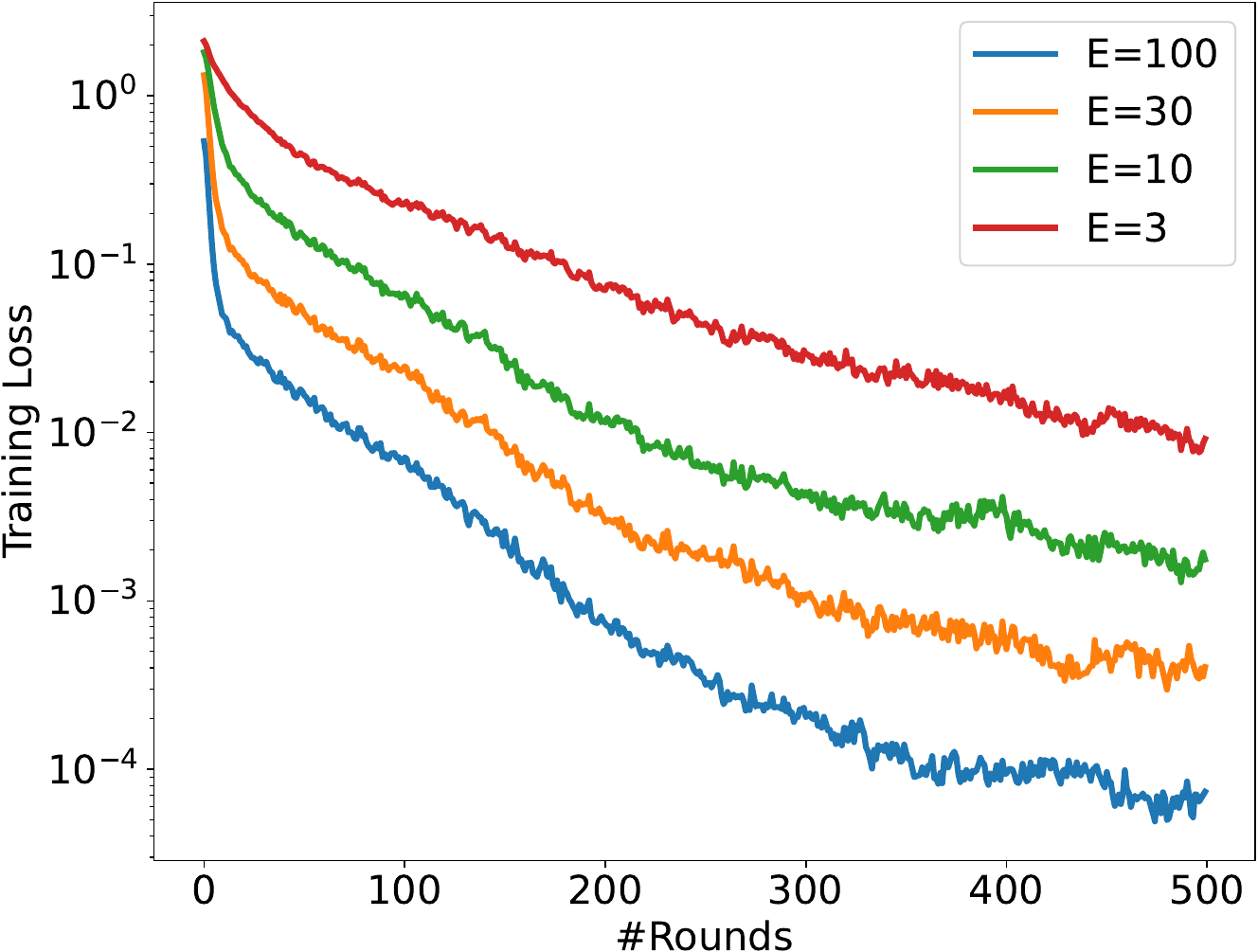}}
    \subfigure[Test Accuracy]{\includegraphics[width=0.23\textwidth]{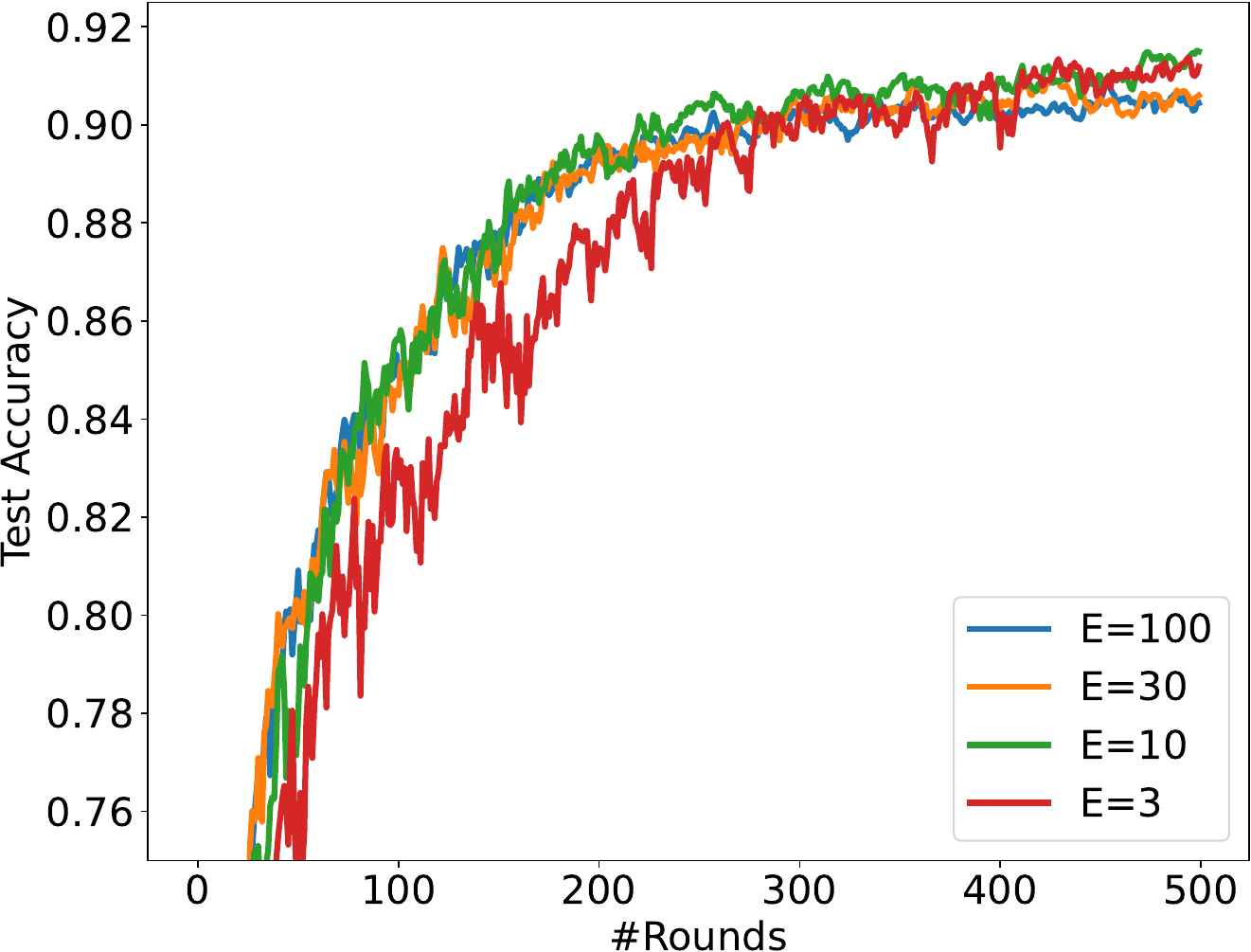}}
    \caption{The learning curves for FedAMS with different number of local epochs $E$ in training CIFAR-10 data on the ResNet-18 model.}
    \setlength{\abovecaptionskip}{-10pt}
    \setlength{\belowcaptionskip}{-10pt}
    \label{fig:abla-local-ep}
\end{figure}

\subsection{Communication-Efficient FedCAMS}

Figure \ref{fig:cifar10_ef_resnet} shows the convergence results of FedAMS and FedCAMS\footnote{Here FedCAMS adopt Option 1 for the final update step for fairness comparisons (same as FedAMS). } with different compression strategies on training CIFAR-10 dataset with the ResNet-18 model. It includes comparisons between scaled sign compressor and top-$k$ compressor with compression ratio $r \in\{1/64, 1/128, 1/256\}$. We compare the training loss and test accuracy against global rounds and the (pseudo) gradient communication bits\footnote{Note that here we only count the client-to-server one-way communications compression.}. FedAMS, who does not conduct any communication compression, performs the best in terms of training loss yet requires a large volume of communication costs. For our FedCAMS compression methods, sign compressor and top-$k$ compressor with ratio $r=1/64$ achieve similar performance in terms of test accuracy against the training rounds and obtain the best trade-off between communication efficiency and model accuracy. 
Figure \ref{fig:cifar10_ef_resnet} (b)(d) show the direct comparison against the communication bits of training ResNet-18 on CIFAR-10. In particular, we can observe that the top-$k$ compressor with a smaller $r$ (i.e., a heavier compression with more information lost), obtains better communication efficiency but a slower convergence rate. Note that for a $d$ dimensional vector, the overall cost of a scaled sign compressor is $32+d$ bits, and it is roughly the same communication costs as a top-$k$ compressor\footnote{Top-$k$ compressor also needs to communicate about the chosen $k$ locations which roughly double the costs.} with a ratio $r=1/64$. This verifies our theoretical results in Section \ref{subsec:cfedams}.

\begin{figure}[ht!]
\centering
\subfigure[Training Loss]{\includegraphics[width=0.23\textwidth]{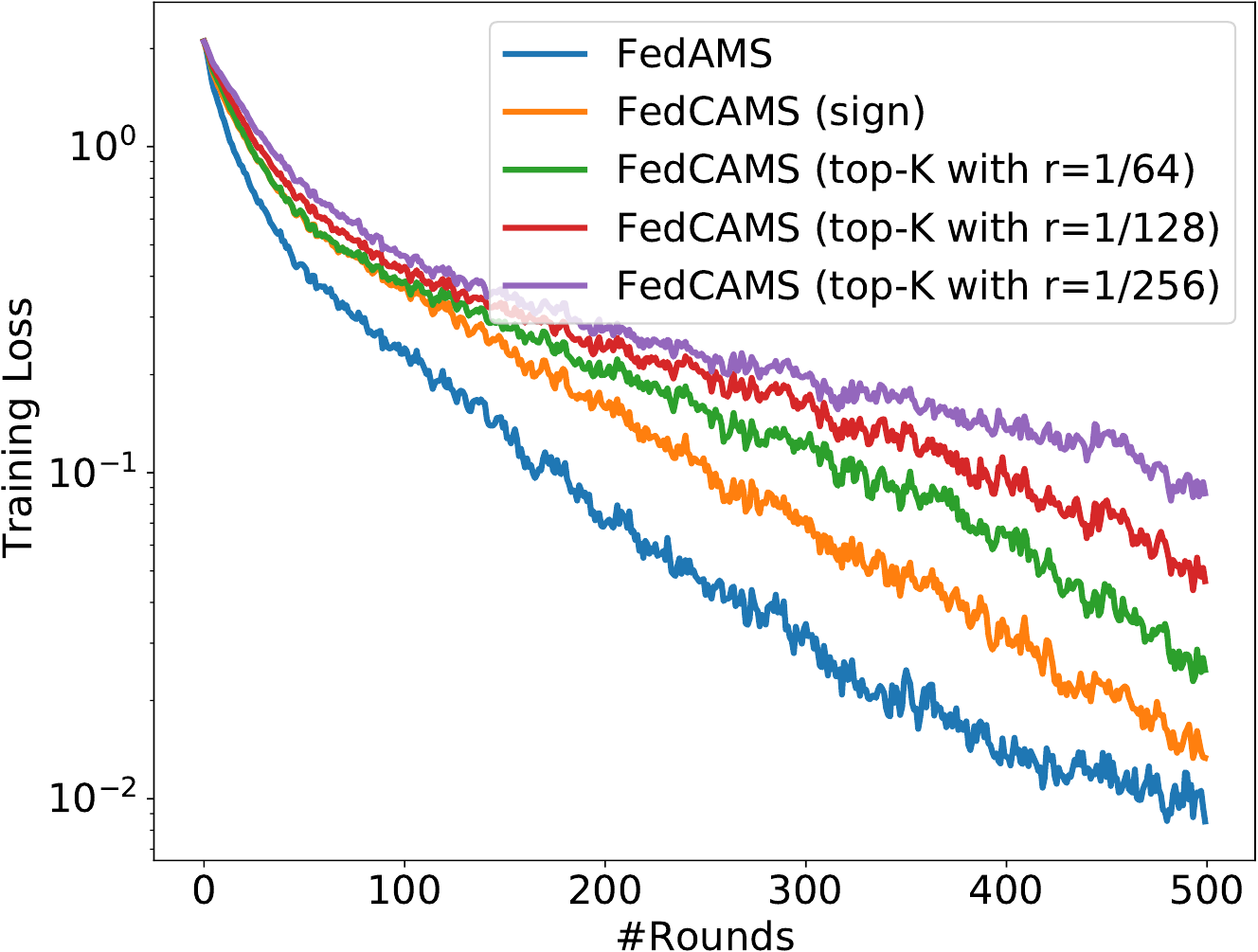}}
\subfigure[Training Loss]{\includegraphics[width=0.23\textwidth]{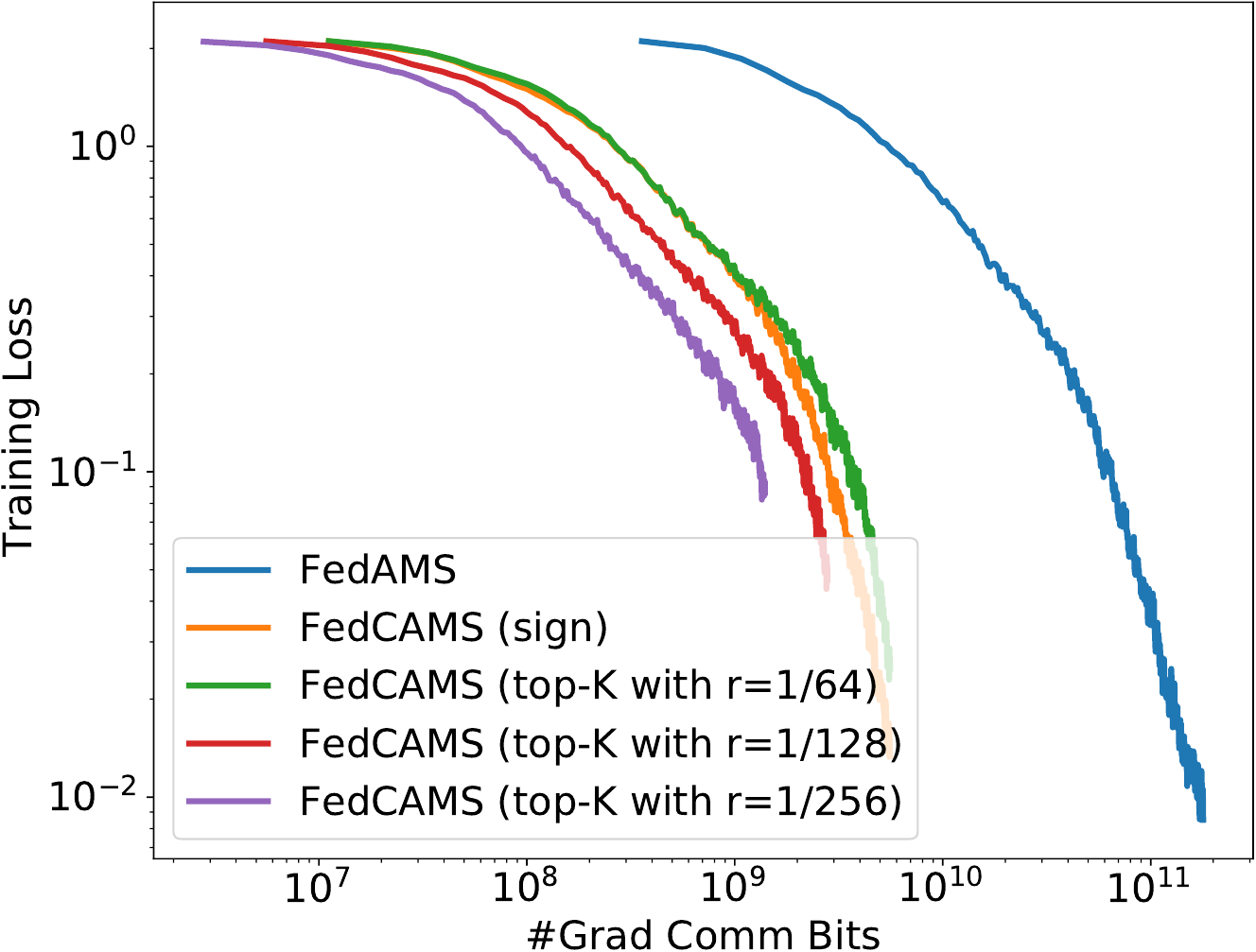}}
\subfigure[Test Accuracy]{\includegraphics[width=0.23\textwidth]{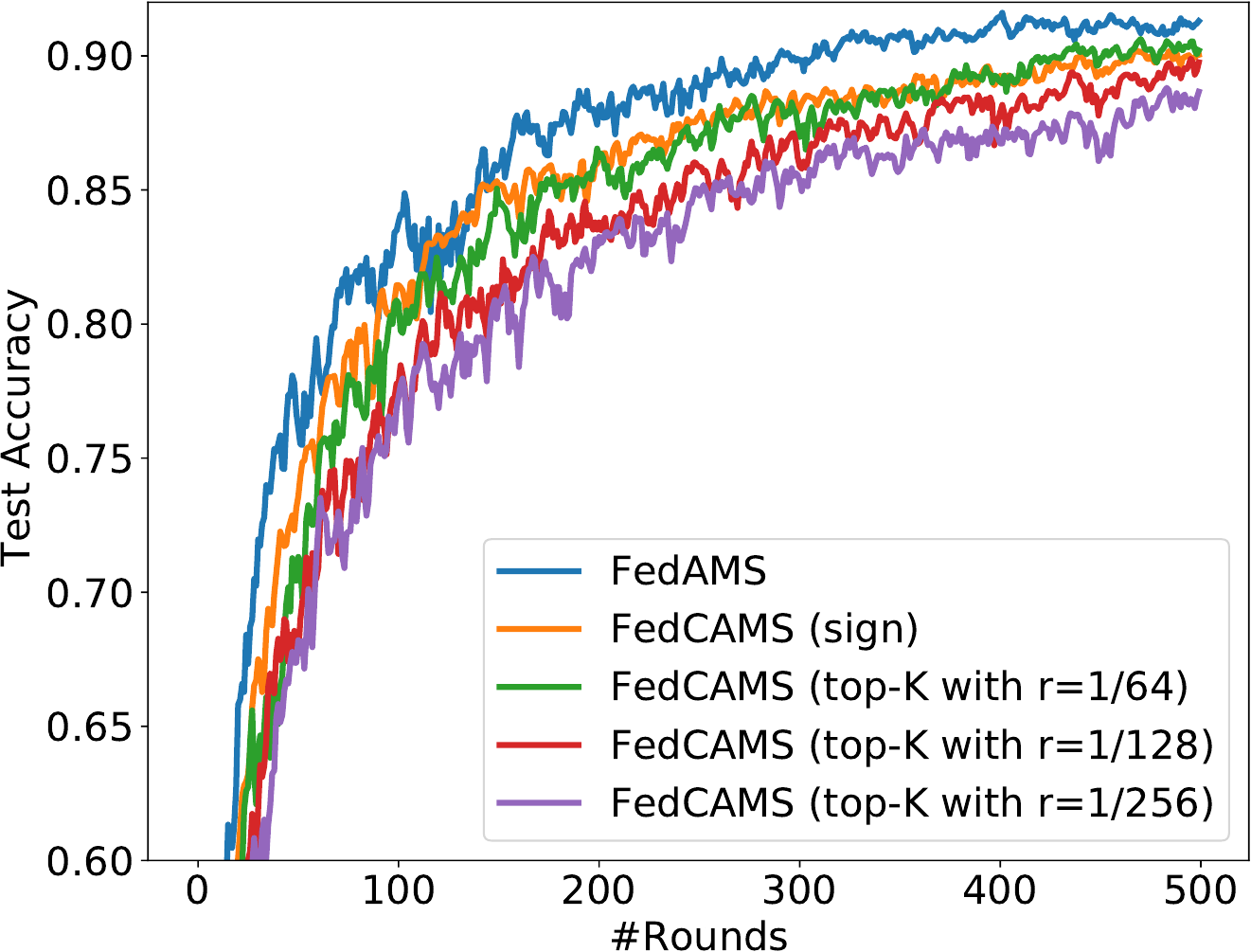}}
\subfigure[Test Accuracy]{\includegraphics[width=0.23\textwidth]{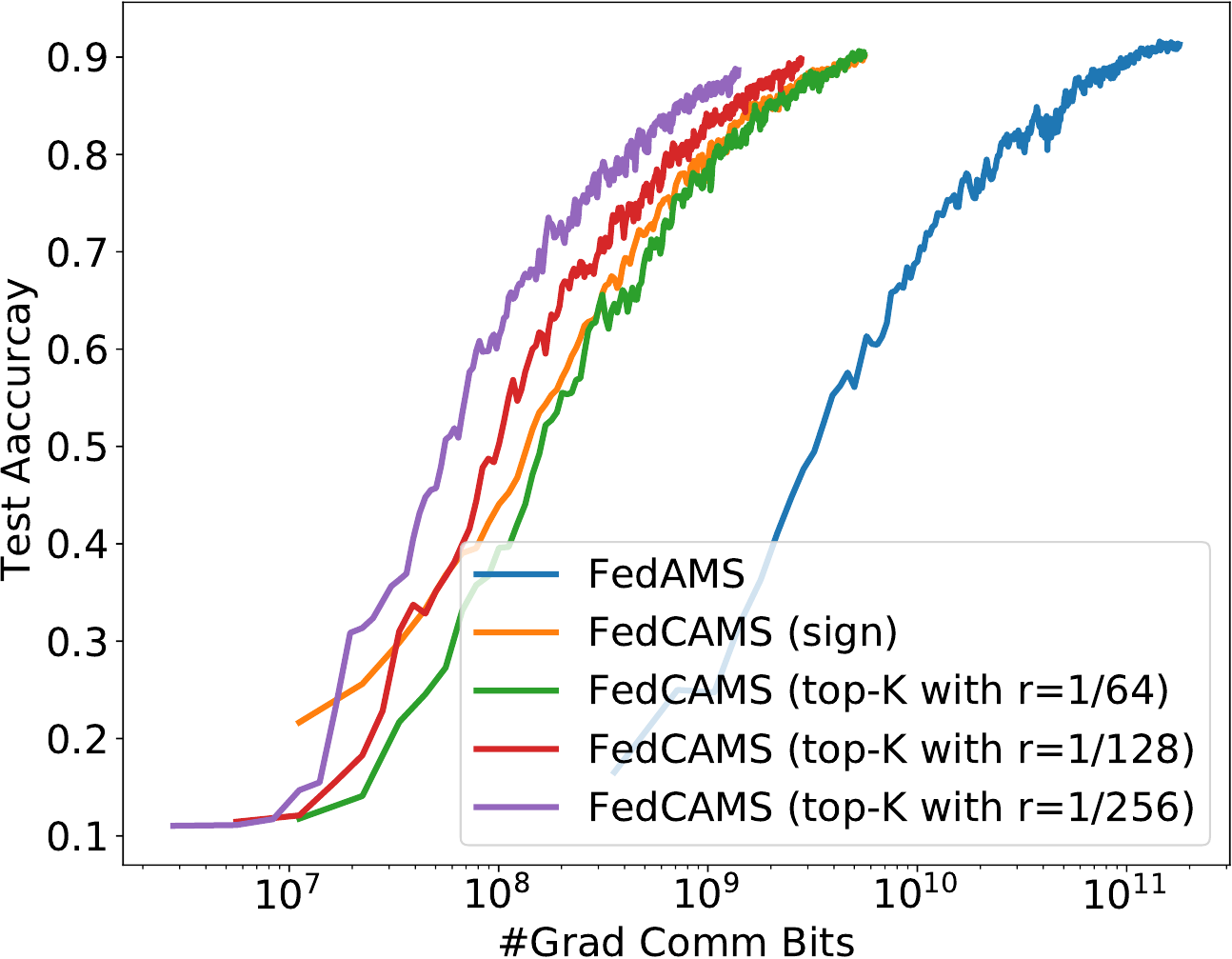}}
\caption{The learning curves for FedCAMS and uncompressed FedAMS on training CIFAR-10 data on the ResNet-18 model.}
\label{fig:cifar10_ef_resnet}
\end{figure}

Figure \ref{fig:cifar10_ef_convmixer} shows the convergence results of FedAMS and FedCAMS with the same compression strategies as in Figure \ref{fig:cifar10_ef_resnet} on training CIFAR-10 dataset with the ConvMixer-256-8 model. We notice that FedCAMS with the scaled sign compressor achieves roughly the same training loss and test accuracy as FedAMS but with a few orders of magnitude less in communication costs, while other top-$k$ compression models have significantly worse performance. Among the top-$k$ compressor trained models, the one with compression ratio $r=1/64$ still obtains better training loss and test accuracy but higher communication costs. These results suggest that our proposed FedCAMS is communication-efficient while maintaining high accuracy.

\begin{figure}[ht!]
\centering
\subfigure[Training Loss]{\includegraphics[width=0.23\textwidth]{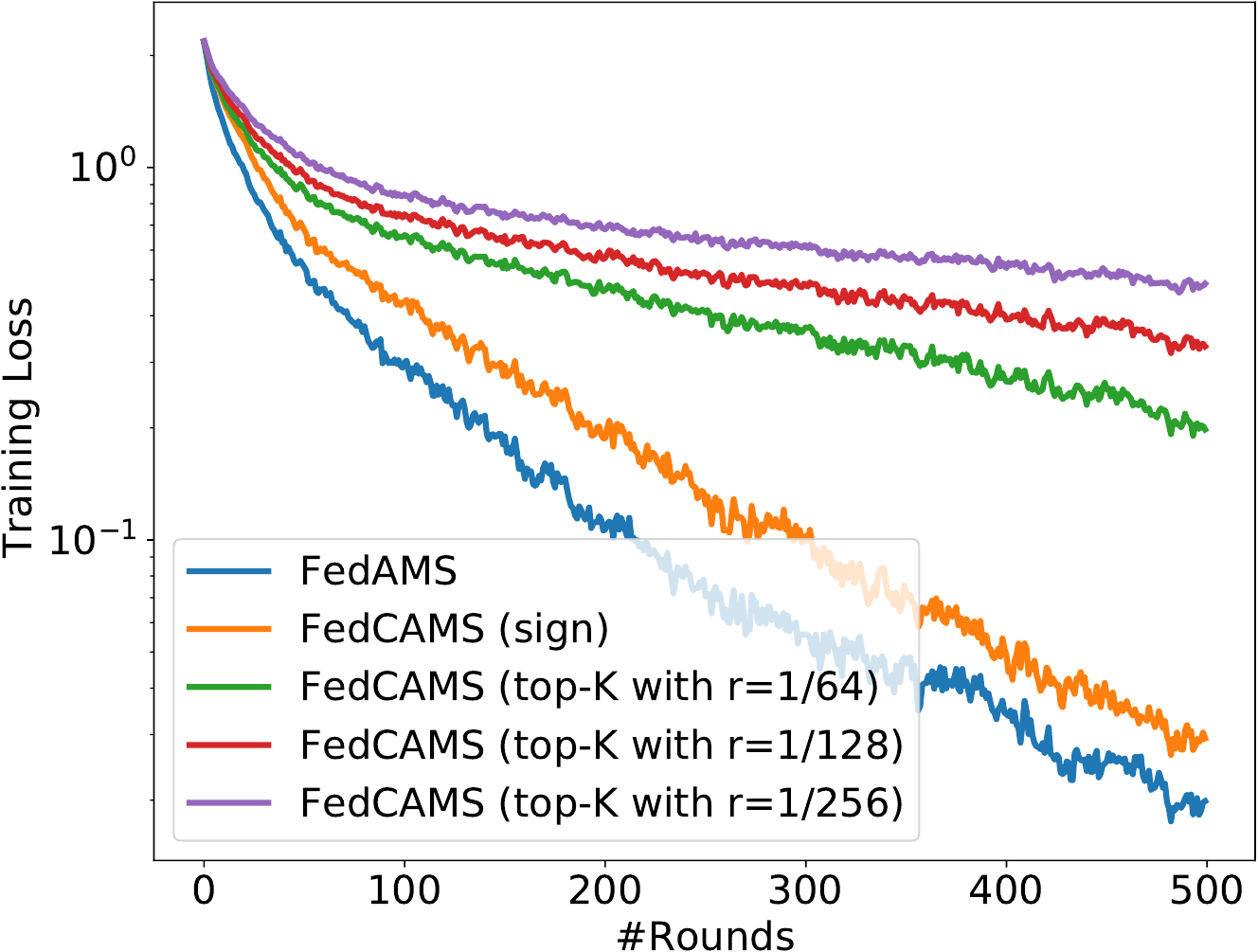}}
\subfigure[Training Loss]{\includegraphics[width=0.23\textwidth]{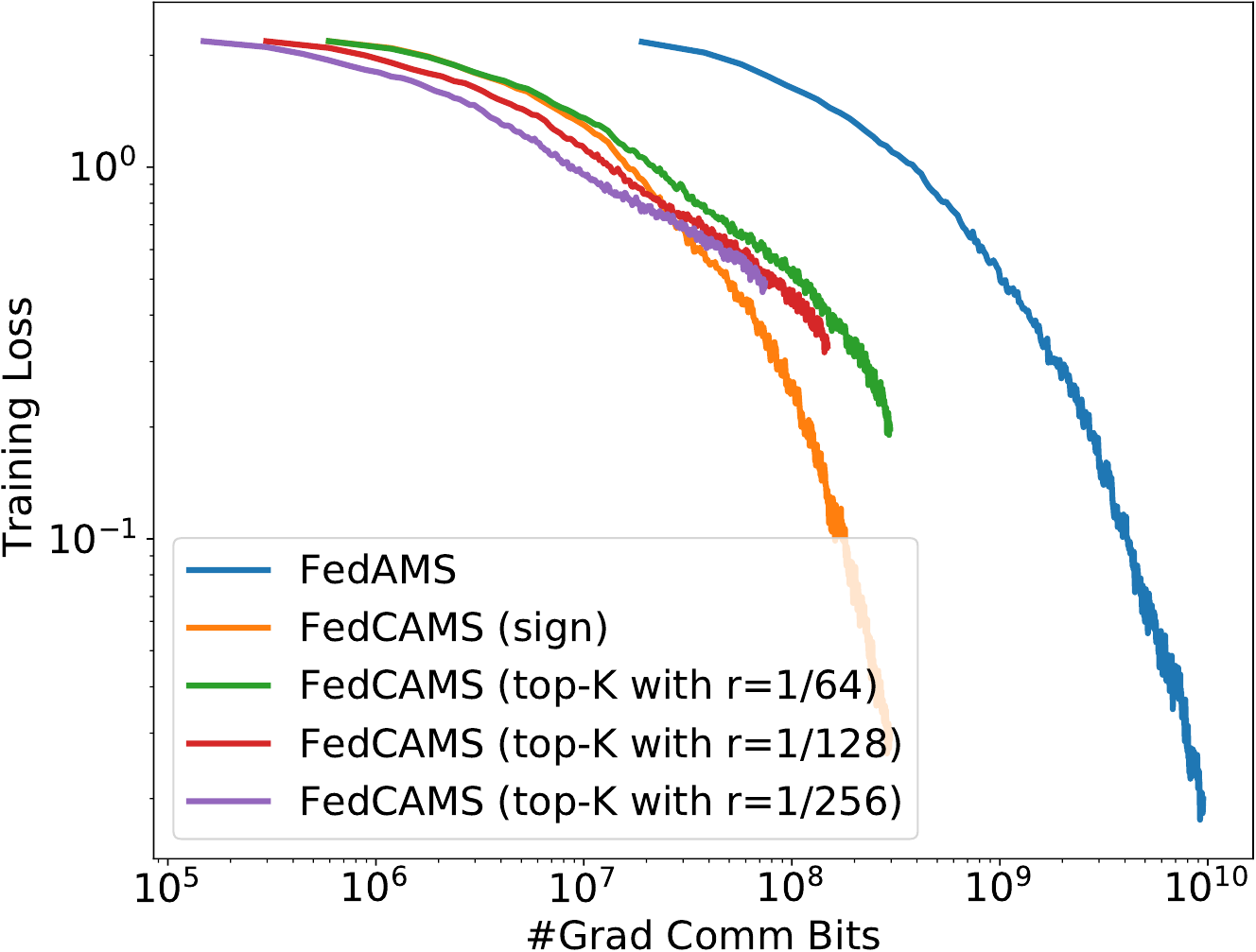}}
\subfigure[Test Accuracy]{\includegraphics[width=0.23\textwidth]{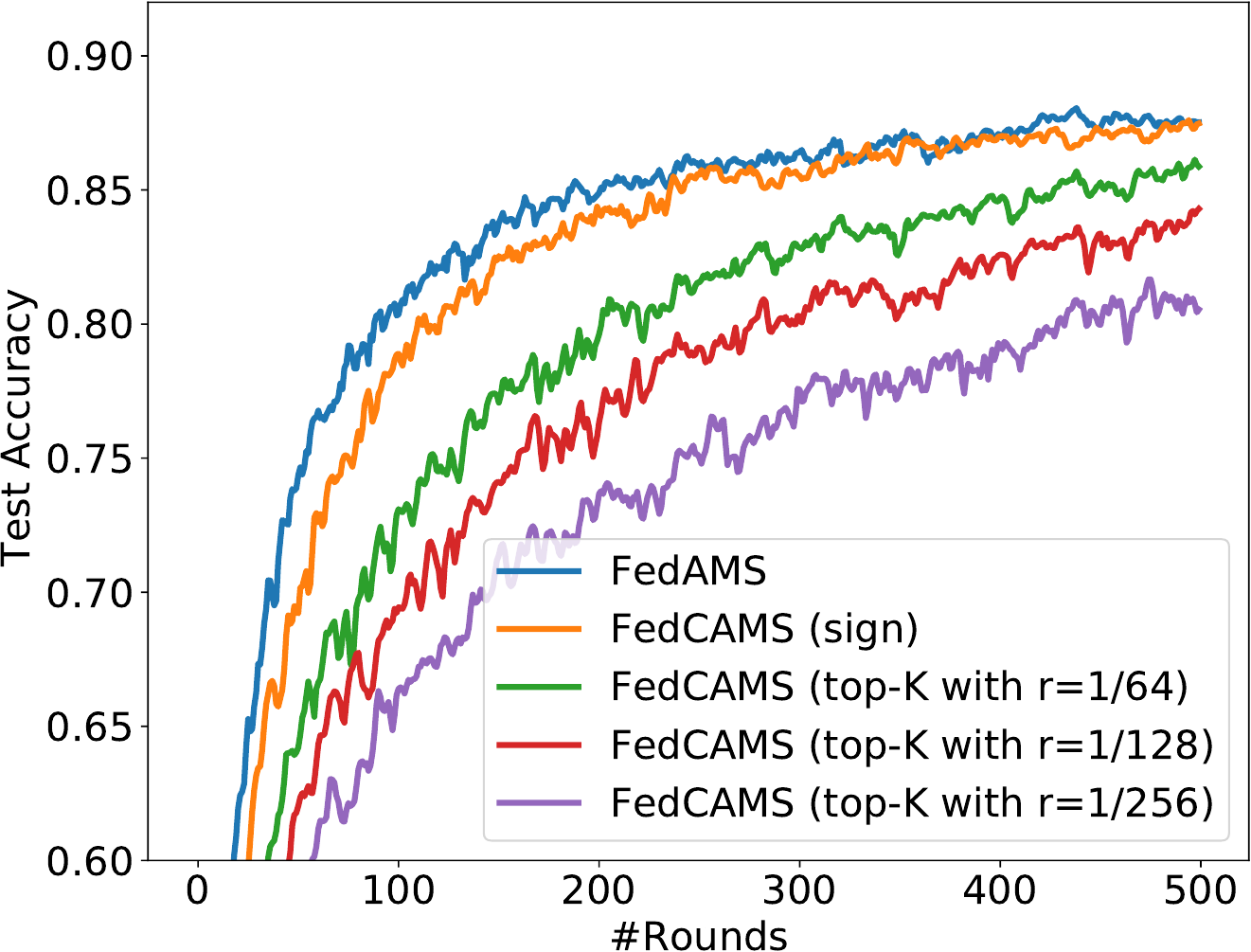}}
\subfigure[Test Accuracy]{\includegraphics[width=0.23\textwidth]{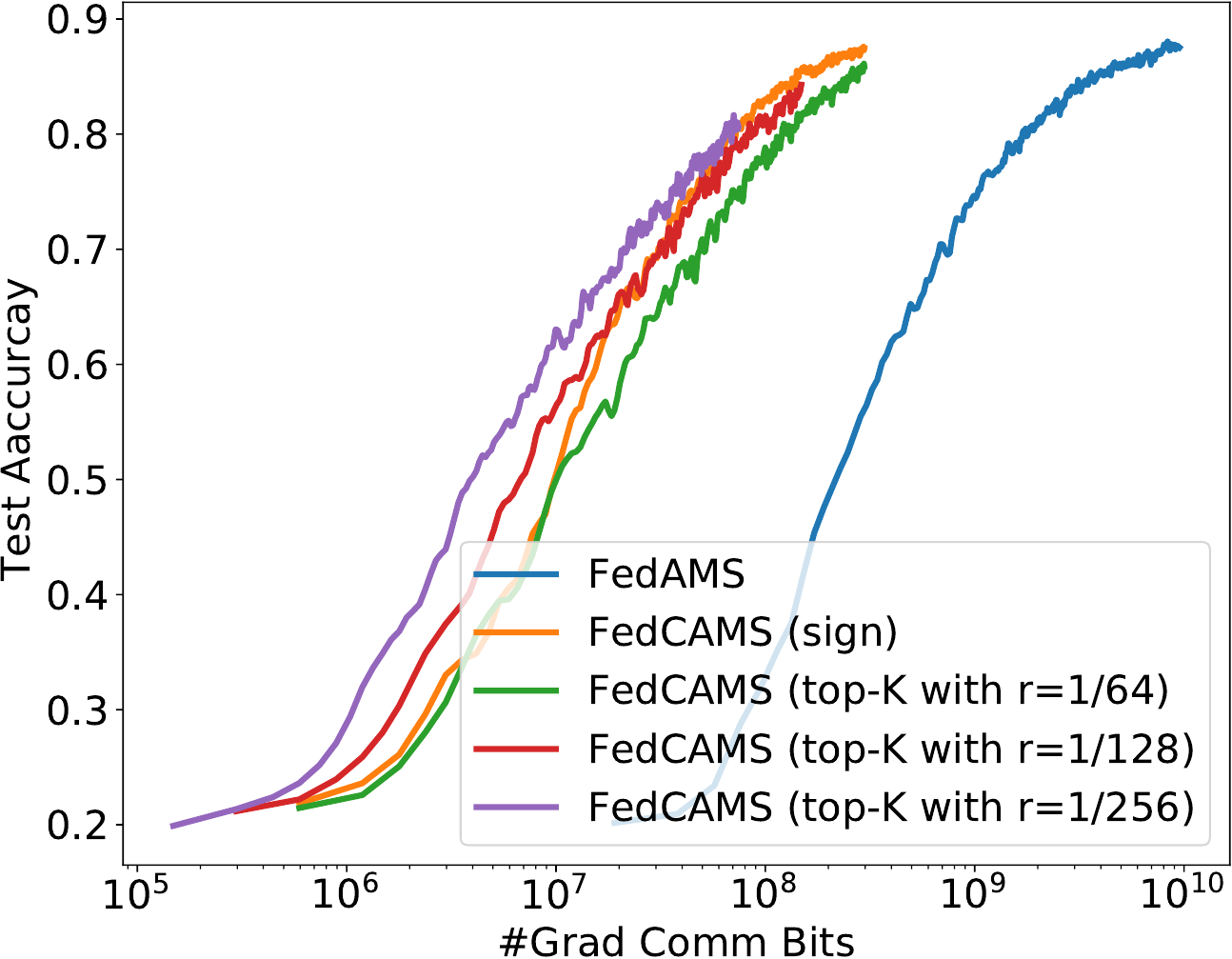}}

\caption{The learning curves for FedCAMS and  uncompressed FedAMS on training CIFAR-10 data on ConvMixer-256-8 model.}
\label{fig:cifar10_ef_convmixer}
\end{figure}

\vspace{-5pt}
\section{Conclusions and Future Work}
In this paper, we propose a communication-efficient compressed federated adaptive gradient optimization framework, FedCAMS, which largely reduces the communication overhead and addresses the adaptivity issue in federated optimization methods. FedCAMS is based on our proposed general adaptive federated optimization framework, FedAMS, which contains variants of FedAdam feature max stabilization mechanisms. We present an improved theoretical convergence analysis of adaptive federated optimization, based on which we prove that in the nonconvex stochastic optimization setting, our proposed FedCAMS achieves the same convergence rate as its uncompressed counterpart FedAMS with a few orders of magnitude less communication cost. Experiments on various benchmarks verified our theoretical results.

Our current analysis is limited to one-way communication compression from clients to the central server. However, extending our current analysis to two-way communication compression is highly non-trivial as it can be hard to guarantee the distributed global model from server to clients to stay synchronized due to error feedback and the biased compressor, especially in the partial participation setting. We leave it as future work.

\section*{Acknowledgements}
We thank the anonymous reviewers for their helpful comments.
This research was supported in part by a Seed Grant award from the Institute for Computational and Data Sciences at the Pennsylvania State University as well as Dell Technology AI Infrastructure Level Technologies Grant. 
The views and conclusions contained in this paper are those of the authors and should not be interpreted as representing any funding agencies.

\bibliography{ref}
\bibliographystyle{icml2022}

\newpage
\appendix
\onecolumn

\section{Proof in Section \ref{subsec:fedams}}
\subsection{Proof of Theorem \ref{thm:full-noncomp}} Similar to previous works studied adaptive methods \citep{chen2018convergence, zhou2018convergence, chen2020closing}, we introduce a Lyapunov sequence $\zb_t$: assume $\xb_0= \xb_1$, for each $t\geq 1$, we have
\begin{align}
    \zb_t = \xb_t + \frac{\beta_1}{1-\beta_1} (\xb_t-\xb_{t-1}) = \frac{1}{1-\beta_1} \xb_t - \frac{\beta_1}{1-\beta_1} \xb_{t-1}.
\end{align}
For the difference of sequence $\zb_t$, we have 
\begin{align*}
    \zb_{t+1} - \zb_t & = \frac{1}{1-\beta_1} (\xb_{t+1} - \xb_{t}) - \frac{\beta_1}{1-\beta_1} (\xb_{t} - \xb_{t-1}) \\
    & = \frac{1}{1-\beta_1}(\eta \Vbh_{t}^{-1/2} \mb_t) - \frac{\beta_1}{1-\beta_1} \eta \Vbh_{t-1}^{-1/2} \mb_{t-1} \\
    & = \frac{1}{1-\beta_1}\eta \Vbh_{t}^{-1/2} \bigg[\beta_1 \mb_{t-1} + (1-\beta_1) \Delta_t \bigg]- \frac{\beta_1}{1-\beta_1} \eta \Vbh_{t-1}^{-1/2} \mb_{t-1} \\
    & = \eta \Vbh_t^{-1/2} \Delta_{t} - \eta \frac{\beta_1}{1-\beta_1}\bigg(\Vbh_{t-1}^{-1/2} - \Vbh_{t}^{-1/2} \bigg)\mb_{t-1}.
\end{align*}
Since $f$ is $L$-smooth, taking conditional expectation at time $t$, we have
\begin{align}\label{eq:f-Lsmooth-noncomp}
& \EE[f(\zb_{t+1})]-f(\zb_{t}) \notag\\
& \leq \EE[\langle \nabla f(\zb_t), \zb_{t+1}-\zb_t \rangle] + \frac{L}{2} \EE[\|\zb_{t+1}-\zb_t\|^2] \notag\\
& \leq \EE\bigg[\bigg\langle\nabla f(\zb_{t}), \eta\Vbh_t^{-1/2} \Delta_{t}\bigg\rangle\bigg]-\EE\bigg[\bigg\langle\nabla f(\zb_{t}), \eta  \frac{\beta_{1}}{1-\beta_{1}} \bigg(\Vbh_{t-1}^{-1/2} - \Vbh_{t}^{-1/2}\bigg) \mb_{t-1}\bigg\rangle\bigg] \notag\\
& \quad +\frac{\eta^{2} L}{2} \EE\bigg[\bigg\|\Vbh_t^{-1/2} \Delta_{t}-\frac{\beta_{1}}{1-\beta_{1}} \bigg(\Vbh_{t-1}^{-1/2} - \Vbh_{t}^{-1/2}\bigg) \mb_{t-1}\bigg\|^{2}\bigg] \notag\\
& = \underbrace{\EE\bigg[\bigg\langle\nabla f(\xb_{t}), \eta\Vbh_t^{-1/2} \Delta_{t}\bigg\rangle\bigg]}_{I_1} \underbrace{-\eta \EE\bigg[\bigg\langle\nabla f(\zb_{t}), \frac{\beta_{1}}{1-\beta_{1}} \bigg(\Vbh_{t-1}^{-1/2} - \Vbh_{t}^{-1/2}\bigg) \mb_{t-1}\bigg\rangle\bigg]}_{I_2}\notag\\
& \quad + \underbrace{\frac{\eta^{2} L}{2} \EE\bigg[\bigg\|\Vbh_t^{-1/2} \Delta_{t}-\frac{\beta_{1}}{1-\beta_{1}} \bigg(\Vbh_{t-1}^{-1/2} - \Vbh_{t}^{-1/2}\bigg) \mb_{t-1}\bigg\|^{2}\bigg]}_{I_3}+  \underbrace{\EE\bigg[\bigg\langle\nabla f(\zb_{t})-\nabla f(\xb_t), \eta\Vbh_t^{-1/2} \Delta_{t}\bigg\rangle\bigg]}_{I_4},
\end{align}
here we recall the notation $\Vbh_t = \diag(\hat{\vb}_t) = \diag(\max(\hat{\vb}_{t-1}, \vb_t, \epsilon))$.

\textbf{Bounding $I_1$: }We have
\begin{align}\label{eq:I1-1}
    I_1 & = \EE\bigg[\bigg\langle\nabla f(\xb_{t}), \eta\frac{ \Delta_{t}}{\sqrt{\hat{\vb}_{t}}} \bigg\rangle\bigg] \notag\\
    & \leq \eta\EE\bigg[\bigg\langle\nabla f(\xb_{t}), \frac{\sqrt{2} \cdot \Delta_{t}}{\sqrt{\vb_{t}+\epsilon}} \bigg\rangle\bigg] \notag\\
    & = \sqrt{2} \eta\EE\bigg[\bigg\langle\nabla f(\xb_{t}), \frac{\Delta_{t}}{\sqrt{\beta_2 \vb_{t-1}+\epsilon}}\bigg\rangle\bigg] + \sqrt{2} \eta\EE\bigg[\bigg\langle\nabla f(\xb_{t}),\frac{ \Delta_{t}}{\sqrt{\vb_{t}+\epsilon}}- \frac{ \Delta_{t}}{\sqrt{\beta_2 \vb_{t-1}+\epsilon}}\bigg\rangle\bigg],
\end{align}
where the first inequality follows by the fact that $\hat{\vb}_t \geq \frac{\vb_t + \epsilon}{2} $. 
For the second term in \eqref{eq:I1-1}, we have 
\begin{align} \label{eq:I1-1-2}
    & \sqrt{2} \eta\EE\bigg[\bigg\langle\nabla f(\xb_{t}),\frac{\Delta_{t}}{\sqrt{\vb_{t}+\epsilon}}- \frac{\Delta_{t}}{\sqrt{\beta_2 \vb_{t-1}+\epsilon}}\bigg\rangle\bigg] \notag\\
    & \leq \sqrt{2} \eta \|\nabla f(\xb_t)\| \EE\bigg[\bigg\|\frac{1}{\sqrt{\vb_{t}+\epsilon}}- \frac{1}{\sqrt{\beta_2 \vb_{t-1}+\epsilon}}\bigg\| \cdot \|\Delta_t\| \bigg] \notag\\
    & \leq \frac{\eta \sqrt{2(1-\beta_2)} G}{\epsilon} \EE[\|\Delta_t\|^2],
\end{align}
where the second inequality follows from Lemma~\ref{lm:Wt} and \ref{lm:gmv-bound}, and we will further apply the bound for $\EE[\|\Delta_t\|^2]$ following from Lemma \ref{lm:EDelta}. For the first term in \eqref{eq:I1-1}, we have
\begin{align} \label{eq:I1-2}
    & \sqrt{2} \eta\EE\bigg[\bigg\langle\nabla f(\xb_{t}), \frac{ \Delta_{t}}{\sqrt{\beta_2 \vb_{t-1}+\epsilon}}\bigg\rangle\bigg] \notag\\
    & = \sqrt{2} \eta\EE\bigg[\bigg\langle \frac{\nabla f(\xb_t)}{\sqrt{\beta_2 \vb_{t-1}+\epsilon}} , \Delta_t + \eta_l K \nabla f(\xb_t) - \eta_l K \nabla f(\xb_t) \bigg\rangle\bigg] \notag\\
    & = - \sqrt{2} \eta\eta_l K\EE\bigg[ \bigg\|\frac{\nabla f(\xb_t)}{\sqrt[4]{\beta_2 \vb_{t-1} + \epsilon}} \bigg\|^2\bigg] + \sqrt{2} \eta\EE\bigg[\bigg\langle \frac{\nabla f(\xb_t)}{\sqrt{\beta_2 \vb_{t-1}+\epsilon}} , \Delta_t + \eta_l K \nabla f(\xb_t) \bigg\rangle\bigg] \notag\\
    & = - \sqrt{2} \eta\eta_l K\EE\bigg[ \bigg\|\frac{\nabla f(\xb_t)}{\sqrt[4]{\beta_2 \vb_{t-1} + \epsilon}} \bigg\|^2\bigg] + \sqrt{2}\eta \bigg\langle \frac{\nabla f(\xb_t)}{\sqrt{\beta_2 \vb_{t-1}+\epsilon}}, \EE \bigg[-\frac{1}{m}\sum_{i=1}^m \sum_{k=0}^{K-1} \eta_l \gb_{t,k}^i + \eta_l K \nabla f(\xb_t)\bigg] \bigg\rangle \notag\\
    & = - \sqrt{2} \eta\eta_l K\EE\bigg[ \bigg\|\frac{\nabla f(\xb_t)}{\sqrt[4]{\beta_2 \vb_{t-1} + \epsilon}} \bigg\|^2\bigg] + \sqrt{2}\eta \bigg\langle \frac{\nabla f(\xb_t)}{\sqrt{\beta_2 \vb_{t-1}+\epsilon}}, \EE \bigg[-\frac{\eta_l}{m}\sum_{i=1}^m \sum_{k=0}^{K-1}  \gb_{t,k}^i + \frac{\eta_l K}{m}\sum_{i=1}^m \nabla F_i(\xb_t)\bigg] \bigg\rangle,
\end{align}
where the third equality follows the local update rule. For the last term in \eqref{eq:I1-2}, we have 
\begin{align}
    & \sqrt{2} \eta \bigg\langle \frac{\nabla f(\xb_t)}{\sqrt{\beta_2 \vb_{t-1}+\epsilon}}, \EE \bigg[-\frac{\eta_l }{m}\sum_{i=1}^m \sum_{k=0}^{K-1} \gb_{t,k}^i + \frac{\eta_l K }{m}\sum_{i=1}^m \nabla F_i(\xb_t)\bigg] \bigg\rangle \notag\\
    & = \sqrt{2} \eta \bigg\langle \frac{\sqrt{\eta_l K}}{\sqrt[4]{\beta_2 \vb_{t-1} + \epsilon}} \nabla f(\xb_t),  - \frac{\sqrt{\eta_l K}}{K m} \frac{1}{\sqrt[4]{\beta_2 \vb_{t-1} + \epsilon}}\EE \bigg[\sum_{i=1}^m \sum_{k=0}^{K-1} (\nabla F_i(\xb_{t,k}^i) - \nabla F_i(\xb_t))\bigg] \bigg\rangle \notag\\
    & = \frac{\sqrt{2} \eta \eta_l K}{2} \EE\bigg[\bigg\|\frac{\nabla f(\xb_t)}{\sqrt[4]{\beta_2 \vb_{t-1} + \epsilon}} \bigg\|^2\bigg] + \frac{\sqrt{2} \eta\eta_l}{2 K m^2} \EE \bigg[\bigg\|\frac{1}{\sqrt[4]{\beta_2 \vb_{t-1} + \epsilon}}\sum_{i=1}^m \sum_{k=0}^{K-1} (\nabla F_i(\xb_{t,k}^i) - \nabla F_i(\xb_t))\bigg\|^2\bigg] \notag\\
    & \qquad - \frac{\sqrt{2} \eta\eta_l}{2 K m^2} \EE\bigg[ \bigg\|\frac{1}{\sqrt[4]{\beta_2 \vb_{t-1} + \epsilon}}\sum_{i=1}^m \sum_{k=0}^{K-1} \nabla F_i(\xb_{t,k}^i)\bigg\|^2\bigg] \notag\\
    & \leq \frac{\sqrt{2} \eta \eta_l K}{2} \EE\bigg[\bigg\|\frac{\nabla f(\xb_t)}{\sqrt[4]{\beta_2 \vb_{t-1} + \epsilon}} \bigg\|^2\bigg] + \frac{\sqrt{2} \eta\eta_l}{2m}\sum_{i=1}^m \sum_{k=0}^{K-1}\EE \bigg[\bigg\| \frac{\nabla F_i(\xb_{t,k}^i) - \nabla F_i(\xb_t)}{\sqrt[4]{\beta_2 \vb_{t-1} + \epsilon}} \bigg\|^2\bigg] \notag\\
    & \qquad - \frac{\sqrt{2} \eta\eta_l}{2 K m^2} \EE\bigg[ \bigg\|\frac{1}{\sqrt[4]{\beta_2 \vb_{t-1} + \epsilon}} \sum_{i=1}^m \sum_{k=0}^{K-1} \nabla F_i(\xb_{t,k}^i)\bigg\|^2\bigg] \notag\\ 
    & \leq \frac{\sqrt{2} \eta \eta_l K}{2} \EE\bigg[\bigg\|\frac{\nabla f(\xb_t)}{\sqrt[4]{\beta_2 \vb_{t-1} + \epsilon}} \bigg\|^2\bigg] + \frac{\sqrt{2} \eta\eta_l L^2}{2m}\sum_{i=1}^m \sum_{k=0}^{K-1}\EE\bigg[ \bigg\| \frac{\xb_{t,k}^i - \xb_t}{\sqrt[4]{\beta_2 \vb_{t-1} + \epsilon}} \bigg\|^2\bigg] \notag\\
    & \qquad - \frac{\sqrt{2} \eta\eta_l}{2 K m^2} \EE\bigg[ \bigg\|\frac{1}{\sqrt[4]{\beta_2 \vb_{t-1} + \epsilon}}\sum_{i=1}^m \sum_{k=0}^{K-1} \nabla F_i(\xb_{t,k}^i)\bigg\|^2\bigg],
\end{align}
where the second equation follows from $\langle \xb,\yb\rangle = \frac{1}{2}[\|\xb\|^2 + \|\yb\|^2 - \|\xb-\yb\|^2]$, the first inequality holds by applying Cauchy-Schwarz inequality, the second inequality follows from Assumption~\ref{as:smooth}. 

Hence by applying Lemma~\ref{lm:xikt-xt} with the local learning rate condition: $\eta_l \leq \frac{1}{8KL}$, we have
\begin{align}
    & \sqrt{2} \cdot \eta \bigg\langle \frac{\nabla f(\xb_t)}{\sqrt{\beta_2 \vb_{t-1}+\epsilon}}, \EE \bigg[-\frac{\eta_l}{m}\sum_{i=1}^m \sum_{k=0}^{K-1} \gb_{t,k}^i + \frac{\eta_l K}{m}\sum_{i=1}^m \nabla F_i(\xb_t)\bigg] \bigg\rangle \notag\\
    & \leq \frac{3\sqrt{2} \eta \eta_l K}{4} \EE\bigg[\bigg\|\frac{\nabla f(\xb_t)}{\sqrt[4]{\beta_2 \vb_{t-1} + \epsilon}} \bigg\|^2\bigg] + \frac{5\eta\eta_l^3 K^2 L^2}{\sqrt{2\epsilon}} (\sigma_l^2+6K \sigma_g^2) - \frac{\sqrt{2} \eta\eta_l}{2 K m^2} \EE \bigg[\bigg\|\frac{1}{\sqrt[4]{\beta_2 \vb_{t-1} + \epsilon}}\sum_{i=1}^m \sum_{k=0}^{K-1} \nabla F_i(\xb_{t,k}^i)\bigg\|^2\bigg].
\end{align}
Then merging pieces together, we have
\begin{align}\label{eq:I1}
    I_1 & \leq -\frac{\sqrt{2} \eta\eta_l K}{4} \EE\bigg[ \bigg\|\frac{\nabla f(\xb_t)}{\sqrt[4]{\beta_2 \vb_{t-1} + \epsilon}} \bigg\|^2\bigg] + \frac{5\eta\eta_l^3 K^2 L^2}{\sqrt{2 \epsilon}} (\sigma_l^2+6K \sigma_g^2) \notag\\
    & \quad - \frac{\sqrt{2} \eta\eta_l}{2 K m^2} \EE \bigg[ \bigg\|\frac{1}{\sqrt[4]{\beta_2 \vb_{t-1} + \epsilon}}\sum_{i=1}^m \sum_{k=0}^{K-1} \nabla F_i(\xb_{t,k}^i)\bigg\|^2\bigg] + \frac{\eta \sqrt{2(1-\beta_2)} G}{\epsilon} \EE[\|\Delta_t\|^2] \notag\\
    & \leq -\frac{\eta\eta_l K}{4} \EE\bigg[ \bigg\|\frac{\nabla f(\xb_t)}{\sqrt[4]{\beta_2 \vb_{t-1} + \epsilon}} \bigg\|^2\bigg] + \frac{5\eta\eta_l^3 K^2 L^2}{\sqrt{2 \epsilon}} (\sigma_l^2+6K \sigma_g^2)\notag\\
    & \quad - \frac{\eta\eta_l}{2 K m^2} \EE \bigg[ \bigg\|\frac{1}{\sqrt[4]{\beta_2 \vb_{t-1} + \epsilon}}\sum_{i=1}^m \sum_{k=0}^{K-1} \nabla F_i(\xb_{t,k}^i)\bigg\|^2\bigg] + \frac{\eta \sqrt{2(1-\beta_2)} G}{\epsilon} \EE[\|\Delta_t\|^2].
\end{align}
\textbf{Bounding $I_2$:} The bound for $I_2$ mainly follows by the update rule and definition of virtual sequence $\zb_t$,
\begin{align}\label{eq:I2}
    I_2 & = -\eta \EE\bigg[\bigg\langle\nabla f(\zb_{t}), \frac{\beta_{1}}{1-\beta_{1}} \bigg(\Vbh_{t-1}^{-1/2} - \Vbh_{t}^{-1/2}\bigg) \mb_{t-1}\bigg\rangle\bigg] \notag\\
    & \quad = -\eta \EE\bigg[\bigg\langle\nabla f(\zb_t) - \nabla f(\xb_{t}) + \nabla f(\xb_{t}), \frac{\beta_{1}}{1-\beta_{1}} \bigg(\Vbh_{t-1}^{-1/2} - \Vbh_{t}^{-1/2}\bigg) \mb_{t-1} \bigg\rangle\bigg] \notag\\
    & \leq \eta \EE\bigg[\|\nabla f(\xb_{t})\|\bigg
    \|\frac{\beta_{1}}{1-\beta_{1}} \bigg(\Vbh_{t-1}^{-1/2} - \Vbh_{t}^{-1/2}\bigg) \mb_{t-1}\bigg\|\bigg] \notag\\
    & \quad + \eta^2 L \EE \bigg[\bigg\|\frac{\beta_1}{1-\beta_1} \Vbh_{t-1}^{-1/2} \mb_{t-1}\bigg\| \bigg\|\frac{\beta_{1}}{1-\beta_{1}} \bigg(\Vbh_{t-1}^{-1/2} - \Vbh_{t}^{-1/2}\bigg) \mb_{t-1} \bigg\|\bigg] \notag \\
    & \leq \eta \frac{\beta_1}{1-\beta_1} \eta_l K G^2 \EE\bigg[\Big\|\Vbh_{t-1}^{-1/2} - \Vbh_{t}^{-1/2}\Big\|_1\bigg] + \eta^2  \frac{\beta_1^2}{(1-\beta_1)^2} L \eta_l^2 K^2 G^2 \epsilon^{-1/2} \EE\bigg[\Big\|\Vbh_{t-1}^{-1/2} - \Vbh_{t}^{-1/2}\Big\|_1\bigg],
\end{align}
where the last inequality holds by applying Lemma \ref{lm:gmv-bound} and the fact of $\hat{\vb}_{t-1} \geq \epsilon$.

\textbf{Bounding $I_3$:} It can be bounded as follows:
\begin{align}\label{eq:I3}
    I_3 & = \frac{\eta^2 L }{2} \EE\bigg[ \bigg\|\Vbh_t^{-1/2} \Delta_{t} + \frac{\beta_{1}}{1-\beta_{1}} \bigg(\Vbh_{t-1}^{-1/2} - \Vbh_{t}^{-1/2}\bigg) \mb_{t-1} \bigg\|^2\bigg] \notag\\
    & \leq \eta^2 L \EE\bigg[ \big\|\Vbh_t^{-1/2} \Delta_{t}\big\|^2\bigg]+ \eta^2 L \EE \bigg[ \bigg\|\frac{\beta_{1}}{1-\beta_{1}} \bigg(\Vbh_{t-1}^{-1/2} - \Vbh_{t}^{-1/2}\bigg) \mb_{t-1} \bigg\|^2\bigg] \notag\\
    & \leq \eta^2 L \EE  \bigg[\big\|\Vbh_t^{-1/2} \Delta_{t}\big\|^2 \bigg]+ \eta^2 L  \frac{\beta_1^2}{(1-\beta_1)^2} \eta_l^2 K^2 G^2 \EE \bigg [\Big\|\Vbh_{t-1}^{-1/2} - \Vbh_{t}^{-1/2}\Big\|^2\bigg],
\end{align}
where the first inequality follows by Cauchy-Schwarz inequality, and the second one follows by Lemma \ref{lm:gmv-bound}.

\textbf{Bounding $I_4$: }
\begin{align}
    I_4 & = \EE\bigg[\bigg\langle\nabla f(\zb_{t})-\nabla f(\xb_t), \eta \Vbh_t^{-1/2} \Delta_{t}\bigg\rangle\bigg] \notag\\
    & \leq \EE\bigg[\|\nabla f(\zb_t) - \nabla f(\xb_t)\| \big\|\eta \Vbh_t^{-1/2} \Delta_{t}\big\| \bigg] \notag\\
    & \leq L \EE\bigg[ \|\zb_t - \xb_t\|\big\|\eta \Vbh_t^{-1/2} \Delta_{t}\big\| \bigg] \notag\\
    & \leq \frac{\eta^2 L}{2} \EE \bigg[\big\|\Vbh_t^{-1/2} \Delta_{t}\big\|^2\bigg] + \frac{\eta^2 L}{2} \EE\bigg[\bigg\|\frac{\beta_1}{1-\beta_1} \Vbh_{t-1}^{-1/2} \mb_{t-1}\bigg\|^2 \bigg],\notag
\end{align}
where the first inequality holds by the fact of $\langle \ab, \bb \rangle \leq \|\ab\|\|\bb\|$, the second one follows from Assumption \ref{as:smooth} and the third one holds by the definition of virtual sequence $\zb_t$ and the fact of $\|\ab\|\|\bb\| \leq \frac{1}{2}\|\ab\|^2 + \frac{1}{2}\|\bb\|^2$. Then summing $I_4$ over $t=1,\cdots, T$, we have
\begin{align}\label{eq:I4-1}
    \sum_{t=1}^T I_4 & \leq \frac{\eta^2 L}{2\epsilon} \sum_{t=1}^T \EE [\|\Delta_{t}\|^2] + \frac{\eta^2 L}{2\epsilon}\sum_{t=1}^T  \EE\bigg[\bigg\|\frac{\beta_1}{1-\beta_1} \mb_t \bigg\|^2 \bigg] \notag\\
    & \leq \frac{\eta^2 L}{2\epsilon} \sum_{t=1}^T \EE[\|\Delta_t\|^2] + \frac{\eta^2 L}{2 \epsilon} \frac{\beta_1^2}{(1-\beta_1)^2} \sum_{t=1}^T \EE [\|\mb_t\|^2].
\end{align}
By Lemma \ref{lm:mE-bound}, we have
\begin{align}
    \sum_{t=1}^T \EE[\|\mb_t\|^2] \leq \frac{T K \eta_l^2}{m} \sigma_l^2 + \frac{\eta_l^2}{m^2} \sum_{t=1}^T \EE \bigg[\bigg\|\sum_{i=1}^m \sum_{k=0}^{K-1}\nabla F_i(\xb_{t,k}^i ) \bigg\|^2\bigg].\notag
\end{align}
Therefore, the summation of $I_4$ term is bounded by 
\begin{align}\label{eq:I4-2}
    \sum_{t=1}^T I_4 & \leq \frac{\eta^2 L}{2\epsilon}  \sum_{t=1}^T \EE [\|\Delta_t\|^2]  +  \frac{\beta_1^2}{(1-\beta_1)^2} \frac{\eta^2 L}{2\epsilon} \frac{\eta_l^2}{m^2} \sum_{t=1}^T \EE \bigg[\bigg\|\sum_{i=1}^m \sum_{k=0}^{K-1}\nabla F_i(\xb_{t,k}^i ) \bigg\|^2\bigg] \notag\\
    & \quad +  \frac{\beta_1^2}{(1-\beta_1)^2} \frac{\eta^2 L}{2\epsilon} \frac{TK \eta_l^2}{m} \sigma_l^2.
\end{align}
\\\textbf{Merging pieces together:}
Substituting \eqref{eq:I1}, \eqref{eq:I2} and \eqref{eq:I3} into \eqref{eq:f-Lsmooth-noncomp}, summing over from $t=1$ to $T$ and then adding \eqref{eq:I4-1}, we have
\begin{align}\label{eq:fzI-1}
    & \EE[f(\zb_{T+1})]-f(\zb_1) = \sum_{t=1}^{T} [I_1+I_2+I_3+I_4] \notag\\
    & \leq -\frac{\eta\eta_l K}{4} \sum_{t=1}^{T}\EE\bigg[ \bigg\|\frac{\nabla f(\xb_t)}{\sqrt[4]{\beta_2 \vb_{t-1} + \epsilon}} \bigg\|^2\bigg]+ \frac{5\eta\eta_l^3 K^2 L^2 T}{\sqrt{2\epsilon}} (\sigma_l^2+6K \sigma_g^2) + \frac{\sqrt{2(1-\beta_2)} \eta G}{\epsilon} \sum_{t=1}^{T}\EE[\|\Delta_t\|^2]  \notag\\
    & \quad  - \frac{\eta\eta_l}{2 K m^2} \sum_{t=1}^{T} \EE\bigg[ \bigg\|\frac{1}{\sqrt[4]{\beta_2 \vb_{t-1} + \epsilon}}\sum_{i=1}^m \sum_{k=0}^{K-1} \nabla F_i(\xb_{t,k}^i))\bigg\|^2\bigg] \notag\\
    & \quad + \frac{\beta_1}{1-\beta_1} \eta\eta_l K G^2\sum_{t=1}^{T}\EE\bigg[\Big\|\Vbh_{t-1}^{-1/2} - \Vbh_{t}^{-1/2}\Big\|_1\bigg] + \frac{\beta_1^2}{(1-\beta_1)^2} \frac{\eta^2\eta_l^2 K^2 G^2}{\sqrt{\epsilon}} \sum_{t=1}^{T}\EE\bigg[\Big\|\Vbh_{t-1}^{-1/2} - \Vbh_{t}^{-1/2}\Big\|_1\bigg] \notag\\
    & \quad + \frac{\beta_1^2}{(1-\beta_1)^2} \eta^2\eta_l^2K^2 L G^2 \sum_{t=1}^{T}\EE\bigg[\Big\|\Vbh_{t-1}^{-1/2} - \Vbh_{t}^{-1/2}\Big\|^2\bigg] + \eta^2 L \sum_{t=1}^{T} \EE \bigg[\big\|\Vbh_t^{-1/2} \Delta_{t}\big\|^2 \bigg] \notag\\
    & \quad + \frac{\eta^2 L}{2\epsilon} \sum_{t=1}^{T} \EE [\|\Delta_{t}\|^2]  + \frac{\eta^2 L}{2\epsilon} \frac{\beta_1^2}{(1-\beta_1)^2} \sum_{t=1}^{T} \EE[\|\mb_t\|^2].
\end{align}
By applying Lemma \ref{lm:EDelta} into all terms containing the second moment estimate of model difference $\Delta_t$ in \eqref{eq:fzI-1}, and using the fact that $(\sqrt{\beta_2 K^2 G^2 + \epsilon})^{-1} \|\xb\| \leq(\sqrt{\beta_2 \eta_l^2 K^2 G^2 + \epsilon})^{-1}\|\xb\| \leq \big\|\frac{\xb}{\sqrt{\beta_2 \vb_t + \epsilon}}\big\| \leq \epsilon^{-1/2} \|\xb\|$, we have
\begin{align} \label{eq:fzI-2}
    & \EE[f(\zb_{T+1})]-f(\zb_1) \notag\\
    & \leq -\frac{\eta\eta_l K}{4} \sum_{t=1}^{T} \EE\bigg[\bigg\|\frac{\nabla f(\xb_t)}{\sqrt[4]{\beta_2 \vb_{t-1} + \epsilon}} \bigg\|^2\bigg]+ \frac{5\eta\eta_l^3 K^2 L^2 T}{\sqrt{2\epsilon}} (\sigma_l^2+6K \sigma_g^2)  + \frac{\beta_1}{1-\beta_1} \frac{\eta \eta_l K G^2 d}{\sqrt{\epsilon}} \notag\\
    & \quad + \frac{\beta_1^2}{(1-\beta_1)^2} \frac{2 \eta^2 \eta_l^2 K^2 L G^2 d}{\epsilon} - \frac{\eta\eta_l}{2 K m^2} \sum_{t=1}^{T} \EE \bigg[ \bigg\|\frac{1}{\sqrt[4]{\beta_2 \vb_{t-1} + \epsilon}} \sum_{i=1}^m \sum_{k=0}^{K-1} \nabla F_i(\xb_{t,k}^i))\bigg\|^2\bigg] \notag\\
    & \quad + \bigg(\eta^2 L + \frac{\eta^2 L}{2} + \sqrt{2(1-\beta_2)} \eta G \bigg) \bigg[ \frac{K T \eta_l^2}{m\epsilon} \sigma_l^2 + \frac{ \eta_l^2}{m^2 \epsilon} \sum_{t=1}^{T} \EE \bigg[\bigg\|\sum_{i=1}^m \sum_{k=0}^{K-1} \nabla F_i(\xb_{t,k}^i)\bigg\|^2 \bigg] \notag\\
    & \quad + \frac{\beta_1^2}{(1-\beta_1)^2}\frac{\eta^2 L}{2\epsilon}\frac{\eta_l^2}{m^2} \sum_{t=1}^{T} \EE \bigg[\bigg\|\sum_{i=1}^m \sum_{k=0}^{K-1} \nabla F_i(\xb_{t,k}^i)\bigg\|^2 \bigg] + \frac{\beta_1^2}{(1-\beta_1)^2} \frac{\eta^2 L}{2\epsilon}\frac{T K \eta_l^2}{m}\sigma_l^2  \notag\\
    & \leq -\frac{\eta\eta_l K}{4 \sqrt{\beta_2 \eta_l^2 K^2 G^2 + \epsilon}} \sum_{t=1}^{T} \EE[\|\nabla f(\xb_t)\|^2]+ \frac{5\eta\eta_l^3 K^2 L^2 T}{\sqrt{2\epsilon}} (\sigma_l^2+6K \sigma_g^2)  \notag\\
    & \quad + \frac{\beta_1}{1-\beta_1} \frac{\eta \eta_l K G^2 d}{\sqrt{\epsilon}} + \frac{\beta_1^2}{(1-\beta_1)^2}\frac{2\eta^2 \eta_l^2 K^2 L G^2 d}{\epsilon} \notag\\
    & \quad + \bigg(\eta^2 L + \frac{\eta^2 L}{2} + \sqrt{2(1-\beta_2)} \eta G + \frac{\beta_1^2}{(1-\beta_1)^2} \frac{\eta^2 L}{2} \bigg) \frac{KT\eta_l^2}{m \epsilon} \sigma_l^2 - \sum_{t=1}^{T} \EE \bigg[\bigg\|\sum_{i=1}^m \sum_{k=0}^{K-1} \nabla F_i(\xb_{t,k}^i)\bigg\|^2 \bigg] \notag\\
    & \quad \cdot \bigg[\frac{\eta\eta_l}{2\sqrt{\beta_2 K^2 G^2 + \epsilon} K m^2} -\bigg(\eta^2 L + \frac{\eta^2 L}{2} + \eta \sqrt{2(1-\beta_2)} G + \frac{\beta_1^2}{(1-\beta_1)^2} \frac{\eta^2 L}{2} \bigg)\frac{\eta_l^2}{m^2\epsilon}\bigg] \notag\\
    & \leq -\frac{\eta\eta_l K}{4 \sqrt{\beta_2 \eta_l^2 K^2 G^2 + \epsilon}} \sum_{t=1}^{T} \EE[\|\nabla f(\xb_t)\|^2]+ \frac{5\eta\eta_l^3 K^2 L^2 T}{\sqrt{2\epsilon}} (\sigma_l^2+6K \sigma_g^2)  \notag\\
    & \quad + \frac{\beta_1}{1-\beta_1} \frac{\eta \eta_l K G^2 d}{\sqrt{\epsilon}} + \frac{\beta_1^2}{(1-\beta_1)^2}\frac{2 \eta^2 \eta_l^2 K^2 L G^2 d}{\epsilon} \notag\\
    & \quad + \bigg(\eta^2 L + \frac{\eta^2 L}{2} + \eta \sqrt{2(1-\beta_2)}G +\frac{\beta_1^2}{(1-\beta_1)^2} \frac{\eta^2 L}{2} \bigg) \frac{KT\eta_l^2}{m \epsilon} \sigma_l^2.
\end{align}
The last inequality holds due to additional constraint of local learning rate $\eta_l$ with the inequality $\frac{\eta\eta_l}{2 \sqrt{\beta_2 K^2 G^2 + \epsilon} K m^2} - (\frac{3 \eta^2 L}{2} + \eta \sqrt{2(1-\beta_2)} G + \frac{\beta_1^2}{(1-\beta_1)^2} \frac{\eta^2 L}{2} )\frac{\eta_l^2}{m^2\epsilon} \geq 0$, thus we obtain the constraint $\eta_l \leq \frac{\epsilon}{K \sqrt{\beta_2 K^2 G^2 + \epsilon} [(3+C_1^2)\eta L+2\sqrt{2(1-\beta_2)}G]}$. Hence we have
\begin{align}\label{eq:fzI-3}
    & \frac{\eta\eta_l K}{4 \sqrt{\beta_2 \eta_l^2 K^2 G^2 + \epsilon} \cdot T} \sum_{t=1}^T \EE[\|\nabla f(\xb_t)\|^2] \notag\\
    & \leq \frac{f(\zb_0)-\EE [f(\zb_T)]}{T} + \frac{5\eta\eta_l^3 K^2 L^2}{\sqrt{2\epsilon}} (\sigma_l^2+6K \sigma_g^2) + [(3 + C_1^2) \eta^2 L + 2\sqrt{2(1-\beta_2)} \eta G ] \frac{K\eta_l^2}{2m\epsilon} \sigma_l^2 \notag\\
    & \quad + \frac{C_1 \eta \eta_l K G^2 d}{T \sqrt{\epsilon}} + \frac{2  C_1^2  \eta^2 \eta_l^2 K^2 L G^2 d}{T \epsilon}.
\end{align}
Therefore, 
\begin{align}
    \min \EE[\|\nabla f(\xb_t)\|^2] & \leq 4 \sqrt{\beta_2 \eta_l^2 K^2 G^2 + \epsilon}\cdot\bigg[\frac{f_0-f_*}{\eta \eta_l K T} + \frac{\Psi}{T} + \Phi \bigg],
\end{align}
where $\Psi = \frac{C_1 G^2 d}{\sqrt{\epsilon}}+ \frac{2C_1^2 \eta\eta_l K L G^2 d}{\epsilon}, \Phi = \frac{5 \eta_l^2 K L^2}{\sqrt{2 \epsilon}} (\sigma_l^2+6K \sigma_g^2)+ [(3 + C_1^2) \eta L + 2\sqrt{2(1-\beta_2)} G ] \frac{\eta_l}{2m\epsilon} \sigma_l^2$, where $C_1 = \frac{\beta_1}{1-\beta_1}$. 

\subsection{Proof of Corollary \ref{cor:fedams-full}}

If we pick $\eta_l= \Theta( \frac{1}{\sqrt{T}K})$ and $\eta = \Theta(\sqrt{Km})$, we have $\min_{t \in[T]} \EE[\|\nabla f(\xb_t)\|^2] = \cO(\frac{1}{\sqrt{TKm}})$.


\subsection{Proof of Theorem \ref{thm:part-fedams}} 

\textit{Notations and equations:} For partial participation, i.e. $ |\cS_t|=n, \forall t\in[T]$. The global model difference is the average of local model difference from the subset $\cS_t$, i.e., $\Delta_t = \frac{1}{n} \sum_{i \in \cS_t} \Delta_i^t$. Denote $ \Bar{\Delta}_t = \frac{1}{m} \sum_{i=1}^m \Delta_i^t$, and for convenience, we follow the previous notation of $\Vbh_t = \diag (\hat{\vb}_t + \epsilon)$. Next we show that the global model difference $\Delta_t$ is an unbiased estimator of $\Bar{\Delta}_t$:
\begin{align}\label{eq: Delta-unbiased}
    \EE_{\cS_t}[\Delta_t] = \frac{1}{n} \EE_{\cS_t}[\sum_{i=1}^n \Delta_t^{w_i}] = \EE_{\cS_t}[\Delta_t^{w_1}] = \frac{1}{m} \sum_{i=1}^m \Delta_t^i = \Bar{\Delta}_t.
\end{align}
Define the virtual sequence $\zb_t$ same as previous: assume $\xb_0 = \xb_1$, for each $t \geq 1$, we have 
\begin{align}
    & \zb_t = \xb_t + \frac{\beta_1}{1-\beta_1} (\xb_t - \xb_{t-1}) = \frac{1}{1-\beta_1} \xb_t - \frac{\beta_1}{1-\beta_1} \xb_{t-1}, \\
    & \zb_{t+1} - \zb_{t} = \eta \Vbh_t^{-1/2} \Delta_t - \eta \frac{\beta_1}{1-\beta_1}\bigg(\Vbh_{t-1}^{-1/2} - \Vbh_t^{-1/2}\bigg) \mb_{t-1}.
\end{align}
By Assumption \ref{as:smooth}, we have 
\begin{align}\label{eq:f-Lsmooth-noncomp'}
& \EE[f(\zb_{t+1})]-f(\zb_{t}) \notag\\
& \quad \leq \EE\bigg[\bigg\langle\nabla f(\zb_{t}), \eta \Vbh_t^{-1/2} \Delta_{t}\bigg\rangle\bigg]-\EE\bigg[\bigg\langle\nabla f(\zb_{t}), \eta \frac{\beta_{1}}{1-\beta_{1}} \bigg(\Vbh_{t-1}^{-1/2} - \Vbh_{t}^{-1/2}\bigg) \mb_{t-1}\bigg\rangle\bigg] \notag\\
& \quad +\frac{\eta^{2} L}{2} \EE\bigg[\bigg\|\Vbh_t^{-1/2} \Delta_{t}-\frac{\beta_{1}}{1-\beta_{1}} \bigg(\Vbh_{t-1}^{-1/2} - \Vbh_{t}^{-1/2}\bigg) \mb_{t-1}\bigg\|^{2}\bigg] \notag\\
& = \underbrace{\EE\bigg[\bigg\langle\nabla f(\xb_{t}), \eta\Vbh_t^{-1/2} \Delta_{t}\bigg\rangle\bigg]}_{I'_1} \underbrace{-\eta \EE\bigg[\bigg\langle\nabla f(\zb_{t}), \frac{\beta_{1}}{1-\beta_{1}} \bigg(\Vbh_{t-1}^{-1/2} - \Vbh_{t}^{-1/2}\bigg) \mb_{t-1}\bigg\rangle\bigg]}_{I'_2}\notag\\
& \quad + \underbrace{\frac{\eta^{2} L}{2} \EE\bigg[\bigg\|\Vbh_t^{-1/2} \Delta_{t}-\frac{\beta_{1}}{1-\beta_{1}} \bigg(\Vbh_{t-1}^{-1/2} - \Vbh_{t}^{-1/2}\bigg) \mb_{t-1}\bigg\|^{2}\bigg]}_{I'_3}+  \underbrace{\EE\bigg[\bigg\langle\nabla f(\zb_{t})-\nabla f(\xb_t), \eta\Vbh_t^{-1/2} \Delta_{t}\bigg\rangle\bigg]}_{I'_4}.
\end{align}
Since $\Delta_t$ is an unbiased estimator of $\Bar{\Delta}_t$, the main difference of convergence analysis for partial participation cases is bounding $\EE[\|\Delta_t\|^2]$. 

Note that the bound for $I'_2$ is exactly the same as the bound for $I_2$. For the corresponding three terms, $I'_1$, $I'_3$ and $I'_4$ which include the second-order momentum estimate of $\Delta_t$. For $I'_1$, we have
\begin{align}\label{eq:I1'-1}
    I'_1 & = \EE\bigg[\bigg\langle\nabla f(\xb_{t}), \eta\frac{ \Delta_{t}}{\sqrt{\hat{\vb}_{t}}} \bigg\rangle\bigg] \notag\\
    & \leq \eta\EE\bigg[\bigg\langle\nabla f(\xb_{t}), \frac{\sqrt{2} \cdot \Delta_{t}}{\sqrt{\vb_{t} + \epsilon}} \bigg\rangle\bigg] \notag\\
    & = \eta\EE\bigg[\bigg\langle\nabla f(\xb_{t}), \frac{\sqrt{2} \cdot \Delta_{t}}{\sqrt{\beta_2 \vb_{t-1}+\epsilon}}\bigg\rangle\bigg] + \eta\EE\bigg[\bigg\langle\nabla f(\xb_{t}),\frac{\sqrt{2} \cdot \Delta_{t}}{\sqrt{\vb_{t} + \epsilon}}- \frac{ \sqrt{2} \cdot \Delta_{t}}{\sqrt{\beta_2 \vb_{t-1}+\epsilon}}\bigg\rangle\bigg].
\end{align}
The first term in \eqref{eq:I1'-1} does not change in partial participation scheme. The second term is changed due to the variance of $\Delta_t$ changes. For the second term of $I'_1$, we have
\begin{align}\label{eq:I1'-2}
    \sqrt{2} \eta\EE\bigg[\bigg\langle\nabla f(x_{t}),\frac{\Delta_{t}}{\sqrt{\vb_{t} + \epsilon}}- \frac{\Delta_{t}}{\sqrt{\beta_2 \vb_{t-1}+\epsilon}}\bigg\rangle\bigg] \leq \frac{\sqrt{2(1-\beta_2)}\eta G}{\epsilon} \EE[\|\Delta_t\|^2].
\end{align}
For $I'_3$, we have
\begin{align}\label{eq:I3'}
    \sum_{t=1}^T I'_3 \leq \frac{\eta^2 L}{\epsilon} \sum_{t=1}^T \EE [\|\Delta_t\|^2] + \eta^2 L \frac{\beta_1^2}{(1-\beta_1)^2} \eta_l^2 K^2 G^2 \sum_{t=1}^T \EE \bigg[\bigg\|\bigg(\Vbh_{t-1}^{-1/2} - \Vbh_{t}^{-1/2}\bigg)\bigg\|^2\bigg],
\end{align}
and for $I'_4$, similar to \eqref{eq:I4-1}, we have
\begin{align}\label{eq:I4'-partial}
    \sum_{t=1}^T I'_4 \leq \frac{\eta^2 L}{2\epsilon} \sum_{t=1}^T \EE [\|\Delta_t\|^2] + \frac{\eta^2 L}{2\epsilon} \frac{\beta_1^2}{(1-\beta_1)^2} \sum_{t=1}^{T}\EE [\|\mb_t\|^2].
\end{align}
From Lemma \ref{lm:mE-bound-partial}, we have
\begin{align} \label{eq:m'-partial}
    \sum_{t=1}^{T}\EE [\|\mb_t\|^2] \leq \frac{K T \eta_l^2}{n} \sigma_l^2 + \frac{\eta_l^2}{n^2}  \sum_{t=1}^{T} \EE \bigg[\bigg\|\sum_{i \in \cS_t} \sum_{k=0}^{K-1} \nabla F_i(\xb_{t,k}^i) \bigg\|^2\bigg].
\end{align}
Then substituting \eqref{eq:m'-partial} into \eqref{eq:I4'-partial}, we have
\begin{align}\label{eq:I4'}
    \sum_{t=1}^T I'_4 & \leq \frac{\eta^2 L}{2\epsilon} \sum_{t=1}^T \EE [\|\Delta_t\|^2] + \frac{\beta_1^2}{(1-\beta_1)^2} \frac{\eta^2 \eta_l^2 L}{2n^2 \epsilon} \sum_{t=1}^T \EE\bigg[\bigg\|\sum_{i \in \cS_t} \sum_{k=0}^{K-1} \nabla F_i (\xb_{t,k}^i) \bigg\|^2\bigg] + \frac{\beta_1^2}{(1-\beta_1)^2} \frac{\eta^2 \eta_l^2 K T L}{2n \epsilon}\sigma_l^2 \notag\\
    & \leq \frac{\eta^2 L}{2\epsilon} \sum_{t=1}^T \EE [\|\Delta_t\|^2] + \frac{\beta_1^2}{(1-\beta_1)^2} \frac{\eta^2 \eta_l^2 L}{2n^2 \epsilon} \sum_{t=1}^T \EE\bigg[\bigg\|\sum_{i=1}^m \sum_{k=0}^{K-1} \PP\{i \in \cS_t \} \nabla F_i (\xb_{t,k}^i) \bigg\|^2\bigg] \notag\\
    & \quad + \frac{\beta_1^2}{(1-\beta_1)^2} \frac{\eta^2 \eta_l^2 K T L}{2n \epsilon}\sigma_l^2,
\end{align}
where we will further apply the bound for $\EE[\|\Delta_t\|^2]$ following by Lemma \ref{lm:EDelta'}. The second term in \eqref{eq:I4'} can be bounded from \eqref{eq:i-in-St-norm-noncomp}. Therefore, summing up \eqref{eq:I1'-2}, \eqref{eq:I3'} and \eqref{eq:I2}, summing over from $t = 1$ to $T$, then adding \eqref{eq:I4'-partial}, we have
\begin{align}\label{eq:fzI-1'}
    & \EE[f(\zb_{T+1})]-f(\zb_1) = \sum_{t=1}^{T} [I'_1+I_2+I'_3+I'_4] \notag\\
    & \leq -\frac{\eta\eta_l K}{4} \sum_{t=1}^{T}\EE\bigg[ \bigg\|\frac{\nabla f(\xb_t)}{\sqrt[4]{\beta_2 \vb_{t-1} + \epsilon}} \bigg\|^2\bigg]+ \frac{5\eta\eta_l^3 K^2 L^2 T}{\sqrt{2\epsilon}} (\sigma_l^2+6K \sigma_g^2) + \frac{\sqrt{2(1-\beta_2)} \eta G}{\epsilon} \sum_{t=1}^{T}\EE[\|\Delta_t\|^2]  \notag\\
    & \quad  - \frac{\eta\eta_l}{2 K m^2} \sum_{t=1}^{T} \EE\bigg[ \bigg\|\frac{1}{\sqrt[4]{\beta_2 \vb_{t-1} + \epsilon}}\sum_{i=1}^m \sum_{k=0}^{K-1} \nabla F_i(\xb_{t,k}^i)\bigg\|^2\bigg] \notag\\
    & \quad + \frac{\beta_1}{1-\beta_1} \eta\eta_l K G^2\sum_{t=1}^{T}\EE\bigg[\Big\|\Vbh_{t-1}^{-1/2} - \Vbh_{t}^{-1/2}\Big\|_1\bigg] + \frac{\beta_1^2}{(1-\beta_1)^2} \frac{\eta^2\eta_l^2 K^2 G^2}{\sqrt{\epsilon}} \sum_{t=1}^{T}\EE\bigg[\Big\|\Vbh_{t-1}^{-1/2} - \Vbh_{t}^{-1/2}\Big\|_1\bigg] \notag\\
    & \quad + \frac{\beta_1^2}{(1-\beta_1)^2} \eta^2\eta_l^2K^2 L G^2 \sum_{t=1}^{T}\EE\bigg[\Big\|\Vbh_{t-1}^{-1/2} - \Vbh_{t}^{-1/2}\Big\|^2\bigg] + \eta^2 L \sum_{t=1}^{T} \EE \bigg[\big\|\Vbh_t^{-1/2} \Delta_{t}\big\|^2 \bigg] \notag\\
    & \quad + \frac{\eta^2 L}{2\epsilon} \sum_{t=1}^{T} \EE [\|\Delta_{t}\|^2] + \frac{\eta^2 L}{2\epsilon} \frac{\beta_1^2}{(1-\beta_1)^2} \sum_{t=1}^{T} \EE[\|\mb_t\|^2].
\end{align}
By applying Lemma \ref{lm:EDelta} into all terms containing the second moment estimate of model difference $\Delta_t$ in \eqref{eq:fzI-1'}, using the fact that $(\sqrt{\beta_2 K^2 G^2 + \epsilon})^{-1} \|\xb\| \leq(\sqrt{\beta_2 \eta_l^2 K^2 G^2 + \epsilon})^{-1}\|\xb\| \leq \big\|\frac{\xb}{\sqrt{\beta_2 \vb_t + \epsilon}}\big\| \leq \epsilon^{-1/2} \|\xb\|$, and applying Lemma \ref{lm:mE-bound-partial}, we have
\begin{align}
    & \EE[f(\zb_{T+1})]-f(\zb_1) \notag\\
    & \leq -\frac{\eta\eta_l K}{4 \sqrt{\beta_2 \eta_l^2 K^2 G^2 + \epsilon}} \sum_{t=1}^{T} \EE[\|\nabla f(\xb_t)\|^2]+ \frac{5\eta\eta_l^3 K^2 L^2 T}{\sqrt{2\epsilon}} (\sigma_l^2+6K \sigma_g^2)  \notag\\
    & \quad + \frac{\beta_1}{1-\beta_1} \frac{\eta \eta_l K G^2 d}{\sqrt{\epsilon}} + \frac{\beta_1^2}{(1-\beta_1)^2}\frac{2\eta^2 \eta_l^2 K^2 L G^2 d}{\epsilon} \notag\\
    & \quad + \bigg(\frac{3 \eta^2 L}{2} + \frac{\beta_1^2}{2(1-\beta_1)^2} \eta^2 L + \sqrt{2(1-\beta_2)} \eta G\bigg) \frac{KT\eta_l^2}{n \epsilon} \sigma_l^2 - \sum_{t=1}^{T} \EE \bigg[\bigg\|\sum_{i=1}^m \sum_{k=0}^{K-1} \nabla F_i(\xb_{t,k}^i)\bigg\|^2 \bigg] \notag\\
    & \quad \cdot \bigg[\frac{\eta\eta_l}{2 \sqrt{\beta_2 K^2 G^2 + \epsilon} K m^2} - \bigg(\frac{3 \eta^2 L}{2} + \frac{\beta_1^2}{2(1-\beta_1)^2} \eta^2 L + \sqrt{2(1-\beta_2)} \eta G\bigg) \frac{\eta_l^2 (n-1)}{mn(m-1)\epsilon} \bigg]\notag\\
    & \quad +\bigg(\frac{3 \eta^2 L}{2} + \frac{\beta_1^2}{2(1-\beta_1)^2} \eta^2 L + \sqrt{2(1-\beta_2)} \eta G\bigg) \frac{\eta_l^2 (m-n)}{mn(m-1)\epsilon} \bigg[15mK^3 L^3 \eta_l^2(\sigma_l^2 + 6K\sigma_g^2)T \notag\\
    & \quad + (90mK^4 L^2\eta_l^2 + 3mK^2) \sum_{t=1}^{T} \EE[\|\nabla f(\xb_t)\|^2] + 3mK^2 T \sigma_g^2\bigg],
\end{align}
then we have 
\begin{align}
    & \EE[f(\zb_{T+1})] - f(\zb_1) \notag\\
    & \leq -\frac{\eta\eta_l K}{4 \sqrt{\beta_2 \eta_l^2 K^2 G^2 + \epsilon}} \sum_{t=1}^{T} \EE[\|\nabla f(\xb_t)\|^2]+ \frac{5\eta\eta_l^3 K^2 L^2 T}{\sqrt{2\epsilon}} (\sigma_l^2+6K \sigma_g^2) + \frac{\beta_1}{1-\beta_1} \frac{\eta \eta_l K G^2 d}{\sqrt{\epsilon}} \notag\\
    & \quad + \frac{\beta_1^2}{(1-\beta_1)^2}\frac{2 \eta^2 \eta_l^2 K^2 L G^2 d}{\epsilon} + \bigg(\frac{3 \eta^2 L}{2} + \frac{\beta_1^2}{2(1-\beta_1)^2} \eta^2 L + \sqrt{2(1-\beta_2)} \eta G\bigg) \frac{KT\eta_l^2}{n \epsilon} \sigma_l^2\notag\\
    & \quad + \bigg(\frac{3 \eta^2 L}{2} + \frac{\beta_1^2}{2(1-\beta_1)^2} \eta^2 L + \sqrt{2(1-\beta_2)} \eta G\bigg) \frac{\eta_l^2 (m-n)}{mn(m-1)\epsilon} \bigg[15mK^3 L^3 \eta_l^2(\sigma_l^2 + 6K\sigma_g^2)T \notag\\
    & \quad + (90mK^4 L^2\eta_l^2 + 3mK^2) \sum_{t=1}^{T} \EE[\|\nabla f(\xb_t)\|^2] + 3mK^2 T \sigma_g^2\bigg].
\end{align}
By adopting additional constraint of local learning rate $\eta_l$ with the inequality $\Big(\frac{3 \eta^2 L}{2} + \frac{\beta_1^2}{2(1-\beta_1)^2} \eta^2 L + \sqrt{2(1-\beta_2)} \eta G\Big)  \frac{\eta_l^2 (n-1)}{mn(m-1)\epsilon} - \frac{\eta\eta_l}{2\sqrt{\beta_2 K^2 G^2 + \epsilon} K m^2} \leq 0$, thus we obtain the constraint $\eta_l \leq \frac{n(m-1)}{m(n-1)} \frac{\epsilon}{\sqrt{\beta_2 K^2 G^2 + \epsilon} K (3\eta L+C_1^2\eta L+2\sqrt{2(1-\beta_2)}G)}$, and we further need $\eta_l$ satisfies $\frac{\eta\eta_l K}{4 \sqrt{\beta_2 \eta_l^2 K^2 G^2 + \epsilon}} - \Big(\frac{3 \eta^2 L}{2} + \frac{\beta_1^2}{2(1-\beta_1)^2} \eta^2 L + \sqrt{2(1-\beta_2)} \eta G\Big) \frac{\eta_l^2 (m-n)}{mn(m-1)\epsilon} (90mK^4L^2\eta_l^2 + 3mK^2) \geq \frac{\eta\eta_l K}{8 \sqrt{\beta_2 \eta_l^2 K^2 G^2 + \epsilon}}$. Hence we have the following condition on local learning rate $\eta_l$,
\begin{align}
    \eta_l \leq \frac{n(m-1)\epsilon}{48m (n-1)} \bigg[K \sqrt{\beta_2 K^2 G^2 + \epsilon} \cdot \bigg(\frac{3 \eta L}{2} + \frac{\beta_1^2}{2(1-\beta_1)^2} \eta L + \sqrt{2(1-\beta_2)} G\bigg) \bigg]^{-1},
\end{align}
then we have
\begin{align}\label{eq:fzI-3'}
    & \frac{\eta\eta_l K}{8 \sqrt{\beta_2 \eta_l^2 K^2 G^2 + \epsilon} \cdot T} \sum_{i=1}^T \EE[\|\nabla f(\xb_t)\|^2] \notag\\
    & \leq \frac{f(\zb_0)-\EE [f(\zb_T)]}{T} + \frac{C_1 \eta \eta_l K G^2 d}{T \sqrt{\epsilon}} + \frac{2  C_1^2  \eta^2 \eta_l^2 K^2 L G^2 d}{T \epsilon} \notag\\
    & + \frac{5\eta\eta_l^3 K^2 L^2}{\sqrt{2\epsilon}} (\sigma_l^2+6K \sigma_g^2) + \bigg(\frac{3 \eta^2 L}{2} + \frac{\beta_1^2}{2(1-\beta_1)^2} \eta^2 L + \sqrt{2(1-\beta_2)}\eta G\bigg) \frac{K\eta_l^2}{n\epsilon} \sigma_l^2 \notag\\
    & + \bigg(\frac{3 \eta^2 L}{2} + \frac{\beta_1^2}{2(1-\beta_1)^2} \eta^2 L + \sqrt{2(1-\beta_2)}\eta G\bigg) \frac{\eta_l^2(m-n)}{mn(m-1)\epsilon}[ 15 m K^3 L^2 \eta_l^2(\sigma_l^2+ 6K \sigma_g^2) + 3m K^2\sigma_g^2].
\end{align}
Therefore, 
\begin{align}
    \min \EE[\|\nabla f(\xb_t)\|^2] & \leq 8 \sqrt{\beta_2 \eta_l^2 K^2 G^2 + \epsilon} \bigg[\frac{f_0-f_*}{\eta \eta_l K T} + \frac{\Psi}{T} + \Phi \bigg],
\end{align}
where $\Psi = \frac{C_1 G^2 d}{\sqrt{\epsilon}}+ \frac{2C_1^2 \eta\eta_l K L G^2 d}{\epsilon}, \Phi = \frac{5 \eta_l^2 K L^2}{\sqrt{2\epsilon}} (\sigma_l^2+6K \sigma_g^2)+ [(3+C_1^2) \eta L + 2\sqrt{2(1-\beta_2)} G] \frac{\eta_l}{2n\epsilon} \sigma_l^2 + [(3+C_1^2)\eta L + 2\sqrt{2(1-\beta_2)} G] \frac{\eta_l(m-n)}{2n(m-1)\epsilon} [15 K^2L^2 \eta_l^2 (\sigma_l^2 + 6K\sigma_g^2) + 3K\sigma_g^2]$ and $C_1 = \frac{\beta_1}{1-\beta_1}$.

\subsection{Proof of Corollary \ref{cor:fedams-part}}

If we choose $\eta_l= \Theta( \frac{1}{\sqrt{T} K})$ and $\eta = \Theta(\sqrt{Kn})$, we have $\min_{t \in[T]} \EE[\|\nabla f(\xb_t)\|^2] = \cO(\frac{\sqrt{K}}{\sqrt{Tn}})$.


\section{Proof of Theorems in Section \ref{subsec:cfedams} and Partial Participation Setting for FedCAMS}\label{appendix:fedcams}
\subsection{Compression Dissimilarity}\label{appendix:fedcams-assumption}

\begin{assumption}[Rewrite Assumption \ref{as:compressor-adapt}] \label{as:compressor-adapt-appendix} For the biased compressor satisfies \ref{as:compressor}, there exists a constant $\xi$ such that, for each iteration $t \geq 0$, we have 
\begin{align*}
    & \bigg\|\cC\bigg( \frac{1}{m} \sum_{i=1}^m [\Delta_t^i + \eb_t^i] \bigg) - \frac{1}{m} \sum_{i=1}^m \cC(\Delta_t^i + \eb_t^i)\bigg\| \leq \gamma \bigg\| \frac{1}{m} \sum_{i=1}^m \Delta_t^i \bigg\|.
\end{align*}

\end{assumption}
\vspace{-3pt}
Note that Assumption \ref{as:compressor-adapt} implies that the overlap between ``the average of compression'' and ``the compression of average'' leads to a bounded loss of information at each step. A similar assumption has been adopted in \citet{alistarh2018convergence}, and \citet{haddadpour2021federated} assumed a constant bound for the unbiased compressor (such as quantization compressor) to measure the gap between these terms. We also empirically justify this assumption in practice to validate it on training the CIFAR-10 dataset with ResNet-18 and ConvMixer-256-8 models. Figure \ref{fig:compressor} shows that the $\gamma$ value during training varies over time but maintains a bounded $\gamma$. 

\begin{figure*}[ht!]
    \centering
    
    \subfigure[Scaled sign compressor]{\includegraphics[width=0.36\textwidth]{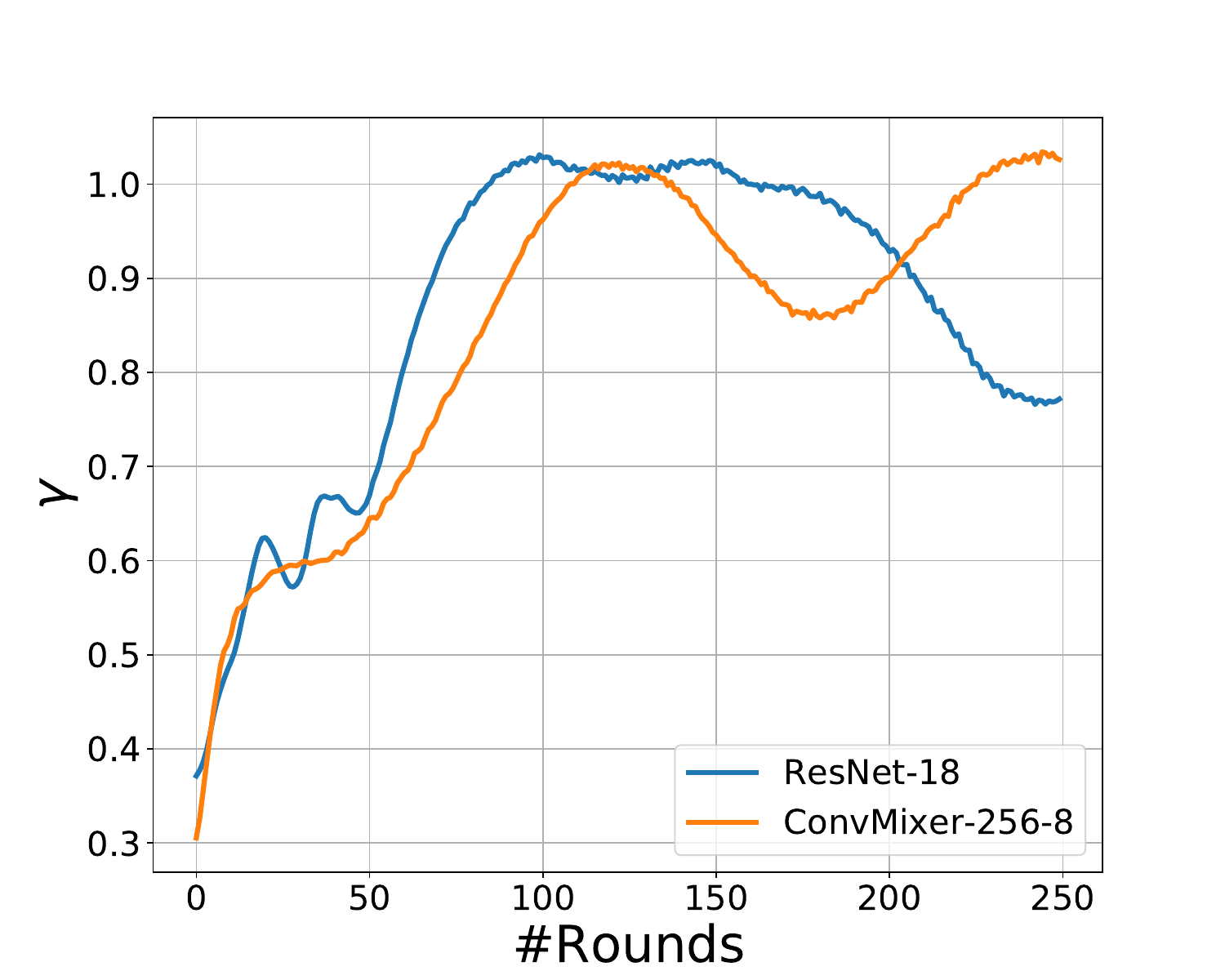}}
    \subfigure[Top-k with $r = 1/64$ compressor]{\includegraphics[width=0.36\textwidth]{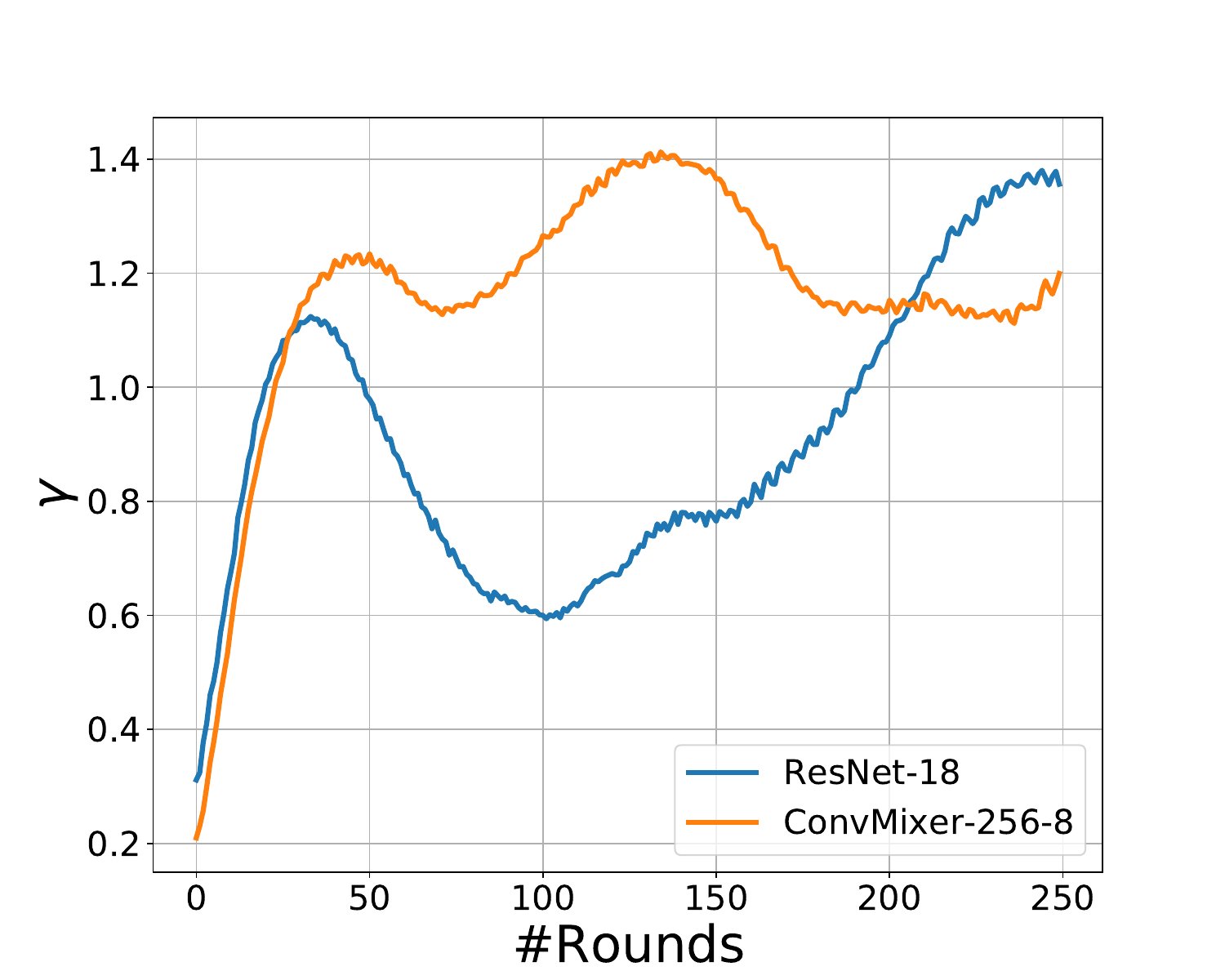}}
    \caption{Empirical justification for Assumption \ref{as:compressor-adapt} on various models and compressors training on the CIFAR-10 dataset. }
    \label{fig:compressor}
\end{figure*}

\subsection{Proof of Theorem \ref{thm:full-cfedams}}

\textit{Notations and equations:} From the update rule of Algorithm \ref{alg:compfedams}, we have $\eb_1 = 0$, $\eb_t= \frac{1}{m} \sum_{i=1}^m \eb_t^i$ and $\mb_{t}=(1-\beta_1) \sum_{i=1}^{t} \beta_{1}^{t-i} \hat{\Delta}_{i}$. Denote a global uncompressed difference $\Delta_t= \frac{1}{m} \sum_{i=1}^m \Delta_i^t$. Denote a virtual momentum sequence: $\mb_{t}^{\prime}=\beta_{1} \mb_{t-1}^{\prime}+(1-\beta_1) \Delta_t$, hence we have $\mb_{t}^{\prime}=(1-\beta_1) \sum_{i=1}^{t} \beta_{1}^{t-i} \Delta_{i}$. By the aforementioned definition and notation, we have 
\begin{align}  \label{eq:Dhat-D}
    \hat{\Delta}_t -\Delta_t & = \frac{1}{m} \sum_{i=1}^m ( \hat{\Delta}_t^i-\Delta_t^i)= \frac{1}{m} \sum_{i=1}^m (\eb_t^i-\eb_{t+1}^i) = \eb_t-\eb_{t+1}.
\end{align}
Denote the weighted averaging error sequence $\bGamma_t = (1-\beta_1)\sum_{\tau=1}^{t} \beta_1^{t-\tau} \eb_\tau$, with the imput $\eb_1 = 0$, we obtain the relation between $\bGamma_t$ and $\mb_t$ as follows
\begin{align}\label{eq:mGamma}
    \mb_t - \mb'_t = (1-\beta_1)\sum_{\tau=1}^t \beta_1^{t-\tau}(\hat{\Delta}_\tau - \Delta_\tau) = (1-\beta_1)\sum_{\tau=1}^t \beta_1^{t-\tau}(\eb_\tau - \eb_{\tau+1}) = \bGamma_t - \bGamma_{t+1},
\end{align}
where the last step holds due to $\bGamma_{t+1} = (1-\beta_1)\sum_{\tau=1}^{t+1} \beta_1^{t-\tau} \eb_{\tau+1} = (1-\beta_1)\sum_{\tau=1}^{t} \beta_1^{t-\tau} \eb_{\tau+1} + \beta_1^t \eb_1$. 

Similar to previous works studied adaptive methods \citep{chen2020closing,zhou2018convergence,chen2018convergence}, we introduce a Lyapunov sequence $\zb_t$: assume $\xb_0 = \xb_1$, for each $t \geq 1$, we have
\begin{align*}
    \yb_t = \xb_t + \frac{\beta_1}{1-\beta_1} (\xb_t - \xb_{t-1})  = \frac{1}{1-\beta_1} \xb_t- \frac{\beta_1}{1-\beta_1}\xb_{t-1}.
\end{align*}
Therefore, by the update rule of $\xb_t$, we have
\begin{align}
    \yb_{t+1} & = \xb_{t+1} + \eta \frac{\beta_1}{1-\beta_1}\Vbh_t^{-1/2} \mb_{t} \notag\\
    & = \xb_{t+1} + \eta \frac{\beta_1}{1-\beta_1}\Vbh_t^{-1/2} [\mb'_{t} + \bGamma_t - \bGamma_{t+1}] \notag\\
    & = \xb_{t+1} + \eta \frac{\beta_1}{1-\beta_1}\Vbh_t^{-1/2} \mb'_{t} + \eta \frac{\beta_1}{1-\beta_1}\Vbh_t^{-1/2} \bigg[\frac{\bGamma_{t+1} - (1-\beta_1)\eb_{t+1}}{\beta_1}-\bGamma_{t+1} \bigg]\notag\\
    & = \xb_{t  +1} + \eta \frac{\beta_1}{1-\beta_1} \Vbh_t^{-1/2} \mb'_t + \eta \Vbh_t^{-1/2} \bGamma_{t+1}- \eta \Vbh_t^{-1/2}\eb_{t+1}.
\end{align}
The third equation holds due to the fact that $\bGamma_{t+1} = \beta_1\bGamma_t + (1-\beta_1)\eb_{t+1}$.
We then introduce a new sequence based on the previous Lyapunov sequence $\yb_t$ as follows
\begin{align}\label{eq:lyapunov-z}
    \zb_{t+1} = \yb_{t+1} + \eta \Vbh_t^{-1/2} \eb_{t+1} = \xb_{t  +1} + \eta \frac{\beta_1}{1-\beta_1} \Vbh_t^{-1/2} \mb'_t + \eta \Vbh_t^{-1/2} \bGamma_{t+1}. 
\end{align}
The sequence difference $\zb_{t+1}- \zb_t$ can be represented by
\begin{align}
    \zb_{t+1} - \zb_t & = \xb_{t+1} - \xb_t + \eta\frac{\beta_1}{1-\beta_1} \Vbh_{t}^{-1/2} \mb'_t - \eta \frac{\beta_1}{1-\beta_1}\Vbh_{t-1}^{-1/2} \mb'_{t-1} + \eta\Vbh_{t}^{-1/2} \bGamma_{t+1} - \eta\Vbh_{t-1}^{-1/2} \bGamma_{t} \notag\\
    & = \eta \Vbh_t^{-1/2} \mb_t + \eta\Vbh_{t}^{-1/2} \bGamma_{t+1} + \eta\frac{\beta_1}{1-\beta_1} \Vbh_{t}^{-1/2} \mb'_t - \eta \frac{\beta_1}{1-\beta_1}\Vbh_{t-1}^{-1/2} \mb'_{t-1}- \eta\Vbh_{t-1}^{-1/2} \bGamma_{t},
\end{align}
where the second equation follows the update rule of $\xb_{t+1}$. Following \eqref{eq:mGamma}, then combining likely terms and applying the definition of $\mb'_t$, we have 
\begin{align}
    \zb_{t+1} - \zb_t & = \eta \Vbh_t^{-1/2} \mb'_t +\eta \Vbh_t^{-1/2} \bGamma_t + \eta\frac{\beta_1}{1-\beta_1} \Vbh_{t}^{-1/2} \mb'_t - \eta \frac{\beta_1}{1-\beta_1}\Vbh_{t-1}^{-1/2} \mb'_{t-1} - \eta\Vbh_{t-1}^{-1/2} \bGamma_{t} \notag\\
    & = \eta\frac{1}{1-\beta_1} \Vbh_{t}^{-1/2} \mb'_t - \eta \frac{\beta_1}{1-\beta_1}\Vbh_{t-1}^{-1/2} \mb'_{t-1} +\eta \Vbh_t^{-1/2} \bGamma_t - \eta\Vbh_{t-1}^{-1/2} \bGamma_{t} \notag\\
    & = \eta\frac{1}{1-\beta_1} \Vbh_{t}^{-1/2} [\beta_1\mb'_{t-1} + (1-\beta_1)\Delta_t] - \eta \frac{\beta_1}{1-\beta_1}\Vbh_{t-1}^{-1/2} \mb'_{t-1} +\eta \Vbh_t^{-1/2} \bGamma_t - \eta\Vbh_{t-1}^{-1/2} \bGamma_{t} \notag\\
    & = \eta \Vbh_{t}^{-1/2} \Delta_t - \eta \frac{\beta_1}{1-\beta_1}\bigg(\Vbh_{t-1}^{-1/2}-\Vbh_{t}^{-1/2}\bigg) \mb'_{t-1} - \eta \bigg(\Vbh_{t-1}^{-1/2} - \Vbh_t^{-1/2}\bigg) \bGamma_t.
\end{align}
Therefore, we obtain a helpful Lyapunov sequence for our proof of FedCAMS. The proof of FedCAMS in full participation settings has a similar outline with the proof of FedAMS. By Assumption \ref{as:smooth}, we have
\begin{align}\label{eq:f-Lsmooth}
& \EE[f(\zb_{t+1})]-f(\zb_{t}) \notag\\
& \leq \EE[\langle \nabla f(\zb_t), \zb_{t+1}-\zb_t \rangle] + \frac{L}{2} \EE[\|\zb_{t+1}-\zb_t\|^2] \notag\\
& \leq \EE\bigg[\bigg\langle\nabla f(\zb_{t}), \eta \Vbh_{t}^{-1/2}  \Delta_{t}\bigg\rangle\bigg]\notag\\
& \quad -\EE\bigg[\bigg\langle\nabla f(\zb_{t}), \eta \frac{\beta_{1}}{1-\beta_{1}} \bigg(\Vbh_{t-1}^{-1/2} - \Vbh_{t}^{-1/2}\bigg) \mb_{t-1}^{\prime}+ \eta \bigg(\Vbh_{t-1}^{-1/2} - \Vbh_{t}^{-1/2}\bigg) \bGamma_{t}\bigg\rangle\bigg] \notag\\
& \quad +\frac{\eta^{2} L}{2} \EE\bigg[\bigg\|\Vbh_{t}^{-1/2}  \Delta_{t}-\frac{\beta_{1}}{1-\beta_{1}} \bigg(\Vbh_{t-1}^{-1/2} - \Vbh_{t}^{-1/2}\bigg) \mb_{t-1}^{\prime}-\bigg(\Vbh_{t-1}^{-1/2} - \Vbh_{t}^{-1/2}\bigg) \bGamma_{t}\bigg\|^{2}\bigg] \notag\\
& = \underbrace{\EE\bigg[\bigg\langle\nabla f(\xb_{t}), \eta \Vbh_{t}^{-1/2}  \Delta_{t}\bigg\rangle\bigg]}_{T_1} \underbrace{-\eta \EE\bigg[\bigg\langle\nabla f(\zb_{t}), \frac{\beta_{1}}{1-\beta_{1}} \bigg(\Vbh_{t-1}^{-1/2} - \Vbh_{t}^{-1/2}\bigg) \mb_{t-1}^{\prime}+\bigg(\Vbh_{t-1}^{-1/2} - \Vbh_{t}^{-1/2}\bigg) \bGamma_{t}\bigg\rangle\bigg]}_{T_2}\notag\\
& \quad + \underbrace{\frac{\eta^{2} L}{2} \EE\bigg[\bigg\|\Vbh_{t}^{-1/2}  \Delta_{t}-\frac{\beta_{1}}{1-\beta_{1}} \bigg(\Vbh_{t-1}^{-1/2} - \Vbh_{t}^{-1/2}\bigg) \mb_{t-1}^{\prime}-\bigg(\Vbh_{t-1}^{-1/2} - \Vbh_{t}^{-1/2}\bigg) \bGamma_{t}\bigg\|^{2}\bigg]}_{T_3} \notag\\
& \quad +  \underbrace{\EE\bigg[\bigg\langle\nabla f(\zb_{t})-\nabla f(\xb_t), \eta \Vbh_{t}^{-1/2} \Delta_{t}\bigg\rangle\bigg]}_{T_4},
\end{align}
here we recall the notation $\Vbh_t = \diag(\hat{\vb}_t) = \diag(\max(\hat{\vb}_{t-1}, \vb_t, \epsilon))$.

\textbf{Bounding $T_1$:}We have
\begin{align}\label{eq:T1-1}
    T_1 & = \EE\bigg[\bigg\langle\nabla f(\xb_{t}), \eta\frac{ \Delta_{t}}{\sqrt{\hat{\vb}_{t}}} \bigg\rangle\bigg] \notag\\
    & \leq \eta\EE\bigg[\bigg\langle\nabla f(\xb_{t}), \frac{\sqrt{2} \cdot \Delta_{t}}{\sqrt{\vb_{t}+\epsilon}} \bigg\rangle\bigg] \notag\\
    & = \sqrt{2} \eta\EE\bigg[\bigg\langle\nabla f(\xb_{t}), \frac{\Delta_{t}}{\sqrt{\beta_2 \vb_{t-1}+\epsilon}}\bigg\rangle\bigg] + \sqrt{2} \eta\EE\bigg[\bigg\langle\nabla f(\xb_{t}),\frac{ \Delta_{t}}{\sqrt{\vb_{t}+\epsilon}} - \frac{\Delta_{t}}{\sqrt{\beta_2 \vb_{t-1}+\epsilon}}\bigg\rangle\bigg],
\end{align}
where the first inequality follows by the fact that $\hat{\vb}_t \geq \frac{\vb_t + \epsilon}{2} $. 
For the second term in \eqref{eq:T1-1}, we have 
\begin{align} \label{eq:T1-1-2}
    & \sqrt{2}\cdot \eta\EE\bigg[\bigg\langle\nabla f(\xb_{t}),\frac{\Delta_{t}}{\sqrt{\vb_{t}+\epsilon}} - \frac{\Delta_{t}}{\sqrt{\beta_2 \vb_{t-1}+\epsilon}}\bigg\rangle\bigg] \notag\\
    & \leq \sqrt{2}\cdot\eta \cdot \|\nabla f(\xb_t)\| \EE\bigg[\bigg\|\frac{1}{\sqrt{\vb_{t}+\epsilon}}- \frac{1}{\sqrt{\beta_2 \vb_{t-1}+\epsilon}}\bigg\| \cdot \|\Delta_t\| \bigg] \notag\\
    & \leq \frac{\eta \sqrt{2(1-\beta_2)} G}{\epsilon} \EE[\|\Delta_t\|^2],
\end{align}
where the second inequality follows from Lemma~\ref{lm:Wt} and \ref{lm:gmv-bound}, and we will further apply the bound for $\EE[\|\Delta_t\|^2]$ by applying Lemma \ref{lm:EDelta}. For the first term in \eqref{eq:T1-1}, we have
\begin{align} \label{eq:T1-2}
    & \sqrt{2}\cdot\eta\EE\bigg[\bigg\langle\nabla f(\xb_{t}), \frac{\Delta_{t}}{\sqrt{\beta_2 \vb_{t-1}+\epsilon}}\bigg\rangle\bigg] \notag\\
    & = \sqrt{2}\cdot\eta\EE\bigg[\bigg\langle \frac{\nabla f(\xb_t)}{\sqrt{\beta_2 \vb_{t-1}+\epsilon}} , \Delta_t + \eta_l K \nabla f(\xb_t) - \eta_l K \nabla f(\xb_t) \bigg\rangle\bigg] \notag\\
    & = - \sqrt{2} \eta\eta_l K\EE\bigg[ \bigg\|\frac{\nabla f(\xb_t)}{\sqrt[4]{\beta_2 \vb_{t-1} + \epsilon}} \bigg\|^2\bigg] + \sqrt{2} \eta\EE\bigg[\bigg\langle \frac{\nabla f(\xb_t)}{\sqrt{\beta_2 \vb_{t-1}+\epsilon}} , \Delta_t + \eta_l K \nabla f(\xb_t) \bigg\rangle\bigg] \notag\\
    & = - \sqrt{2}\eta\eta_l K\EE\bigg[ \bigg\|\frac{\nabla f(\xb_t)}{\sqrt[4]{\beta_2 \vb_{t-1} + \epsilon}} \bigg\|^2\bigg] + \sqrt{2}\eta \bigg\langle \frac{\nabla f(\xb_t)}{\sqrt{\beta_2 \vb_{t-1}+\epsilon}}, \EE \bigg[-\frac{\eta_l }{m}\sum_{i=1}^m \sum_{k=0}^{K-1} \gb_{t,k}^i + \eta_l K \nabla f(\xb_t)\bigg] \bigg\rangle \notag\\
    & = - \sqrt{2}\eta\eta_l K\EE\bigg[ \bigg\|\frac{\nabla f(\xb_t)}{\sqrt[4]{\beta_2 \vb_{t-1} + \epsilon}} \bigg\|^2\bigg] + \sqrt{2}\eta \bigg\langle \frac{\nabla f(\xb_t)}{\sqrt{\beta_2 \vb_{t-1}+\epsilon}}, \EE \bigg[-\frac{\eta_l}{m}\sum_{i=1}^m \sum_{k=0}^{K-1} \gb_{t,k}^i + \frac{\eta_l K}{m}\sum_{i=1}^m \nabla F_i(\xb_t)\bigg] \bigg\rangle.
\end{align}
For the last term in \eqref{eq:T1-2}, we have 
\begin{align}
    & \sqrt{2} \eta \bigg\langle \frac{\nabla f(\xb_t)}{\sqrt{\beta_2 \vb_{t-1}+\epsilon}}, \EE \bigg[-\frac{\eta_l}{m}\sum_{i=1}^m \sum_{k=0}^{K-1} \gb_{t,k}^i + \frac{\eta_l K}{m}\sum_{i=1}^m \nabla F_i(\xb_t)\bigg] \bigg\rangle \notag\\
    & = \sqrt{2} \eta \bigg\langle \frac{\sqrt{\eta_l K}}{\sqrt[4]{\beta_2 \vb_{t-1} + \epsilon}} \nabla f(\xb_t),  - \frac{\sqrt{\eta_l K}}{K m} \frac{1}{\sqrt[4]{\beta_2 \vb_{t-1} + \epsilon}}\EE \bigg[\sum_{i=1}^m \sum_{k=0}^{K-1} (\nabla F_i(\xb_{t,k}^i) - \nabla F_i(\xb_t))\bigg] \bigg\rangle \notag\\
    & = \frac{\sqrt{2}\eta \eta_l K}{2} \bigg\|\frac{\nabla f(\xb_t)}{\sqrt[4]{\beta_2 \vb_{t-1} + \epsilon}} \bigg\|^2 + \frac{\sqrt{2} \eta\eta_l}{2 K m^2} \EE \bigg[ \bigg\|\frac{1}{\sqrt[4]{\beta_2 \vb_{t-1} + \epsilon}}\sum_{i=1}^m \sum_{k=0}^{K-1} (\nabla F_i(\xb_{t,k}^i) - \nabla F_i(\xb_t))\bigg\|^2 \bigg]\notag\\
    & \quad - \frac{\sqrt{2} \eta\eta_l}{2 K m^2} \EE \bigg[ \bigg\|\frac{1}{\sqrt[4]{\beta_2 \vb_{t-1} + \epsilon}}\sum_{i=1}^m \sum_{k=0}^{K-1} \nabla F_i(\xb_{t,k}^i)\bigg\|^2 \bigg] \notag\\
    & \leq \frac{\sqrt{2}\eta \eta_l K}{2} \bigg\|\frac{\nabla f(\xb_t)}{\sqrt[4]{\beta_2 \vb_{t-1} + \epsilon}} \bigg\|^2 + \frac{\sqrt{2}\eta\eta_l}{2m}\sum_{i=1}^m \sum_{k=0}^{K-1}\EE \bigg[ \bigg\| \frac{\nabla F_i(\xb_{t,k}^i) - \nabla F_i(\xb_t)}{\sqrt[4]{\beta_2 \vb_{t-1} + \epsilon}} \bigg\|^2\bigg] \notag\\
    & \quad - \frac{\sqrt{2}\eta\eta_l}{2 K m^2} \EE \bigg[ \bigg\|\frac{1}{\sqrt[4]{\beta_2 \vb_{t-1} + \epsilon}} \sum_{i=1}^m \sum_{k=0}^{K-1} \nabla F_i(\xb_{t,k}^i)\bigg\|^2\bigg],
\end{align}
where the second equation follows from $\langle \xb,\yb\rangle = \frac{1}{2}[\|\xb\|^2 + \|\yb\|^2 - \|\xb-\yb\|^2]$, and the inequality holds by applying Cauchy-Schwarz inequality. Then by Assumption~\ref{as:smooth}, we have 
\begin{align}
    & \sqrt{2} \eta \bigg\langle \frac{\nabla f(\xb_t)}{\sqrt{\beta_2 \vb_{t-1}+\epsilon}}, \EE \bigg[-\frac{\eta_l}{m}\sum_{i=1}^m \sum_{k=0}^{K-1} \gb_{t,k}^i + \frac{\eta_l K}{m}\sum_{i=1}^m \nabla F_i(\xb_t)\bigg] \bigg\rangle \notag\\
    & \leq \frac{\sqrt{2} \eta \eta_l K}{2} \bigg\|\frac{\nabla f(\xb_t)}{\sqrt[4]{\beta_2 \vb_{t-1} + \epsilon}} \bigg\|^2 + \frac{\sqrt{2}\eta\eta_l L^2}{2m}\sum_{i=1}^m \sum_{k=0}^{K-1}\EE \bigg[\bigg\| \frac{\xb_{t,k}^i - \xb_t}{\sqrt[4]{\beta_2 \vb_{t-1} + \epsilon}} \bigg\|^2\bigg] \notag\\
    & \quad - \frac{\sqrt{2}\eta\eta_l}{2 K m^2} \EE\bigg[ \bigg\|\frac{1}{\sqrt[4]{\beta_2 \vb_{t-1} + \epsilon}}\sum_{i=1}^m \sum_{k=0}^{K-1} \nabla F_i(\xb_{t,k}^i)\bigg\|^2\bigg] \notag\\ 
    & \leq \frac{3\sqrt{2}\eta \eta_l K}{4} \bigg\|\frac{\nabla f(\xb_t)}{\sqrt[4]{\beta_2 \vb_{t-1} + \epsilon}} \bigg\|^2 + \frac{5\eta\eta_l^3 K^2 L^2}{\sqrt{2\epsilon}} (\sigma_l^2+6K \sigma_g^2) \notag\\
    & \quad - \frac{\sqrt{2}\eta\eta_l}{2 K m^2} \EE\bigg[ \bigg\|\frac{1}{\sqrt[4]{\beta_2 \vb_{t-1} + \epsilon}}\sum_{i=1}^m \sum_{k=0}^{K-1} \nabla F_i(\xb_{t,k}^i)\bigg\|^2\bigg],
\end{align}
where the last inequality holds by applying Lemma~\ref{lm:xikt-xt} and the constraint of local learning rate $\eta_l \leq \frac{1}{8KL}$. Then we have
\begin{align}\label{eq:T1}
    T_1 & \leq -\frac{\sqrt{2} \cdot\eta\eta_l K}{4} \EE\bigg[ \bigg\|\frac{\nabla f(\xb_t)}{\sqrt[4]{\beta_2 \vb_{t-1} + \epsilon}} \bigg\|^2\bigg] + \frac{5\eta\eta_l^3 K^2 L^2}{\sqrt{2 \epsilon}} (\sigma_l^2+6K \sigma_g^2) \notag\\
    & \quad - \frac{\sqrt{2} \cdot\eta\eta_l}{2 K m^2} \EE \bigg[ \bigg\|\frac{1}{\sqrt[4]{\beta_2 \vb_{t-1} + \epsilon}}\sum_{i=1}^m \sum_{k=0}^{K-1} \nabla F_i(\xb_{t,k}^i)\bigg\|^2\bigg] + \frac{\eta \sqrt{2(1-\beta_2)} G}{\epsilon} \EE[\|\Delta_t\|^2] \notag\\
    & \leq -\frac{\eta\eta_l K}{4} \EE \bigg[\bigg\|\frac{\nabla f(\xb_t)}{\sqrt[4]{\beta_2 \vb_{t-1} + \epsilon}} \bigg\|^2\bigg] + \frac{5\eta\eta_l^3 K^2 L^2}{\sqrt{2\epsilon}} (\sigma_l^2+6K \sigma_g^2) \notag\\
    & \quad - \frac{\eta\eta_l}{2 K m^2} \EE \bigg[ \bigg\|\frac{1}{\sqrt[4]{\beta_2 \vb_{t-1} + \epsilon}}\sum_{i=1}^m \sum_{k=0}^{K-1} \nabla F_i(\xb_{t,k}^i)\bigg\|^2\bigg] + \frac{\eta \sqrt{2(1-\beta_2)} G}{\epsilon} \EE[\|\Delta_t\|^2].
\end{align}

\textbf{Bounding $T_2$:}
The bound for $T_2$ mainly follows by the update rule and definition of virtual sequence $\zb_t$.
\begin{align}\label{eq:T2}
    T_2 & = -\eta \EE\bigg[\bigg\langle\nabla f(\zb_{t}), \frac{\beta_{1}}{1-\beta_{1}} \bigg(\Vbh_{t-1}^{-1/2} - \Vbh_{t}^{-1/2}\bigg) \mb_{t-1}^{\prime}+\bigg(\Vbh_{t-1}^{-1/2} - \Vbh_{t}^{-1/2}\bigg) \bGamma_{t}\bigg\rangle\bigg] \notag\\
    & = \eta \EE\bigg[\bigg\langle -\nabla f(\xb_{t}) + \nabla f(\xb_{t}) - \nabla f(\zb_{t}), \bigg(\Vbh_{t-1}^{-1/2} - \Vbh_{t}^{-1/2}\bigg) \bigg(\frac{\beta_{1}}{1-\beta_{1}} \mb_{t-1}^{\prime}+\bGamma_{t}\bigg) \bigg\rangle\bigg] \notag\\
    & \leq \eta \EE\bigg[\|\nabla f(\xb_{t})\|\bigg
    \|\bigg(\Vbh_{t-1}^{-1/2} - \Vbh_{t}^{-1/2}\bigg) \bigg(\frac{\beta_{1}}{1-\beta_{1}} \mb_{t-1}^{\prime}+\bGamma_{t}\bigg)\bigg\|\bigg] \notag\\
    & \quad + \eta^2 L \EE \bigg[\bigg\| \Vbh_{t-1}^{-1/2}\bigg(\frac{\beta_1}{1-\beta_1} \mb_{t-1}^{\prime}+ \bGamma_{t}\bigg)\bigg\| \bigg\|\bigg(\Vbh_{t-1}^{-1/2} - \Vbh_{t}^{-1/2}\bigg) \bigg(\frac{\beta_{1}}{1-\beta_{1}} \mb_{t-1}^{\prime}+\bGamma_{t}\bigg)\bigg\|\bigg] \notag \\
    & \leq \eta C_1 \eta_l K G^2 \EE\bigg[\Big\| \Vbh_{t-1}^{-1/2} - \Vbh_{t}^{-1/2}\Big\|_1\bigg] + \eta^2 C_1^2 L \eta_l^2 K^2 G^2 \epsilon^{-1/2} \EE\bigg[\Big\|\Vbh_{t-1}^{-1/2} - \Vbh_{t}^{-1/2}\Big\|_1\bigg],
\end{align}
where the last inequality holds by Lemma \ref{lm:gmv-bound}, here $C_1= \frac{\beta_1}{1-\beta_1}+ \frac{2q}{1-q^2}$.

\textbf{Bounding $T_3$:} It can be bounded as follows:
\begin{align}\label{eq:T3}
    T_3 & = \frac{\eta^2 L }{2} \EE\bigg[ \bigg\|\Vbh_t^{-1/2} \Delta_{t} + \frac{\beta_{1}}{1-\beta_{1}} \bigg(\Vbh_{t-1}^{-1/2} - \Vbh_{t}^{-1/2}\bigg) \mb_{t-1}^{\prime}+\bigg(\Vbh_{t-1}^{-1/2} - \Vbh_{t}^{-1/2}\bigg) \bGamma_{t}\bigg\|^2\bigg] \notag\\
    & \leq \eta^2 L \EE\bigg[ \big\|\Vbh_{t}^{-1/2}  \Delta_{t}\big\|^2\bigg]+ \eta^2 L \EE\bigg[ \bigg\|\frac{\beta_{1}}{1-\beta_{1}} \bigg(\Vbh_{t-1}^{-1/2} - \Vbh_{t}^{-1/2}\bigg) \mb_{t-1}^{\prime}+\bigg(\Vbh_{t-1}^{-1/2} - \Vbh_{t}^{-1/2}\bigg) \bGamma_{t}\bigg\|^2\bigg] \notag\\
    & \leq \eta^2 L \EE  \bigg[\big\|\Vbh_{t}^{-1/2}  \Delta_{t}\big\|^2 \bigg]+ \eta^2 L C_1^2 \eta_l^2 K^2 G^2 \EE\bigg[\Big\|\Vbh_{t-1}^{-1/2} - \Vbh_{t}^{-1/2}\Big\|^2\bigg],
\end{align}
where the first inequality follows by Cauchy-Schwarz inequality, and the second one follows by Lemma \ref{lm:gmv-bound}, here $C_1= \frac{\beta_1}{1-\beta_1}+ \frac{2q}{1-q^2}$. 

\textbf{Bounding $T_4$: }
\begin{align}
    T_4 & = \EE\bigg[\bigg\langle\nabla f(\zb_{t})-\nabla f(\xb_t), \eta \Vbh_{t}^{-1/2}  \Delta_{t}\bigg\rangle\bigg] \notag\\
    & \leq \EE\bigg[\|\nabla f(\zb_{t})-\nabla f(\xb_t)\|\big\|\eta \Vbh_{t}^{-1/2}  \Delta_{t}\big\|\bigg] \notag\\
    & \leq L \EE\bigg[\|\zb_{t}-\xb_t\|\big\|\eta \Vbh_{t}^{-1/2}  \Delta_{t}\big\|\bigg] \notag\\
    & \leq \frac{\eta^2 L}{2} \EE \bigg[\big\|\Vbh_{t}^{-1/2}  \Delta_{t}\big\|^2\bigg] + \frac{\eta^2 L}{2} \EE\bigg[\bigg\|\frac{\beta_1}{1-\beta_1} \Vbh_{t-1}^{-1/2}\mb_{t-1}^{\prime} + \Vbh_{t-1}^{-1/2}\bGamma_t \bigg\|^2 \bigg],\notag
\end{align}
where the first inequality holds by the fact of $\langle \ab, \bb \rangle \leq \|\ab\|\|\bb\|$, the second one follows from Assumption \ref{as:smooth} and the third one holds by the definition of virtual sequence $\zb_t$ and the fact of $\|\ab\|\|\bb\| \leq \frac{1}{2}\|\ab\|^2 + \frac{1}{2}\|\bb\|^2$. Then summing $T_4$ over $t=1,\cdots, T$, we have
\begin{align}\label{eq:T4-1}
    \sum_{t=1}^T T_4 & \leq \frac{\eta^2 L}{2} \sum_{t=1}^T \EE \bigg[\big\|\Vbh_{t}^{-1/2}  \Delta_{t}\big\|^2\bigg] + \frac{\eta^2 L}{2\epsilon}\sum_{t=1}^T  \EE\bigg[\bigg\|\frac{\beta_1}{1-\beta_1} \mb_{t-1}^{\prime} + \bGamma_t \bigg\|^2 \bigg] \notag\\
    & \leq \frac{\eta^2 L}{2\epsilon} \sum_{t=1}^T \EE [\| \Delta_{t}\|^2] + \frac{\eta^2 L}{\epsilon} \bigg[ \frac{\beta_1^2}{(1-\beta_1)^2} \sum_{t=1}^T \EE \|\mb_{t-1}^{\prime}\|^2 +\sum_{t=1}^T \EE\|\bGamma_t\|^2 \bigg].
\end{align}
By Lemma \ref{lm:mE-bound}, we have
\begin{align}
    \sum_{t=1}^T \EE[\|\mb_{t-1}^{\prime}\|^2] \leq \frac{T K \eta_l^2}{m} \sigma_l^2 + \frac{\eta_l^2}{m^2} \sum_{t=1}^T \EE \bigg[\bigg\|\sum_{i=1}^m \sum_{k=0}^{K-1}\nabla F_i(\xb_{t,k}^i ) \bigg\|^2\bigg],\notag
\end{align}
and 
\begin{align}
    \sum_{t=1}^T \EE [\|\bGamma_t\|^2] & \leq \frac{4T(q+\gamma)^2}{(1-q^2)^2} \frac{K \eta_l^2}{m} \sigma_l^2 + \frac{\eta_l^2}{m^2} \frac{4(q+\gamma)^2}{(1-q^2)^2} \sum_{t=1}^T \EE \bigg[\bigg\|\sum_{i=1}^m \sum_{k=0}^{K-1}\nabla F_i(\xb_{t,k}^i ) \bigg\|^2\bigg].\notag
\end{align}
Therefore, the $T_4$ term is bounded by 
\begin{align}\label{eq:T4-2}
    \sum_{t=1}^T T_4 & \leq \frac{\eta^2 L}{2\epsilon}  \sum_{t=1}^T \EE [\|\Delta_{t}\|^2] + \frac{C_2 \eta^2 L}{\epsilon} \frac{\eta_l^2}{m^2} \sum_{t=1}^T \EE \bigg[\bigg\|\sum_{i=1}^m \sum_{k=0}^{K-1}\nabla F_i(\xb_{t,k}^i ) \bigg\|^2\bigg] + \frac{C_2 \eta^2 L}{\epsilon} \frac{TK \eta_l^2}{m} \sigma_l^2,
\end{align}
where $C_2 = \frac{4(q+\gamma)^2}{(1-q^2)^2} + \frac{\beta_1^2}{(1-\beta_1)^2}$.

\textbf{Merging pieces together:}
Substituting \eqref{eq:T1}, \eqref{eq:T2} and \eqref{eq:T3} into \eqref{eq:f-Lsmooth}, summing over from $t=1$ to $T$ and then adding \eqref{eq:T4-2}, we have
\begin{align}\label{eq:fzT-1}
    & \EE[f(\zb_{T+1})]-f(\zb_1) = \sum_{t=1}^{T} [T_1+T_2+T_3+T_4] \notag\\
    & \leq -\frac{\eta\eta_l K}{4} \sum_{t=1}^{T}\EE\bigg[ \bigg\|\frac{\nabla f(\xb_t)}{\sqrt[4]{\beta_2 \vb_{t-1} + \epsilon}} \bigg\|^2\bigg]+ \frac{5\eta\eta_l^3 K^2 L^2 T}{\sqrt{2\epsilon}} (\sigma_l^2+6K \sigma_g^2) + \frac{\sqrt{2(1-\beta_2)} \eta G}{\epsilon} \sum_{t=1}^{T}\EE[\|\Delta_t\|^2]  \notag\\
    & \quad  - \frac{\eta\eta_l}{2 K m^2} \sum_{t=1}^{T} \EE\bigg[ \bigg\|\frac{1}{\sqrt[4]{\beta_2 \vb_{t-1} + \epsilon}}\sum_{i=1}^m \sum_{k=0}^{K-1} \nabla F_i(\xb_{t,k}^i))\bigg\|^2\bigg] + C_1 \eta\eta_l K G^2\sum_{t=1}^{T}\EE\bigg[\Big\|\Vbh_{t-1}^{-1/2} - \Vbh_{t}^{-1/2}\Big\|_1\bigg] \notag\\
    & \quad + \frac{C_1^2\eta^2\eta_l^2 K^2 G^2}{\sqrt{\epsilon}} \sum_{t=1}^{T}\EE\bigg[\Big\|\Vbh_{t-1}^{-1/2} - \Vbh_{t}^{-1/2}\Big\|_1\bigg] + C_1^2\eta^2\eta_l^2K^2 L G^2 \sum_{t=1}^{T}\EE\bigg[\Big\|\Vbh_{t-1}^{-1/2} - \Vbh_{t}^{-1/2}\Big\|^2\bigg] \notag\\
    & \quad + \eta^2 L \sum_{t=1}^{T} \EE \bigg[\big\|\Vbh_{t}^{-1/2}  \Delta_{t}\big\|^2 \bigg] + \frac{\eta^2 L}{2} \sum_{t=1}^{T} \EE \bigg[\big\|\Vbh_{t}^{-1/2} \Delta_{t}\big\|^2 \bigg] + \frac{\eta^2 L}{2} \frac{\beta_1^2}{(1-\beta_1)^2} \sum_{t=1}^{T} \EE[\|\mb_t^\prime\|^2] + \frac{\eta^2 L}{2} \sum_{t=1}^{T} \EE [\|\bGamma_t\|^2].
\end{align}
Hence by organizing and applying Lemmas, we have
\begin{align}
    & \EE[f(\zb_{T+1})]-f(\zb_1) \notag\\
    & \leq -\frac{\eta\eta_l K}{4} \sum_{t=1}^{T} \EE\bigg[\bigg\|\frac{\nabla f(\xb_t)}{\sqrt[4]{\beta_2 \vb_{t-1} + \epsilon}} \bigg\|^2\bigg]+ \frac{5\eta\eta_l^3 K^2 L^2 T}{\sqrt{2\epsilon}} (\sigma_l^2+6K \sigma_g^2)  \notag\\
    & \quad - \frac{\eta\eta_l}{2 K m^2} \sum_{t=1}^{T} \EE\bigg[ \bigg\|\frac{1}{\sqrt[4]{\beta_2 \vb_{t-1} + \epsilon}} \sum_{i=1}^m \sum_{k=0}^{K-1} \nabla F_i(\xb_{t,k}^i))\bigg\|^2\bigg] +  \frac{C_1 \eta \eta_l K G^2 d}{\sqrt{\epsilon}} + \frac{2  C_1^2  \eta^2 \eta_l^2 K^2 L G^2 d}{\epsilon} \notag\\
    & \quad + \bigg(\eta^2 L + \frac{\eta^2 L}{2} + \sqrt{2(1-\beta_2)} \eta G \bigg) \bigg[ \frac{K T \eta_l^2}{m\epsilon} \sigma_l^2 + \frac{ \eta_l^2}{m^2 \epsilon} \sum_{t=1}^{T} \EE \bigg[\bigg\|\sum_{i=1}^m \sum_{k=0}^{K-1} \nabla F_i(\xb_{t,k}^i)\bigg\|^2 \bigg]\bigg] \notag\\
    & \quad + \frac{\eta^2 L}{\epsilon}\frac{\eta_l^2 C_2}{m^2} \sum_{t=1}^{T} \EE \bigg[\bigg\|\sum_{i=1}^m \sum_{k=0}^{K-1} \nabla F_i(\xb_{t,k}^i)\bigg\|^2 \bigg] + \frac{\eta^2 L}{\epsilon}\frac{T K \eta_l^2 C_2}{m}\sigma_l^2,
\end{align}
by applying Lemma \ref{lm:EDelta} into all terms containing the second moment estimate of model difference $\Delta_t$ in \eqref{eq:fzT-1}, using the fact that $\Big(\sqrt{\beta_2 \frac{(1+q^2)^3}{(1-q^2)^2} K^2 G^2 + \epsilon}\Big)^{-1} \|\xb\| \leq \Big(\sqrt{\beta_2 \frac{(1+q^2)^3}{(1-q^2)^2} \eta_l^2 K^2 G^2 + \epsilon}\Big)^{-1}\|\xb\| \leq \big\|\frac{\xb}{\sqrt{\beta_2 \vb_t + \epsilon}}\big\| \leq \epsilon^{-1/2} \|\xb\|$, and applying Lemma \ref{lm:Dt} and \ref{lm:mE-bound-partial}, we have
\begin{align}
    & \EE[f(\zb_{T+1})]-f(\zb_1) \notag\\
    & \leq -\frac{\eta\eta_l K}{4 \sqrt{4 \beta_2 \frac{(1+q^2)^3}{(1-q^2)^2}\eta_l^2 K^2 G^2+\epsilon}} \sum_{t=1}^{T} \EE[\|\nabla f(\xb_t)\|^2]+ \frac{5\eta\eta_l^3 K^2 L^2 T}{\sqrt{2\epsilon}} (\sigma_l^2+6K \sigma_g^2) \notag\\
    & \quad + \frac{C_1 \eta \eta_l K G^2 d}{\sqrt{\epsilon}} + \frac{2  C_1^2  \eta^2 \eta_l^2 K^2 L G^2 d}{\epsilon} + \bigg(\frac{3\eta^2 L}{2} + C_2 \eta^2 L  + \sqrt{2(1-\beta_2)} \eta G\bigg) \frac{KT\eta_l^2}{m \epsilon} \sigma_l^2 \notag\\
    & \quad - \sum_{t=1}^{T} \EE \bigg[\bigg\|\sum_{i=1}^m \sum_{k=0}^{K-1} \nabla F_i(\xb_{t,k}^i)\bigg\|^2 \bigg] \bigg[\frac{\eta\eta_l}{2 \sqrt{4 \beta_2 \frac{(1+q^2)^3}{(1-q^2)^2}\eta_l^2 K^2 G^2+\epsilon} K m^2} - \bigg(\frac{3\eta^2 L}{2} + C_2 \eta^2 L  + \sqrt{2(1-\beta_2)} \eta G\bigg) \frac{\eta_l^2}{m^2 \epsilon} \bigg]\notag\\
    & \leq -\frac{\eta\eta_l K}{4 C_0} \sum_{t=1}^{T} \EE[\|\nabla f(\xb_t)\|^2]+ \frac{5\eta\eta_l^3 K^2 L^2 T}{\sqrt{2\epsilon}} (\sigma_l^2+6K \sigma_g^2)  \notag\\
    & \quad + \frac{C_1 \eta \eta_l K G^2 d}{\sqrt{\epsilon}} + \frac{2  C_1^2  \eta^2 \eta_l^2 K^2 L G^2 d}{\epsilon} + \bigg(\frac{3\eta^2 L}{2} + C_2 \eta^2 L + \sqrt{2(1-\beta_2)} \eta G \bigg) \frac{KT\eta_l^2}{m \epsilon} \sigma_l^2,
\end{align}
where the last inequality holds by $\eta_l \leq \frac{\epsilon}{\sqrt{4 \beta_2 (1+q^2)^3 (1-q^2)^{-2} K^2 G^2+\epsilon} \cdot K (3\eta L + 2C_2 \eta L + 2\sqrt{2(1-\beta_2)} G)}$.
Hence we have 
\begin{align}\label{eq:fzT-3}
    & \frac{\eta\eta_l K}{4 \sqrt{4 \beta_2 \frac{(1+q^2)^3}{(1-q^2)^2}\eta_l^2 K^2 G^2+\epsilon} \cdot T} \sum_{t=1}^{T} \EE[\|\nabla f(x_t)\|^2] \notag\\
    & \leq \frac{f(\zb_0)-\EE [f(\zb_T)]}{T} + \frac{5\eta\eta_l^3 K^2 L^2}{\sqrt{2\epsilon}} (\sigma_l^2+6K \sigma_g^2) + \frac{C_1 \eta \eta_l K G^2 d}{T \sqrt{\epsilon}} + \frac{2  C_1^2  \eta^2 \eta_l^2 K^2 L G^2 d}{T \epsilon} \notag\\
    & \quad + \big[3\eta^2 L+ 2C_2 \eta^2 L + 2\sqrt{2(1-\beta_2)} \eta G\big] \frac{K\eta_l^2}{2m \epsilon} \sigma_l^2,
\end{align}
where $C_1= \frac{\beta_1}{1-\beta_1}+ \frac{2q}{1-q^2}$ and $C_2 = \frac{\beta_1^2}{(1-\beta_1)^2}+ \frac{4(q+\gamma)^2}{(1-q^2)^2}$. \eqref{eq:fzT-3} also implies, 
\begin{align}
    \min \EE [\|\nabla f(\xb_t)\|^2] \leq 4 \sqrt{4 \beta_2 \frac{(1+q^2)^3}{(1-q^2)^2}\eta_l^2 K^2 G^2+\epsilon} \bigg[\frac{f_0-f_*}{\eta \eta_l K T} + \frac{\Psi}{T} + \Phi \bigg],
\end{align}
where $\Psi = \frac{C_1 G^2 d}{\sqrt{\epsilon}}+ \frac{2C_1^2 \eta\eta_l K L G^2 d}{\epsilon}, \Phi = \frac{5 \eta_l^2 K L^2}{\sqrt{2\epsilon}} (\sigma_l^2+6K \sigma_g^2)+ [(3+2C_2)\eta L + 2\sqrt{2(1-\beta_2)} G] \frac{\eta_l}{2m \epsilon} \sigma_l^2$, $C_1 = \frac{\beta_1}{1-\beta_1} + \frac{2q}{1-q^2}$ and $C_2 = \frac{\beta_1^2}{(1-\beta_1)^2} + \frac{4(q+\gamma)^2}{(1-q^2)^2}$.

\subsection{Proof of Corollary \ref{cor:cfedams-full}}

Let $\eta_l = \Theta(\frac{1}{\sqrt{T} K})$ and $\eta = \Theta(\sqrt{Km}) $, the convergence rate under full participation scheme is $\cO(\frac{1}{\sqrt{TKm}})$.

\subsection{Analysis on the Partial Participation Setting for FedCAMS} \label{sec:ext}
Let us present the theoretical analysis of the partial participation scheme of FedCAMS (Algorithm \ref{alg:compfedams}). Similar to partial participation scheme in Section \ref{subsec:fedams}, we have the following convergence analysis. 
\begin{theorem} \label{thm:part-cfedams}
Under Assumption~\ref{as:smooth}-\ref{as:bounded-v} and \ref{as:compressor}, if the local learning rate $\eta_l$ satisfies the following condition: $\eta_l \leq \min\Big\{\frac{1}{8KL}, \frac{n(m-1)\epsilon}{48m(n-1)}[K\sqrt{4\beta_2 (1+q^2)^3(1-q^2)^{-2} K^2 G^2+\epsilon} (\eta L + \sqrt{2(1-\beta_2)}G)]^{-1} \Big\}$ , then the iterates of Algorithm \ref{alg:compfedams} under partial participation scheme satisfy
\begin{align}\label{eq:part-cfedams}
    \min_{t\in [T]} \EE \|\nabla f(\xb_t)\|^2 \leq 8 \sqrt{4 \beta_2 \frac{(1+q^2)^3}{(1-q^2)^2} \eta_l^2 K^2 G^2 + \epsilon} \bigg[\frac{f_0-f_*}{\eta \eta_l K T} + \frac{\Psi}{T} + \Phi \bigg],
\end{align}
where $\Psi = \frac{C_1 G^2 d}{\sqrt{\epsilon}}+ \frac{2C_1^2 \eta\eta_l K L G^2 d}{\epsilon}$, $\Phi = \frac{C_1\eta\eta_l KLG^2}{\epsilon} + \frac{5 \eta_l^2 K L^2}{ \sqrt{2\epsilon}} (\sigma_l^2+6K \sigma_g^2)+ [\eta L + \sqrt{2(1-\beta_2)} G] \frac{\eta_l}{n\epsilon}\sigma_l^2 + [\eta L + \sqrt{2(1-\beta_2)} G]\frac{\eta_l(m-n)}{n(m-1)\epsilon}  [15 K^2L^2 \eta_l^2 (\sigma_l^2 + 6K\sigma_g^2) + 3K\sigma_g^2]$ and $C_1 = \frac{\beta_1}{1-\beta_1} + \frac{2q}{1-q^2}$.
\end{theorem}
\begin{remark}
The upper bound for $\min_{t\in [T]} \EE \|\nabla f(\xb_t)\|^2$ of partial participation is similar to full participation case but with a larger variance term $\Phi$. This is due to the fact that random sampling of participating workers introduces an additional variance during sampling. 
\end{remark}

\begin{remark}
From Theorem \ref{thm:part-cfedams}, constants $C_1$ is related to the compressor constant $q$. The stronger compression we apply to the model difference $\Delta_t^i$ corresponding to larger $q$ ($q \to 1$) leads to worse convergence due to larger information losses. 
\end{remark}

Next, we provide theoretical proofs for the partial participation analysis for FedCAMS.

\textbf{Proof of Theorem \ref{thm:part-cfedams}: }

\textit{Notations and equations:} From the update rule of Algorithm \ref{alg:compfedams}, we have $\eb_1 = 0$, $\eb_t= \frac{1}{m} \sum_{i=1}^m \eb_t^i$ and $\mb_{t}=(1-\beta_1) \sum_{i=1}^{t} \beta_{1}^{t-i} \hat{\Delta}_{i}$. Denote a global uncompressed difference $\Delta_t= \frac{1}{|\cS_t|} \sum_{i \in \cS_t} \Delta_i^t$. Denote a virtual momentum sequence: $\mb_{t}^{\prime}=\beta_{1} \mb_{t-1}^{\prime}+(1-\beta_1) \Delta_t$, hence we have $\mb_{t}^{\prime}=(1-\beta_1) \sum_{i=1}^{t} \beta_{1}^{t-i} \Delta_{i}$. 
Define additional two virtual sequences $\Delta'_t = \frac{1}{n} \sum_{i=1}^m \Delta_t^i$ and $\hat{\Delta}'_t = \frac{1}{n} \sum_{i=1}^m \hat{\Delta}_t^i$. Note that when the client $i$ does not take part in the round of participation at step $t$, we have $\Delta_t^i = \hat{\Delta}_t^i = 0$, therefore, $\Delta'_t = \Delta_t$ and $\hat{\Delta}'_t = \hat{\Delta}_t$.

By the aforementioned definition and notation, define a subset $\cS_t=\{w_1^t, w_2^t,..., w_n^t \}$, we have
\begin{align}
    \hat{\Delta}_t - \Delta_t = \frac{1}{|\cS_t|} \sum_{i \in \cS_t} ( \hat{\Delta}_t^i - \Delta_t^i) = \frac{1}{n} \sum_{i=1}^m (\hat{\Delta}_t^i - \Delta_t^i) = \frac{1}{n} \sum_{i=1}^m (\eb_t^i - \eb_{t+1}^i) = \eb'_t - \eb'_{t+1},
\end{align}
where the compression errors have the same structure, $\eb'_t = \frac{1}{n} \sum_{i=1}^m \eb_t^i$. Similar to the previous analysis, we define the following sequence:
\begin{align*}
    \bGamma_{t+1} & := (1-\beta_1) \sum_{\tau=1}^{t+1} \beta_1^{t+1-\tau} \eb'_\tau,
\end{align*}
and keep using the Lyapunov function $\zb_t$ from \eqref{eq:lyapunov-z}.
For the expectation of model difference $\Delta_t$, we have 
\begin{align}
    \EE_{\cS_t}[\Delta_t] = \frac{1}{n} \EE_{\cS_t}\bigg[\sum_{i=1}^n \Delta_t^{w_i}\bigg] = \EE_{\cS_t}[\Delta_t^{w_1}] = \frac{1}{m} \sum_{i=1}^m \Delta_t^i = \Bar{\Delta}_t.
\end{align}
The proof of FedCAMS in partial participation settings has a similar outline combing the proof of partial participation in FedAMS and full participation in FedCAMS. By Assumption \ref{as:smooth}, we have
\begin{align}\label{eq:f-Lsmooth'}
& \EE[f(\zb_{t+1})]-f(\zb_{t}) \notag\\
& \leq \underbrace{\EE\bigg[\bigg\langle\nabla f(\xb_{t}), \eta\Vbh_{t}^{-1/2}  \Delta_{t}\bigg\rangle\bigg]}_{T'_1} \notag\\
& \quad \underbrace{- \EE\bigg[\bigg\langle\nabla f(\zb_{t}), \eta\frac{\beta_{1}}{1-\beta_{1}} \bigg(\Vbh_{t-1}^{-1/2} - \Vbh_{t}^{-1/2}\bigg) m_{t-1}^{\prime}+\bigg(\Vbh_{t-1}^{-1/2} - \Vbh_{t}^{-1/2}\bigg) \bGamma_{t}\bigg\rangle\bigg]}_{T'_2}\notag\\
& \quad + \underbrace{\frac{\eta^{2} L}{2} \EE\bigg[\bigg\|\Vbh_{t}^{-1/2}  \Delta_{t}-\frac{\beta_{1}}{1-\beta_{1}} \bigg(\Vbh_{t-1}^{-1/2} - \Vbh_{t}^{-1/2}\bigg) m_{t-1}^{\prime}-\bigg(\Vbh_{t-1}^{-1/2} - \Vbh_{t}^{-1/2}\bigg) \bGamma_{t}\bigg\|^{2}\bigg]}_{T'_3}\notag\\
& \quad +  \underbrace{\EE\bigg[\bigg\langle\nabla f(\zb_{t})-\nabla f(\xb_t), \eta\Vbh_{t}^{-1/2}  \Delta_{t}\bigg\rangle\bigg]}_{T'_4}.
\end{align}

Note that the bound for $T'_2$ is exactly the same as the bound for $T_2$. For the three corresponding terms, $T'_1$, $T'_3$ and $T'_4$ which include the second-order momentum estimate of $\Delta_t$. For $T'_1$, similar to the full participation settings, we have
\begin{align}\label{eq:T1'-1}
    T'_1 & \leq \sqrt{2} \EE\bigg[\bigg\langle\nabla f(\xb_{t}), \eta\frac{\Delta_{t}}{\sqrt{\beta_2 \vb_{t-1}+\epsilon}}\bigg\rangle\bigg] + \sqrt{2}\eta\EE\bigg[\bigg\langle\nabla f(\xb_{t}),\frac{\Delta_{t}}{\sqrt{\vb_t + \epsilon}}- \frac{\Delta_{t}}{\sqrt{\beta_2 \vb_{t-1}+\epsilon}}\bigg\rangle\bigg].
\end{align}
The first term in \eqref{eq:T1'-1} does not change in partial participation scheme. The second term is changed due to the variance of $\Delta_t$ changes. For the second term of $T'_1$, we have
\begin{align}\label{eq:T1'-2}
    \sqrt{2}\eta\EE\bigg[\bigg\langle\nabla f(\xb_{t}),\frac{\Delta_{t}}{\sqrt{\vb_t + \epsilon}} - \frac{\Delta_{t}}{\sqrt{\beta_2 \vb_{t-1}+\epsilon}}\bigg\rangle\bigg] \leq \frac{\sqrt{2(1-\beta_2)}\eta G}{\epsilon} \EE[\|\Delta_t\|^2].
\end{align}
For $T'_3$, similar to the proof of $T_3$, we have
\begin{align}
    \sum_{t=1}^T T'_3 \leq \frac{\eta^2 L}{\epsilon} \sum_{t=1}^T \EE [\|\Delta_t\|^2] + \eta^2 L C_1^2 \eta_l^2 K^2 G^2 \sum_{t=1}^T \EE \bigg[\Big\|\Vbh_{t-1}^{-1/2} - \Vbh_{t}^{-1/2}\Big\|^2\bigg],
\end{align}
where $C_1 = \frac{\beta_1}{1-\beta_1} + \frac{m}{n}\frac{2q}{1-q^2}$. 
For $T_4'$ in partial participation, we have
\begin{align}
    T'_4 & = \eta \EE\bigg[\bigg\langle f(\zb_t) - f(\xb_t), \Vbh_t^{-1/2} \Delta_t \bigg\rangle\bigg] \notag\\
    & \leq \eta \EE\bigg[\|f(\zb_t) - f(\xb_t)\| \Big\|\Vbh_t^{-1/2} \Delta_t \Big\|\bigg] \notag\\
    & \leq \eta^2 L \EE \bigg[\bigg\|\frac{\beta_1}{1-\beta_1} \Vbh_{t-1}^{-1/2} \mb'_{t-1}+\Vbh_{t-1}^{-1/2} \bGamma_t \bigg\| \Big\| \Vbh_t^{-1/2} \Delta_t \Big\| \bigg] \notag\\
    & \leq \frac{C_1 \eta^2 \eta_l^2 K^2 L G^2}{\epsilon}.
\end{align}
Hence, the summation from $T'_1$ to $T'_4$ over total iteration $T$ is:
\begin{align}\label{eq:fzT-1''}
    & \EE[f(\zb_{T+1})]-f(\zb_1) = \sum_{t=1}^{T} [T'_1+T'_2+T'_3+T'_4] \notag\\
    & \leq -\frac{\eta\eta_l K}{4} \sum_{t=1}^{T}\EE \bigg[\bigg\|\frac{\nabla f(\xb_t)}{\sqrt[4]{\beta_2 \vb_{t-1}+\epsilon}}\bigg\|^2\bigg]+ \frac{5\eta\eta_l^3 K^2 L^2 T}{\sqrt{2\epsilon}} (\sigma_l^2+6K \sigma_g^2) +\frac{\sqrt{2(1-\beta_2)}\eta G}{\epsilon} \sum_{t=1}^{T} \EE[\|\Delta_t\|^2] \notag\\
    & \quad - \frac{\eta\eta_l}{2 K m^2} \sum_{t=1}^{T}\EE \bigg[ \bigg\|\frac{1}{\sqrt[4]{\beta_2 \vb_{t-1}+\epsilon}}\sum_{i=1}^m \sum_{k=0}^{K-1} \nabla F_i(\xb_t))\bigg\|^2\bigg] + C_1\eta \eta_l K G^2 \sum_{t=1}^{T}\EE\bigg[\Big\|\Vbh_{t-1}^{-1/2} - \Vbh_{t}^{-1/2}\Big\|_1\bigg] \notag\\
    & \quad + C_1^2 \eta^2 \eta_l^2 K^2 L G^2 \epsilon^{-1/2} \sum_{t=1}^{T}\EE\bigg[\Big\|\Vbh_{t-1}^{-1/2} - \Vbh_{t}^{-1/2}\Big\|_1\bigg] + C_1^2 \eta^2 \eta_l^2 K^2 L G^2 \sum_{t=1}^T \EE \bigg[\Big\|\Vbh_{t-1}^{-1/2} - \Vbh_{t}^{-1/2}\Big\|^2\bigg] \notag\\
    & \quad + \frac{\eta^2 L}{\epsilon} \sum_{t=1}^T \EE [\|\Delta_t\|^2] + \frac{C_1 T \eta^2 \eta_l^2 K^2 L G^2}{\epsilon} \notag\\
    & \leq -\frac{\eta\eta_l K}{4\sqrt{4 \beta_2 \frac{(1+q^2)^3}{(1-q^2)^2} \eta_l^2 K^2 G^2 + \epsilon}} \sum_{t=1}^{T}\EE [\|\nabla f(x_t)\|^2]+ \frac{5\eta\eta_l^3 K^2 L^2 T}{\sqrt{2\epsilon}} (\sigma_l^2+6K \sigma_g^2) + \frac{C_1 \eta \eta_l K G^2 d}{T \sqrt{\epsilon}} \notag\\
    & \quad + \frac{2  C_1^2  \eta^2 \eta_l^2 K^2 L G^2 d}{T \epsilon} - \frac{\eta\eta_l}{2 \sqrt{4 \beta_2 \frac{(1+q^2)^3}{(1-q^2)^2} \eta_l^2 K^2 G^2 + \epsilon} K m^2} \sum_{t=1}^{T}\EE \bigg[ \bigg\|\sum_{i=1}^m \sum_{k=0}^{K-1} \nabla F_i(\xb_t))\bigg\|^2\bigg]\notag\\
    & \quad + \bigg( \frac{\eta^2\eta_l^2 L K T}{n \epsilon}+ \frac{\sqrt{2(1-\beta_2)} \eta\eta_l^2 K T G}{n\epsilon} \bigg) \sigma_l^2 + \frac{C_1 T \eta^2 \eta_l^2 K^2 L G^2}{\epsilon} \notag\\
    & \quad + \bigg(\frac{\eta^2\eta_l^2 L}{\epsilon} + \frac{\sqrt{2(1-\beta_2)}\eta\eta_l^2 G}{\epsilon} \bigg)\frac{m-n}{mn(m-1)} \bigg[ 15 m K^3 L^3 \eta_l^2(\sigma_l^2+ 6K \sigma_g^2)T \notag\\
    & \quad + (90mK^4L^2\eta_l^2+3mK^2)\sum_{t=1}^{T} \EE[\|\nabla f(\xb_t)\|^2] +3mK^2 T \sigma_g^2 \bigg] \notag\\
    & \quad + \bigg(\frac{\eta^2\eta_l^2 L}{\epsilon} + \frac{\sqrt{2(1-\beta_2)}\eta\eta_l^2 G}{\epsilon}\bigg) \frac{n-1}{mn(m-1)}\sum_{t=1}^T \EE \bigg[ \bigg\|\sum_{i=1}^m \sum_{k=0}^{K-1} \nabla F_i(\xb_t))\bigg\|^2\bigg].
\end{align}
The proof outline is similar with previous proof. We take the use of Lemma  \ref{lm:Dt}, \ref{lm:EDelta'}, \ref{lm:mE-bound-partial} for corresponding terms. 
By additional constraints of local learning rate $\eta_l$ with the inequality $\Big[\eta^2 L + \sqrt{2(1-\beta_2)}\eta G \Big] \frac{\eta_l^2 (n-1)}{mn(m-1)\epsilon} - \frac{\eta\eta_l}{2 K m^2} \Big[\sqrt{4 \beta_2 \frac{(1+q^2)^3}{(1-q^2)^2} \eta_l^2 K^2 G^2 + \epsilon}\Big]^{-1} \leq 0$, we obtain the constraint $\eta_l \leq \frac{n(m-1)}{m(n-1)}\frac{\epsilon}{2 K \sqrt{4\beta_2 (1+q^2)^3 (1-q^2)^{-2} K^2 G^2 + \epsilon}[\eta L + \sqrt{2(1-\beta_2)} G]}$, and we further need $\eta_l$ satisfies $\frac{\eta\eta_l K}{4\sqrt{4 \beta_2 (1+q^2)^3 (1-q^2)^{-2}\eta_l^2 K^2 G^2 + \epsilon}} - (\eta^2 L + \sqrt{2(1-\beta_2)} \eta G )\frac{\eta_l^2 (m-n)}{mn(m-1)\epsilon} (90mK^4L^2\eta_l^2 + 3mK^2) \geq \frac{\eta\eta_l K}{8\sqrt{4 \beta_2 (1+q^2)^3 (1-q^2)^{-2} \eta_l^2 K^2 G^2 + \epsilon}}$. Hence for the convergence rate, we have
\begin{align}\label{eq:fzT-3''}
    & \frac{\eta\eta_l K}{8 \sqrt{4 \beta_2 \frac{(1+q^2)^3}{(1-q^2)^2} \eta_l^2 K^2 G^2 + \epsilon}\cdot T} \sum_{i=1}^T \EE[\|\nabla f(x_t)\|^2] \notag\\
    & \leq \frac{f(\zb_0)-\EE [f(\zb_T)]}{T} + \frac{5\eta\eta_l^3 K^2 L^2}{\sqrt{2\epsilon}} (\sigma_l^2+6K \sigma_g^2) + \big(\eta L + \sqrt{2(1-\beta_2)} G \big) \frac{\eta \eta_l^2 K}{n\epsilon}\sigma_l^2 \notag\\
    & + \frac{C_1 \eta \eta_l K G^2 d}{T \sqrt{\epsilon}} + \frac{2  C_1^2  \eta^2 \eta_l^2 K^2 L G^2 d}{T \epsilon} + \frac{C_1 \eta^2 \eta_l^2 K^2 L G^2}{\epsilon}\notag\\
    & + \bigg(\frac{\eta^2\eta_l^2 L}{\epsilon} + \frac{\sqrt{2(1-\beta_2)}\eta\eta_l^2 G}{\epsilon} \bigg)\frac{m-n}{mn(m-1)}[ 15 m K^3 L^2 \eta_l^2(\sigma_l^2+ 6K \sigma_g^2) + 3m K^2\sigma_g^2].
\end{align}
Therefore, 
\begin{align}
    \min \EE[\|\nabla f(x_t)\|^2] & \leq 8 \sqrt{4 \beta_2 \frac{(1+q^2)^3}{(1-q^2)^2} \eta_l^2 K^2 G^2 + \epsilon} \bigg[\frac{f_0-f_*}{\eta \eta_l K T} + \frac{\Psi}{T} + \Phi \bigg],
\end{align}
where $\Psi = \frac{C_1 G^2 d}{\sqrt{\epsilon}}+ \frac{2C_1^2 \eta\eta_l K L G^2 d}{\epsilon}, \Phi = \frac{C_1\eta\eta_l KLG^2}{\epsilon} + \frac{5 \eta_l^2 K L^2}{ \sqrt{2\epsilon}} (\sigma_l^2+6K \sigma_g^2)+ [\eta L + \sqrt{2(1-\beta_2)} G] \frac{\eta_l}{n\epsilon}\sigma_l^2 + [\eta L + \sqrt{2(1-\beta_2)} G]\frac{\eta_l(m-n)}{n(m-1)\epsilon}  [15 K^2L^2 \eta_l^2 (\sigma_l^2 + 6K\sigma_g^2) + 3K\sigma_g^2]$ and $C_1 = \frac{\beta_1}{1-\beta_1} + \frac{m}{n}\frac{2q}{1-q^2}$.


\section{Supporting Lemmas}
\begin{lemma} \label{lm:Wt}
    For the element-wise difference, $W_t = \frac{1}{\sqrt{\vb_{t}+\epsilon}}- \frac{1}{\sqrt{\beta_2\vb_{t-1}+\epsilon}}$, we have $\|W_t\| \leq \frac{\sqrt{1-\beta_2}}{\epsilon} \|\Delta_t\|$.
\end{lemma}
\begin{proof}
    Note that we have:
\begin{align} \label{eq:Wt}
    \|W_t\| & =  \bigg\|\frac{1}{\sqrt{\vb_{t}+\epsilon}}- \frac{1}{\sqrt{\beta_2 \vb_{t-1} + \epsilon}} \bigg\| \notag\\
    & =  \bigg\|\frac{(\sqrt{\beta_2 \vb_{t-1} + \epsilon}-\sqrt{\vb_{t}+\epsilon}) (\sqrt{\beta_2 \vb_{t-1} + \epsilon}+\sqrt{\vb_{t}+\epsilon}) }{\sqrt{\vb_{t}+\epsilon} \sqrt{\beta_2 \vb_{t-1} + \epsilon}(\sqrt{\beta_2 \vb_{t-1} + \epsilon}+\sqrt{\vb_{t}+\epsilon})} \bigg\| \notag\\
    & =  \bigg\|\frac{\beta_2 \vb_{t-1}- \vb_t }{\sqrt{\vb_{t}+\epsilon} \sqrt{\beta_2 \vb_{t-1}+\epsilon}(\sqrt{\beta_2 \vb_{t-1} + \epsilon}+\sqrt{\vb_{t}+\epsilon})} \bigg\| \notag\\
    & =  \bigg\|\frac{-(1-\beta_2)\Delta_t^2}{\sqrt{\vb_{t}+\epsilon} \sqrt{\beta_2 \vb_{t-1}+\epsilon}(\sqrt{\beta_2 \vb_{t-1} + \epsilon}+\sqrt{\vb_{t}+\epsilon})} \bigg\| \notag\\
    & \leq  \bigg\|\frac{(1-\beta_2)\Delta_t^2}{\sqrt{\vb_{t}+\epsilon} \sqrt{\beta_2 \vb_{t-1} + \epsilon} \sqrt{1-\beta_2} \Delta_t} \bigg\| \notag\\
    & \leq \frac{\sqrt{1-\beta_2} }{\epsilon}  \|\Delta_t\|,
\end{align}
where the forth equation holds by the update rule of $\vb_t$, i.e., $\vb_t = \beta_2 \vb_{t-1} + (1-\beta_2)\Delta_t^2$,
and the first inequality holds due $\sqrt{\vb_t + \epsilon} \geq \sqrt{\vb_t} \geq \sqrt{1-\beta_2} \Delta_t$ and $\sqrt{\beta_2 \vb_{t-1} +\epsilon} \geq 0$. This concludes the proof.
\end{proof}

\begin{lemma}\label{lm:Dt}
    For the variance difference sequence $\Vbh_{t-1}^{-1/2} - \Vbh_{t}^{-1/2}$, we have 
\begin{align}\label{eq:Dt}
    \sum_{t=1}^{T} \Big\|\Vbh_{t-1}^{-1/2} - \Vbh_{t}^{-1/2}\Big\|_1 \leq \frac{d}{\sqrt{\epsilon}}, \sum_{t=1}^{T} \Big\|\Vbh_{t-1}^{-1/2} - \Vbh_{t}^{-1/2}\Big\|^2 \leq \frac{d}{\epsilon}.
\end{align}
\end{lemma}

\begin{proof}
By the definition of variance matrix $\Vbh_t$, and the non-decreasing update of FedCAMS, i.e., $\hat{\vb}_{t-1} \leq \hat{\vb}_{t} = \max(\hat{\vb}_{t-1}, \vb_t, \epsilon) $, we have 
\begin{align}
    \sum_{t=1}^{T} \Big\|\Vbh_{t-1}^{-1/2} - \Vbh_{t}^{-1/2}\Big\|_1 & = \sum_{t=1}^{T} \bigg\|\frac{1}{\sqrt{\hat{\vb}_{t-1}}}-\frac{1}{\sqrt{\hat{\vb}_{t}}}\bigg\|_1 \notag\\
    & = \sum_{t=1}^{T} \bigg[ \bigg\|\frac{1}{\sqrt{\hat{\vb}_{t-1}}}\bigg\|_1 - \bigg\|\frac{1}{\sqrt{\hat{\vb}_{t}}}\bigg\|_1 \bigg]\notag\\
    & = \bigg\|\frac{1}{\sqrt{\hat{\vb}_{0}}}\bigg\|_1 - \bigg\|\frac{1}{\sqrt{\hat{\vb}_{T}}}\bigg\|_1 \notag\\
    & \leq \frac{d}{\sqrt{\epsilon}},
\end{align}
where the inequality holds by the definition of $\hat{\vb}_t \in \RR^d$. For the sum of the variance difference under $\ell_2$ norm, we have 
\begin{align}
    \sum_{t=1}^{T} \Big\|\Vbh_{t-1}^{-1/2} - \Vbh_{t}^{-1/2}\Big\|^2 & = \sum_{t=1}^{T} \bigg\|\frac{1}{\sqrt{\hat{\vb}_{t-1}}}-\frac{1}{\sqrt{\hat{\vb}_{t}}}\bigg\|^2 \notag\\
    & = \sum_{t=1}^T \bigg(\frac{1}{\sqrt{\hat{\vb}_{t-1}}}-\frac{1}{\sqrt{\hat{\vb}_{t}}} \bigg)^2\notag\\
    & \leq \sum_{t=1} ^T \bigg(\frac{1}{\hat{\vb}_{t-1}} - \frac{1}{\hat{\vb}_{t}} \bigg) \notag\\
    & \leq \frac{1}{\hat{\vb}_{0}} - \frac{1}{\hat{\vb}_{T}} \notag\\
    & \leq \frac{d}{\epsilon},
\end{align}
where the first inequality holds by the element-wise operation: $\forall \xb, \yb \in \RR^d, \mathbf{0} \leq \yb \leq \xb$, we have $(\xb-\yb)^2 \leq (\xb-\yb)(\xb+\yb) = \xb^2 - \yb^2$. It concludes the proof.
\end{proof}

\begin{lemma}\label{lm:error}
The compression error has the following absolute bound
\begin{align}
    \|\eb_t^i\|^2 \leq \frac{4q^2}{(1-q^2)^2} \eta_l^2 K^2 G^2, \quad \|\eb_t\|^2 \leq \frac{4q^2}{(1-q^2)^2} \eta_l^2 K^2 G^2.
\end{align}
\end{lemma}
\begin{proof}
For all $t \in [T]$, by Assumption \ref{as:compressor} and Young's inequality, we have
\begin{align*}
    \|\eb_{t+1}^i\|^2 & = \|\Delta_t^i + \eb_t^i - \cC(\Delta_t^i + \eb_t^i)\|^2 \\
    & \leq q^2 \|\Delta_t^i + \eb_t^i\|^2 \\
    & \leq q^2 (1+\rho)\|\eb_t^i\|^2 + q^2 \bigg(1+\frac{1}{\rho}\bigg) \|\Delta_t^i\|^2 \\
    & \leq \frac{1+q^2}{2} \|\eb_t^i\|^2 + \frac{2q^2}{1-q^2}\|\Delta_t^i\|^2,
\end{align*}
where the last inequality holds by choosing $\rho = \frac{1-q^2}{2q^2}$. Thus we obtain the absolute bound for the error terms,
\begin{align}
    \|\eb_t^i\|^2 & \leq \frac{4q^2}{(1-q^2)^2} \eta_l^2 K^2 G^2, \notag\\
    \|\eb_t\|^2  & = \bigg\|\frac{1}{m} \sum_{i=1}^m \eb_t^i \bigg\|^2 \leq \frac{1}{m} \sum_{i=1}^m \|\eb_t^i\|^2 \leq \frac{4q^2}{(1-q^2)^2} \eta_l^2 K^2 G^2.
\end{align}
In the case of partial participation, suppose that client $i$ has the participated time set $\cT_i$, and we rewrite the $\cT_i = \{t_0, t_1,.., t_{p_i} \}$, where $t_0 < t_1 < \cdots < t_{p_i}$.
Since when client $i$ are not selected to participate local training, the error stay unchanged. Then for $t_s \in \cT_i$ we have 
\begin{align*}
    \|\eb_{t_{s+1}}^i\|^2 & \leq \frac{1+q^2}{2} \|\eb_{t_{s}}^i\|^2 + \frac{2q^2}{1-q^2}\|\Delta_{t_{s}}^i\|^2 \\
    & \leq \frac{1+q^2}{2} \bigg[\frac{1+q^2}{2} \|\eb_{t_{s-1}}^i\|^2 + \frac{2q^2}{1-q^2}\|\Delta_{t_{s-1}}^i\|^2 \bigg] + \frac{2q^2}{1-q^2}\|\Delta_{t_{s}}^i\|^2 \\
    & = \bigg(\frac{1+q^2}{2} \bigg)^2 \|\eb_{t_{s-1}}^i\|^2 + \frac{1+q^2}{2} \cdot \frac{2q^2}{1-q^2}\|\Delta_{t_{s-1}}^i\|^2 + \frac{2q^2}{1-q^2}\|\Delta_{t_{s}}^i\|^2,
\end{align*}
thus by the similar recursive approach, since we have $\eb_{t_0}^i = 0$, we have 
\begin{align*}
    \EE[\|\eb_{t_{s+1}}^i\|^2] & \leq \frac{2q^2}{1-q^2} \sum_{\tau=1}^s \bigg(\frac{1+q^2}{2}\bigg)^{s-\tau} \EE [\|\Delta_{t_\tau}^i\|^2].
\end{align*}

Thus we similarly obtain the absolute bound for the error terms,
\begin{align}
    \|\eb_t^i\|^2 & \leq \frac{4q^2}{(1-q^2)^2} \eta_l^2 K^2 G^2, \notag\\
    \|\eb_t\|^2  & = \bigg\|\frac{1}{m} \sum_{i=1}^m \eb_t^i \bigg\|^2 \leq \frac{1}{m} \sum_{i=1}^m \|\eb_t^i\|^2 \leq \frac{4q^2}{(1-q^2)^2} \eta_l^2 K^2 G^2.
\end{align}

It concludes the proof.

\end{proof}

\begin{lemma}\label{lm:gmv-bound}
Under Assumptions \ref{as:bounded-g} and \ref{as:compressor}, for FedAMS, we have $\|\nabla f(\xb)\|\leq G$, $\|\Delta_t\|\leq \eta_l KG$, $\|\mb_t\| \leq \eta_lKG$ and $\|\vb_t\| \leq \eta_l^2 K^2G^2$. For FedCAMS, we have $\|\nabla f(\xb)\|\leq G$, $\|\hat{\Delta}_t\|^2 \leq \frac{4(1+q^2)^3}{(1-q^2)^2} \eta_l^2 K^2G^2$, $\|\mb_t'\| \leq \eta_l KG$ and $\|\vb_t\| \leq \frac{4(1+q^2)^3}{(1-q^2)^2} \eta_l^2 K^2G^2$, where $\mb_t' = \beta_1 \mb_{t-1}' + (1-\beta_1) \Delta_t$.
\end{lemma}
\begin{proof}
Since $f$ has $G$-bounded stochastic gradients, for any $\xb$ and $\xi$, we have $\|\nabla f(\xb, \xi)\| \leq G$, we have 
\begin{align*}
    \|\nabla f(\xb)\| = \|\EE_\xi \nabla f(\xb, \xi)\| \leq \EE_\xi \|\nabla f(\xb, \xi)\|\leq G.
\end{align*}
For FedAMS, the model difference $\Delta_t$, by definition, has the following formula, 
\begin{align*}
    \Delta_t = \xb_{t,K}^i - \xb_t = -\eta_l \sum_{k=1}^K \gb_{t,k}^i,
\end{align*}
therefore,
\begin{align*}
    \|\Delta_t\| \leq \eta_l K \|\gb_{t,k}^i \| \leq \eta_l KG.
\end{align*}
Thus the bound for momentum $\mb_t$ and variance $\vb_t$ has the formula of 
\begin{align*}
    \|\mb_t\| & = (1-\beta_1) \sum_{\tau=1}^t \beta_1^{t-\tau} \|\Delta_t\| \leq \eta_l K G, \\
    \|\vb_t\| & = (1-\beta_2) \sum_{\tau=1}^t \beta_2^{t-\tau} \|\Delta_t\|^2 \leq \eta_l^2 K^2 G^2.
\end{align*}
For the compressed version, FedCAMS, we have

\begin{align*}
    \|\hat{\Delta}_t^i\|^2 & \leq \|\cC(\Delta_t^i + \eb_t^i)\|^2 \\
    & \leq \|\cC(\Delta_t^i + \eb_t^i)-(\Delta_t^i + \eb_t^i)+(\Delta_t^i + \eb_t^i)\|^2 \\
    & \leq 2(q^2+1) \|\Delta_t^i + \eb_t^i\|^2 \\
    & \leq 4(q^2+1) [\|\Delta_t^i\|^2 + \|\eb_t^i\|^2] \\
    & \leq \frac{4(1+q^2)^3}{(1-q^2)^2} \eta_l^2 K^2 G^2,
\end{align*}
then we have 
\begin{align*}
    \|\hat{\Delta}_t\|^2 = \bigg\|\frac{1}{m} \sum_{i=1}^m \Delta_t^i \bigg\|^2 \leq \frac{1}{m} \sum_{i=1}^m \|\Delta_t^i \|^2 \leq \frac{4(1+q^2)^3}{(1-q^2)^2} \eta_l^2 K^2 G^2
\end{align*}

where the third inequality holds due to Assumption \ref{as:compressor},and the last inequality holds due to Lemma \ref{lm:error}. The virtual momentum sequence $\|\mb_t'\|$ has the same bound as $\mb_t$ of FedAMS. For the variance sequence of FedCAMS, we have
\begin{align*}
    \|\vb_t\| = (1-\beta_2) \sum_{\tau=1}^t \beta_2^{t-\tau} \|\hat{\Delta}_t\|^2 \leq \frac{4(1+q^2)^3}{(1-q^2)^2} \eta_l^2 K^2 G^2.
\end{align*}
This concludes the proof.
\end{proof}

\begin{lemma} \label{lm:EDelta}
The global model difference $\Delta_t = \sum_{i=1}^m \Delta_t^i$ in full participation cases satisfy 
\begin{align*}
     \EE [\|\Delta_t\|^2] & \leq \frac{K \eta_l^2}{m} \sigma_l^2+\frac{\eta_l^2}{m^2} \EE \bigg[\bigg\| \sum_{i=1}^m \sum_{k=0}^{K-1} \nabla F_i(\xb_{t,k}^i) \bigg\|^2\bigg].
\end{align*}
\end{lemma}
\begin{proof}
For $\EE[\|\Delta_t\|^2]$ in full participation case, we have
\begin{align}\label{eq:EDelta-1}
    \EE [\|\Delta_t\|^2] & = \EE \bigg[\bigg\|\frac{1}{m} \sum_{i=1}^m \sum_{k=0}^{K-1} \eta_l\gb_{t,k}^i \bigg\|^2\bigg] \notag\\
    & = \frac{\eta_l^2}{m^2} \EE \bigg[\bigg\| \sum_{i=1}^m \sum_{k=0}^{K-1}\gb_{t,k}^i \bigg\|^2\bigg] \notag\\
    & = \frac{\eta_l^2}{m^2} \EE \bigg[\bigg\| \sum_{i=1}^m \sum_{k=0}^{K-1}(\gb_{t,k}^i-\nabla F_i(\xb_{t,k}^i)) \bigg\|^2\bigg]+\frac{\eta_l^2}{m^2} \EE \bigg[\bigg\| \sum_{i=1}^m \sum_{k=0}^{K-1} \nabla F_i(\xb_{t,k}^i) \bigg\|^2\bigg] \notag\\
    & \leq \frac{K \eta_l^2}{m} \sigma_l^2+\frac{\eta_l^2}{m^2} \EE \bigg[\bigg\| \sum_{i=1}^m \sum_{k=0}^{K-1} \nabla F_i(\xb_{t,k}^i) \bigg\|^2\bigg],
\end{align}
where the inequality holds by Assumption \ref{as:bounded-g}. This concludes the proof.
\end{proof}

\begin{lemma} \label{lm:EDelta'}
The global model difference $\Delta_t = \sum_{i \in\cS_t} \Delta_t^i$ in partial participation cases satisfy 
\begin{align*}
    \EE[\|\Delta_t\|^2] & = \frac{K \eta_l^2}{n} \sigma_l^2 + \frac{\eta_l^2(m-n)}{mn(m-1)}[15mK^3 L^3 \eta_l^2(\sigma_l^2 + 6K\sigma_g^2) + 90mK^4L^2\eta_l^2+3mK^2 \|\nabla f(\xb_t)\|^2 \notag\\
    &\quad + 3mK^2\sigma_g^2]+ \frac{\eta_l^2 (n-1)}{mn(m-1)} \EE\bigg[ \bigg\|\sum_{i=1}^m \sum_{k=0}^{K-1}\nabla F_i (\xb_{t,k}^i)\bigg\|^2\bigg]. 
\end{align*} 
\end{lemma}
\begin{proof}
We have
\begin{align}\label{eq:c8}
    \EE[\|\Delta_t\|^2] & = \EE\bigg[\bigg\|\frac{1}{n} \sum_{i \in \cS_t} \Delta_t^i \bigg\|^2\bigg] \notag\\
    & = \frac{1}{n^2} \EE\bigg[\bigg\|\sum_{i=1}^m \II\{i\in\cS_t\} \Delta_t^i \bigg\|^2\bigg] \notag\\
    & = \frac{\eta_l^2}{n^2} \EE\bigg[\bigg\| \sum_{i=1}^m \II\{i\in\cS_t\} \sum_{k=0}^{K-1}[\gb_{t,k}^i - \nabla F_i (\xb_{t,k}^i)] \bigg\|^2 + \bigg\| \sum_{i=1}^m \II\{i\in\cS_t\} \sum_{k=0}^{K-1} \nabla F_i (\xb_{t,k}^i) \bigg\|^2\bigg] \notag\\
    & = \frac{\eta_l^2}{n^2} \EE\bigg[\bigg\| \sum_{i=1}^m \PP\{i\in\cS_t\} \sum_{k=0}^{K-1}[\gb_{t,k}^i - \nabla F_i (\xb_{t,k}^i)] \bigg\|^2 + \bigg\| \sum_{i=1}^m \PP\{i\in\cS_t\} \sum_{k=0}^{K-1} \nabla F_i (\xb_{t,k}^i) \bigg\|^2\bigg] \notag\\
    & = \frac{\eta_l^2}{mn} \EE\bigg[\bigg\| \sum_{i=1}^m \sum_{k=0}^{K-1}[\gb_{t,k}^i - \nabla F_i (\xb_{t,k}^i)] \bigg\|^2\bigg] + \frac{\eta_l^2}{n^2}\EE \bigg[\bigg\| \sum_{i=1}^m \PP\{i\in\cS_t\} \sum_{k=0}^{K-1} \nabla F_i (\xb_{t,k}^i) \bigg\|^2\bigg] \notag\\    
    & \leq \frac{K \eta_l^2}{n} \sigma_l^2 + \frac{\eta_l^2}{n^2}\EE \bigg[\bigg\| \sum_{i=1}^m \PP\{i\in\cS_t\} \sum_{k=0}^{K-1} \nabla F_i (\xb_{t,k}^i) \bigg\|^2\bigg],
\end{align}
where the fifth equation holds due to $\PP\{i\in\cS_t\} = \frac{n}{m}$. Note that we have 
\begin{align}
    \bigg\| \sum_{i=1}^m \sum_{k=0}^{K-1} \nabla F_i (\xb_{t,k}^i) \bigg\|^2 & = \sum_{i=1}^m \bigg\|\sum_{k=0}^{K-1} \nabla F_i (\xb_{t,k}^i)\bigg\|^2 + \sum_{i \neq j}\bigg\langle \sum_{k=0}^{K-1} \nabla F_i (\xb_{t,k}^i), \sum_{k=0}^{K-1} \nabla F_j (\xb_{t,k}^j)\bigg\rangle \notag\\
    & = \sum_{i=1}^m m \bigg\|\sum_{k=0}^{K-1}\nabla F_i (\xb_{t,k}^i)\bigg\|^2 -\frac{1}{2} \sum_{i \neq j} \bigg\|\sum_{k=0}^{K-1} \nabla F_i (\xb_{t,k}^i) - \sum_{k=0}^{K-1} \nabla F_j (\xb_{t,k}^j)\bigg\|^2,
\end{align}
where the second equation holds due to $\|\sum_{i=1}^m \xb_i\|^2 = \sum_{i=1}^m m \|\xb_i\|^2 - \frac{1}{2} \sum_{i \neq j} \|\xb_i- \xb_j\|^2$. By the sampling strategy (without replacement), we have $\PP\{i\in\cS_t\} = \frac{n}{m}$ and $\PP\{i,j\in\cS_t\} = \frac{n(n-1)}{m(m-1)}$, thus we have 
\begin{align}\label{eq:i-in-St-norm-noncomp}
    & \bigg\| \sum_{i=1}^m \sum_{k=0}^{K-1} \PP\{i\in\cS_t\} \nabla F_i (\xb_{t,k}^i) \bigg\|^2 \notag\\
    & = \sum_{i=1}^m \PP\{i\in\cS_t\}\bigg \|\sum_{k=0}^{K-1} \nabla F_i (\xb_{t,k}^i)\bigg\|^2 + \sum_{i \neq j} \PP\{i,j\in\cS_t\} \bigg\langle \sum_{k=0}^{K-1} \nabla F_i (\xb_{t,k}^i), \sum_{k=0}^{K-1} \nabla F_j (\xb_{t,k}^j)\bigg\rangle \notag\\
    & = \frac{n}{m} \sum_{i=1}^m \bigg\|\sum_{k=0}^{K-1}\nabla F_i (\xb_{t,k}^i)\bigg\|^2 + \frac{n(n-1)}{m(m-1)} \sum_{i \neq j}\bigg\langle \sum_{k=0}^{K-1} \nabla F_i (\xb_{t,k}^i), \sum_{k=0}^{K-1} \nabla F_j (\xb_{t,k}^j)\bigg\rangle \notag\\
    & = \frac{n^2}{m} \sum_{i=1}^m \bigg\|\sum_{k=0}^{K-1} \nabla F_i (\xb_{t,k}^i)\bigg\|^2 - \frac{n(n-1)}{2m(m-1)} \sum_{i \neq j} \bigg\|\sum_{k=0}^{K-1}\nabla F_i (\xb_{t,k}^i) - \sum_{k=0}^{K-1} \nabla F_j (\xb_{t,k}^j)\bigg\|^2 \notag\\
    & = \frac{n(m-n)}{m(m-1)} \sum_{i=1}^m \bigg\|\sum_{k=0}^{K-1} \nabla F_i (\xb_{t,k}^i)\bigg\|^2 + \frac{n(n-1)}{m(m-1)} \bigg\|\sum_{i=1}^m \sum_{k=0}^{K-1}\nabla F_i (\xb_{t,k}^i)\bigg\|^2, \notag\\
\end{align}
where the third equation holds due to $\langle\xb, \yb\rangle = \frac{1}{2}[\|\xb\|^2 + \|\yb\|^2 - \|\xb-\yb\|^2]$ and the last equation holds due to $\frac{1}{2}\sum_{i\neq j}\|\xb_i -\xb_j \|^2 = \sum_{i=1}^m m \|\xb_i\|^2 - \|\sum_{i=1}^m \xb_i\|^2$. Therefore, for the last term in \eqref{eq:c8}, we have
\begin{align}\label{eq:EDelta-2'}
    \EE[\|\Delta_t\|^2] & = \frac{K \eta_l^2}{n} \sigma_l^2 + \frac{\eta_l^2(m-n)}{mn(m-1)} \sum_{i=1}^m \EE\bigg[\bigg\|\sum_{k=0}^{K-1} \nabla F_i (\xb_{t,k}^i)\bigg\|^2 \bigg] + \frac{\eta_l^2 (n-1)}{mn(m-1)} \EE\bigg[ \bigg\|\sum_{i=1}^m \sum_{k=0}^{K-1}\nabla F_i (\xb_{t,k}^i)\bigg\|^2\bigg]. 
\end{align} 
The second term in \eqref{eq:EDelta-2'} is bounded partially following \citet{reddi2020adaptive},
\begin{align} \label{eq:EDelta-3'}
    \sum_{i=1}^m \bigg\|\sum_{k=0}^{K-1} \nabla F_i(\xb_{t,k}^i)\bigg\|^2 & = \sum_{i=1}^m \EE \bigg\|\sum_{k=0}^{K-1} [\nabla F_i(\xb_{t,k}^i)- \nabla F_i(\xb_t)+ \nabla F_i(\xb_t) -\nabla f(\xb_t) + \nabla f(\xb_t)]\bigg\|^2 \notag\\
    & \leq 3 \sum_{i=1}^m \EE \bigg\|\sum_{k=0}^{K-1} [\nabla F_i(\xb_{t,k}^i)- \nabla F_i(\xb_t)] \bigg\|^2 + 3 m K^2\sigma_g^2 + 3mK^2\|\nabla f(\xb_t)\|^2 \notag\\
    & \leq 3K L^2 \sum_{i=1}^m\sum_{k=0}^{K-1} \EE[\|\xb_{t,k}^i-\xb_t\|^2] + 3 m K^2\sigma_g^2 + 3mK^2\|\nabla f(\xb_t)\|^2 \notag\\
    & \leq 15 m K^3 L^3 \eta_l^2(\sigma_l^2+ 6K \sigma_g^2)+ (90mK^4L^2\eta_l^2+3mK^2)\|\nabla f(\xb_t)\|^2 +3mK^2 \sigma_g^2,
\end{align}
where the last inequality holds by applying Lemma \ref{lm:xikt-xt} (also follows from \citet{reddi2020adaptive}).
Substituting \eqref{eq:EDelta-3'} into \eqref{eq:EDelta-2'}, this concludes the proof.
\end{proof}

\begin{lemma}\label{lm:mE-bound}
Under Assumptions \ref{as:smooth}-\ref{as:bounded-v} and Assumption \ref{as:compressor}, for the momentum sequence $\mb_t = (1-\beta_1) \sum_{\tau=1}^t \beta_1^{t-\tau}\Delta_\tau$ and accumulated error sequence $\bGamma_t = (1-\beta_1) \sum_{\tau=1}^t \beta_1^{t-\tau} \eb_\tau$ in full participation settings, we have 
\begin{align*}
    \sum_{t=1}^T \EE [\|\mb_t\|^2] \leq \frac{T K \eta_l^2}{m} \sigma_l^2 + \frac{\eta_l^2}{m^2} \sum_{t=1}^T \EE \bigg[\bigg\|\sum_{i=1}^m \sum_{k=0}^{K-1}\nabla F_i(\xb_{t,k}^i ) \bigg\|^2\bigg],
\end{align*}
and
\begin{align*}
    \sum_{t=1}^T \EE [\|\bGamma_t\|^2] 
    & \leq \frac{4Tq^2}{(1-q^2)^2} \frac{K \eta_l^2}{m} \sigma_l^2 + \frac{\eta_l^2}{m^2} \frac{4q^2}{(1-q^2)^2} \sum_{t=1}^T \EE \bigg[\bigg\|\sum_{i=1}^m \sum_{k=0}^{K-1}\nabla F_i(\xb_{t,k}^i ) \bigg\|^2\bigg].
\end{align*}
\end{lemma}
\begin{proof} 
By the updating rule, we have
\begin{align}
    \EE[\|\mb_{t}\|^{2}] 
    & = \EE\bigg[\|(1-\beta_1) \sum_{\tau=1}^{t} \beta_{1}^{t-\tau} \Delta_{\tau}\|^{2}\bigg] \notag\\
    & \leq (1-\beta_1)^2 \sum_{i=1}^{d} \EE\bigg[\bigg(\sum_{\tau=1}^{t} \beta_{1}^{t-\tau} \Delta_{\tau, i}\bigg)^{2}\bigg] \notag\\
    & \leq (1-\beta_1)^{2} \sum_{i=1}^{d} \EE\bigg[\bigg(\sum_{\tau=1}^{t} \beta_{1}^{t-\tau}\bigg)\bigg(\sum_{\tau=1}^{t} \beta_{1}^{t-\tau} \Delta_{\tau, i}^{2}\bigg)\bigg] \notag \\
    & \leq (1-\beta_1) \sum_{\tau=1}^{t} \beta_{1}^{t-\tau} \EE[\|\Delta_{\tau}\|^{2}] \notag\\
    & \leq \frac{K \eta_l^2}{m} \sigma_l^2 + \frac{\eta_l^2}{m^2} (1-\beta_1)\sum_{\tau=1}^t \beta_1^{t-\tau} \EE \bigg[\bigg\|\sum_{i=1}^m \sum_{k=0}^{K-1}\nabla F_i(\xb_{t,k}^i ) \bigg\|^2\bigg],
\end{align}
where the second inequality holds by applying Cauchy-Schwarz inequality, and the third inequality holds by summation of series. The last inequality holds by Lemma \ref{lm:EDelta}.
Hence summing over $t=1,\cdots, T$, we have
\begin{align}
    \sum_{t=1}^T \EE [\|\mb_t\|^2] \leq \frac{T K \eta_l^2}{m} \sigma_l^2 + \frac{\eta_l^2}{m^2} \sum_{t=1}^T \EE \bigg[\bigg\|\sum_{i=1}^m \sum_{k=0}^{K-1}\nabla F_i(\xb_{t,k}^i ) \bigg\|^2\bigg].
\end{align}

For the compression error $\eb_t$, by Assumption \ref{as:compressor} and \ref{as:compressor-adapt}, we have 
\begin{align}
    \|\eb_{t+1}\| & = \bigg\|\frac{1}{m} \sum_{i=1}^m \eb_{t+1}^i \bigg\| \notag\\
    & = \bigg\|\frac{1}{m} \sum_{i=1}^m [\Delta_t^i + \eb_t^i] - \frac{1}{m} \sum_{i=1}^m \cC(\Delta_t^i + \eb_t^i)\bigg\| \notag\\
    & \leq \bigg\|\frac{1}{m} \sum_{i=1}^m [\Delta_t^i + \eb_t^i] - \cC\bigg(\frac{1}{m} \sum_{i=1}^m [\Delta_t^i + \eb_t^i] \bigg) \bigg\| + \bigg\|\cC\bigg(\frac{1}{m} \sum_{i=1}^m [\Delta_t^i + \eb_t^i] \bigg) - \frac{1}{m} \sum_{i=1}^m \cC(\Delta_t^i + \eb_t^i)\bigg\| \notag\\
    & \leq q \bigg\|\frac{1}{m} \sum_{i=1}^m [\Delta_t^i + \eb_t^i] \bigg\| + \gamma \bigg\|\frac{1}{m} \sum_{i=1}^m \Delta_t^i \bigg\| \notag\\
    & \leq q\|\Delta_t\| + q\|\eb_t\| + \gamma \|\Delta_t\| \notag\\ 
    & = q \|\eb_t\| + (q+\gamma) \|\Delta_t\|,
\end{align}
where the first equation holds by the definition for error $\eb_{t+1}$, and the second one holds by the update rule for $\eb_{t+1}^i$. The first inequality holds by $\|\ab + \bb\| \leq \|\ab\| + \|\bb\|$, and the second one holds by Assumption \ref{as:compressor} and \ref{as:compressor-adapt}. Thus by Young's inequality, we have 
\begin{align}
    \|\eb_{t+1}\|^2 & \leq \big( q \|\eb_t\| + (q+\gamma) \|\Delta_t\|\big)^2 \notag\\
    & \leq  q^2 (1+\rho) \|\eb_t\|^2 + (q+\gamma)^2 (1+\rho^{-1}) \|\Delta_t\|^2 \notag\\
    & = \frac{1+q^2}{2} \|\eb_t\|^2 + \frac{2(q+\gamma)^2 }{1-q^2}  \|\Delta_t\|^2,
\end{align}
where the equation holds by letting $\rho = \frac{1-q^2}{2q^2}$, and $1+\rho^{-1} = \frac{1+q^2}{1-q^2} \leq \frac{2}{1-q^2}$, then by the similar recursive approach in the proof of Lemma \ref{lm:error}, we have  
\begin{align}
    \EE[\|\eb_{t+1}\|^2] & \leq \frac{2(q+\gamma)^2}{1-q^2} \sum_{\tau=1}^t \bigg(\frac{1+q^2}{2}\bigg)^{t-\tau} \EE [\|\Delta_\tau\|^2] \notag\\
    & \leq \frac{4(q+\gamma)^2}{(1-q^2)^2} \frac{K \eta_l^2}{m} \sigma_l^2 + \frac{\eta_l^2 }{m^2} \frac{2(q+\gamma)^2}{1-q^2}\sum_{\tau=1}^t \bigg(\frac{1+q^2}{2}\bigg)^{t-\tau} \EE \bigg[\bigg\|\sum_{i=1}^m \sum_{k=0}^{K-1}\nabla F_i(\xb_{\tau,k}^i ) \bigg\|^2\bigg].
\end{align}
For the sequence $\bGamma_t$, similar as the previous analysis, we have
\begin{align}
    \EE [\|\bGamma_t\|^2] & = \EE \bigg[\bigg\|(1-\beta_1)\sum_{\tau=1}^t \beta_1^{t-\tau} \eb_\tau\bigg\|^2\bigg] \notag\\
    & \leq (1-\beta_1) \sum_{\tau=1}^t \beta_1^{t-\tau} \EE[ \|\eb_\tau\|^2] \notag\\
    & \leq \frac{4(q+\gamma)^2}{(1-q^2)^2} \frac{K \eta_l^2}{m} \sigma_l^2 + \frac{\eta_l^2}{m^2} \frac{2 (q+\gamma)^2(1-\beta_1)}{1-q^2}\sum_{\tau=1}^t \beta_1^{t-\tau} \sum_{j=1}^\tau \bigg(\frac{1+q^2}{2}\bigg)^{\tau-j} \EE \bigg[\bigg\|\sum_{i=1}^m \sum_{k=0}^{K-1}\nabla F_i(\xb_{j,k}^i ) \bigg\|^2\bigg].
\end{align}
Summing over $t=1,\cdots, T$, we have
\begin{align}
    \sum_{t=1}^T \EE [\|\bGamma_t\|^2] & \leq \frac{4T(q+\gamma)^2}{(1-q^2)^2} \frac{K \eta_l^2}{m} \sigma_l^2 + \frac{\eta_l^2}{m^2} \frac{2(q+\gamma)^2}{1-q^2}\sum_{t=1}^T \sum_{\tau=1}^t \bigg(\frac{1+q^2}{2}\bigg)^{t-\tau} \EE \bigg[\bigg\|\sum_{i=1}^m \sum_{k=0}^{K-1}\nabla F_i(\xb_{\tau,k}^i ) \bigg\|^2\bigg] \notag\\
    & \leq \frac{4T (q+\gamma)^2}{(1-q^2)^2} \frac{K \eta_l^2}{m} \sigma_l^2 + \frac{\eta_l^2}{m^2} \frac{4 (q+\gamma)^2}{(1-q^2)^2} \sum_{t=1}^T \EE \bigg[\bigg\|\sum_{i=1}^m \sum_{k=0}^{K-1}\nabla F_i(\xb_{t,k}^i ) \bigg\|^2\bigg].
\end{align}

\end{proof}

\begin{lemma}\label{lm:mE-bound-partial}
Under Assumptions \ref{as:smooth}-\ref{as:bounded-v} and Assumption \ref{as:compressor}, for the momentum sequence $\mb_t = (1-\beta_1) \sum_{\tau=1}^t \beta_1^{t-\tau}\Delta_\tau $ in partial participation settings, we have 
\begin{align*}
   \sum_{t=1}^{T}\EE [\|\mb_t\|^2] \leq \frac{K T \eta_l^2}{n} \sigma_l^2 + \frac{\eta_l^2}{n^2}  \sum_{t=1}^{T} \EE \bigg[\bigg\|\sum_{i \in \cS_t} \sum_{k=0}^{K-1} \nabla F_i(\xb_{t,k}^i) \bigg\|^2\bigg].
\end{align*}
\end{lemma}
\begin{proof}
The proof outline is the same as the proof of Lemma \ref{lm:mE-bound}, the main difference is $E[\|\Delta_t\|^2]$ has changed, so we need to apply Lemma \ref{lm:EDelta'} instead of Lemma \ref{lm:EDelta} during the proof. 
\end{proof}

\begin{lemma} \label{lm:xikt-xt}
(This lemma directly follows from Lemma 3 in FedAdam \citep{reddi2020adaptive}. For local learning rate which satisfying $\eta_l \leq \frac{1}{8KL}$, the local model difference after $k$ ($\forall k \in \{0,1,...,K-1\}$) steps local updates satisfies
\begin{align}
    \frac{1}{m}\sum_{i=1}^{m}\EE [\|\xb_{t,k}^i - \xb_t\|^2] \leq 5K\eta_l^2(\sigma_l^2+6K\sigma_g^2) + 30 K^2 \eta_l^2 \EE[\|\nabla f(\xb_t)\|^2].
\end{align}
\end{lemma}
\begin{proof}
The proof of Lemma \ref{lm:xikt-xt} is exactly same as the proof of Lemma 3 in \citet{reddi2020adaptive}.
\end{proof}

\section{Additional Discussions}

\noindent\textbf{The additional server-to-worker communication: } Our current analysis only focus on one-way compression from worker to server while the server-to-worker broadcasting is still uncompressed since of cost of broadcasting is in general cheaper than worker to server uploading. 
Note that it is also straightforward to compress $\xb$ for server-to-worker communication in the full participation scheme (with guarantees). So we can indeed achieve high communication efficiency even for two-way compression. Table \ref{tab:communication-x} shows the communication bits comparison for scaled sign and top-$k$ compressors, where $T$ is the total iteration of training and $d$ denotes the dimension of $\xb$. Specifically, Table  \ref{tab:communication-x-resnet} shows the communication bits corresponding to the experiments showing by Figure \ref{fig:cifar10_ef_resnet}. However, for the partial participating setting, it will encounter a synchronization issue, which is highly non-trivial to solve. Thus we leave the two-way compression strategy for future work. 
\begin{table}[ht]
    \centering
    \begin{tabular}{c|ccc}
        \toprule
        Method &  Uncompressed & One-way Compression & Two-way Compression\\
        \midrule
        Scaled sign & $32d\times 2T$ & $(32+d)\times T + 32d\times T$ & $(32+d)\times 2T$\\
        Top-$k$ & $32d\times 2T$ & $\approx32(2k+d)\times T$ & $\approx 32 \times 2k \times 2T$\\
        \bottomrule
    \end{tabular}
    \caption{Communication bits comparisons for scaled sign and top-$k$ compressors.}
    \label{tab:communication-x}
\end{table}

\begin{table}[ht]
    \centering
    \begin{tabular}{c|ccc}
        \toprule
        Method &  Uncompressed & One-way Compression & Two-way Compression\\
        \midrule
        Scaled sign & $3.58\times 10^{11}$ & $1.84\times 10^{11}$ & $1.12\times 10^{10}$\\
        Top-$k$ with $r= 1/64$ & $3.58\times 10^{11}$ & $1.84\times 10^{11}$ & $1.12\times 10^{10}$\\
        Top-$k$ with $r= 1/128$ & $3.58\times 10^{11}$ & $1.82\times 10^{11}$ & $5.59\times 10^{9}$\\
        Top-$k$ with $r= 1/256$ & $3.58\times 10^{11}$ & $1.80\times 10^{11}$ & $2.79\times 10^{9}$\\
        \bottomrule
    \end{tabular}
    \caption{Approximate communication bits comparisons for scaled sign and top-$k$ with $r= 1/64$, $r= 1/128$ and $r= 1/256$ compressors when training CIFAR-10 on ResNet-18 model for $500$ rounds.}
    \label{tab:communication-x-resnet}
\end{table}

\section{Additional Experimental Results}

\subsection{Hyperparameter Settings}
We conduct detailed hyperparameter searches to find the best hyperparameters for each baseline methods including ours. In details, we grid search over the local learning rate $\eta_{l} \in \{0.0001, 0.001, 0.01, 0.1,$ $1.0\}$, the global learning rate $\eta \in \{0.001, 0.01, 0.1, 1.0\}$ for all methods. For adaptive federated optimization methods, we set $\beta_1 = 0.9$, $\beta_2 = 0.99$. For FedAdam, FedYogi, and FedAMSGrad, we search the best $\epsilon$ from $\{10^{-8}, 10^{-4}, 10^{-3}, 10^{-2}, 10^{-1}, 10^{0}\}$. For FedAMS and FedCAMS, we search the max stabilization $\epsilon$ from $\{10^{-8}, 10^{-4}, 10^{-3}, 10^{-2}, 10^{-1}, 10^{0}\}$.

Specifically, for our ResNet-18 experiments, we set the local learning rate $\eta_{l} = 0.01$ and the global learning rate $\eta = 1.0$ for FedAvg, set $\eta_{l} = 0.01$, $\eta = 0.1$ and $\epsilon=0.1$ for FedAdam and FedAMSGrad, set $\eta_{l} = 0.01$, $\eta = 1.0$ and $\epsilon=0.1$ for FedYogi, set $\eta_{l} = 0.01$, $\eta = 1.0$ and max stabilization $\epsilon=0.001$ for FedAMS and FedCAMS. For our ConvMixer-256-8 experiments, we set the local learning rate $\eta_{l} = 0.01$ and the global learning rate $\eta = 1.0$ for FedAvg, set $\eta_{l} = 0.01$, $\eta = 1.0$ and $\epsilon=0.1$ for FedAdam, FedYogi and FedAMSGrad, set $\eta_{l} = 0.01$, $\eta = 1.0$ and max stabilization $\epsilon=0.001$ for FedAMS and FedCAMS.

\subsection{Additional Experiments}

Figure \ref{fig:cifar10_ef_kn_resnet} shows the effect of parameter $n$ on the convergence rate of FedCAMS with choosing groups of parameters: $n\in \{5, 10, 20\}$. For both ResNet-18 and ConvMixer-256-8 models, it is shown that a larger number of participating clients $n$ achieves a faster convergence rate, this backs up our theory.
\begin{figure}[ht!]
\centering
\subfigure[ResNet-18]{\includegraphics[width=0.23\textwidth]{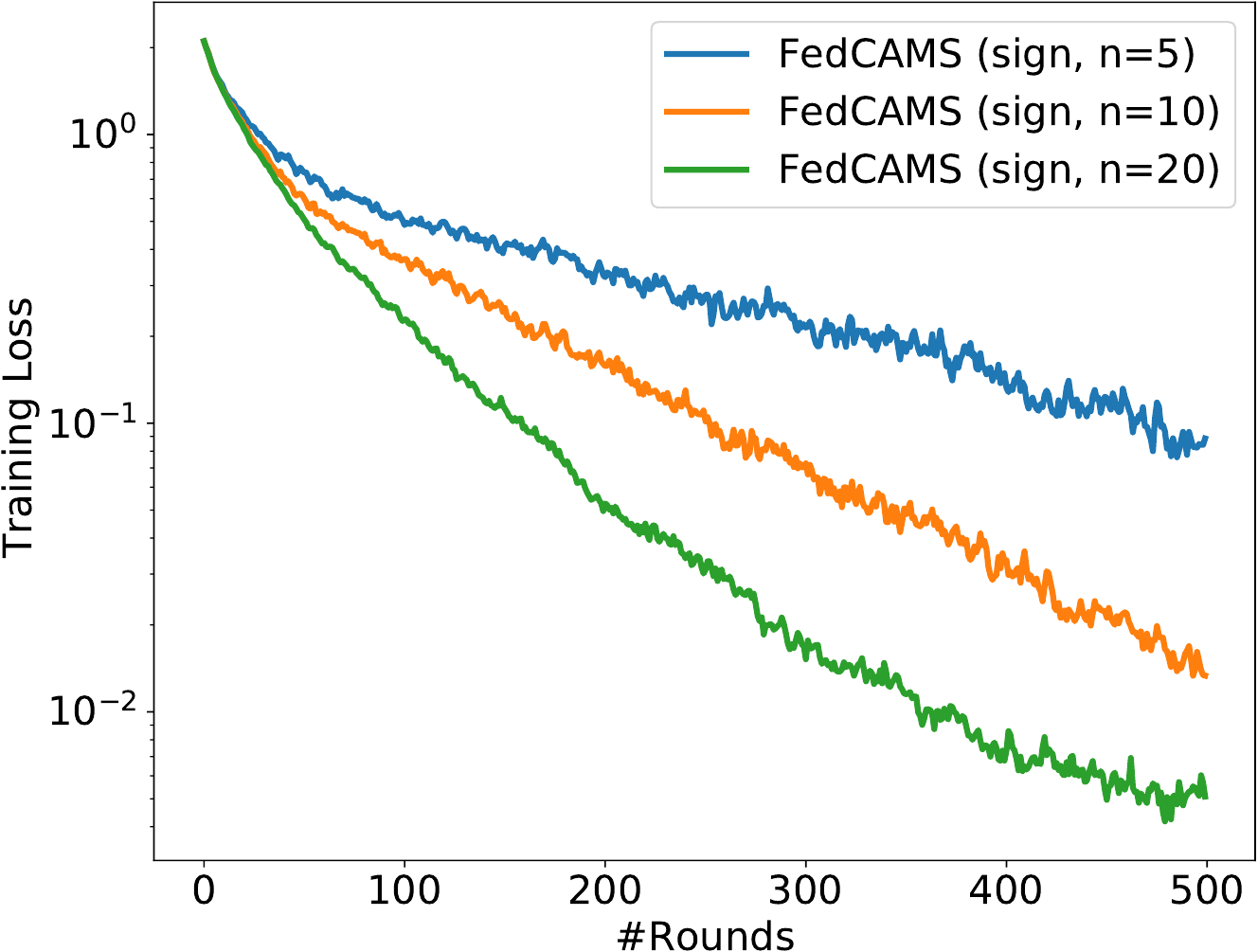}}
\subfigure[ConvMixer-256-8]{\includegraphics[width=0.23\textwidth]{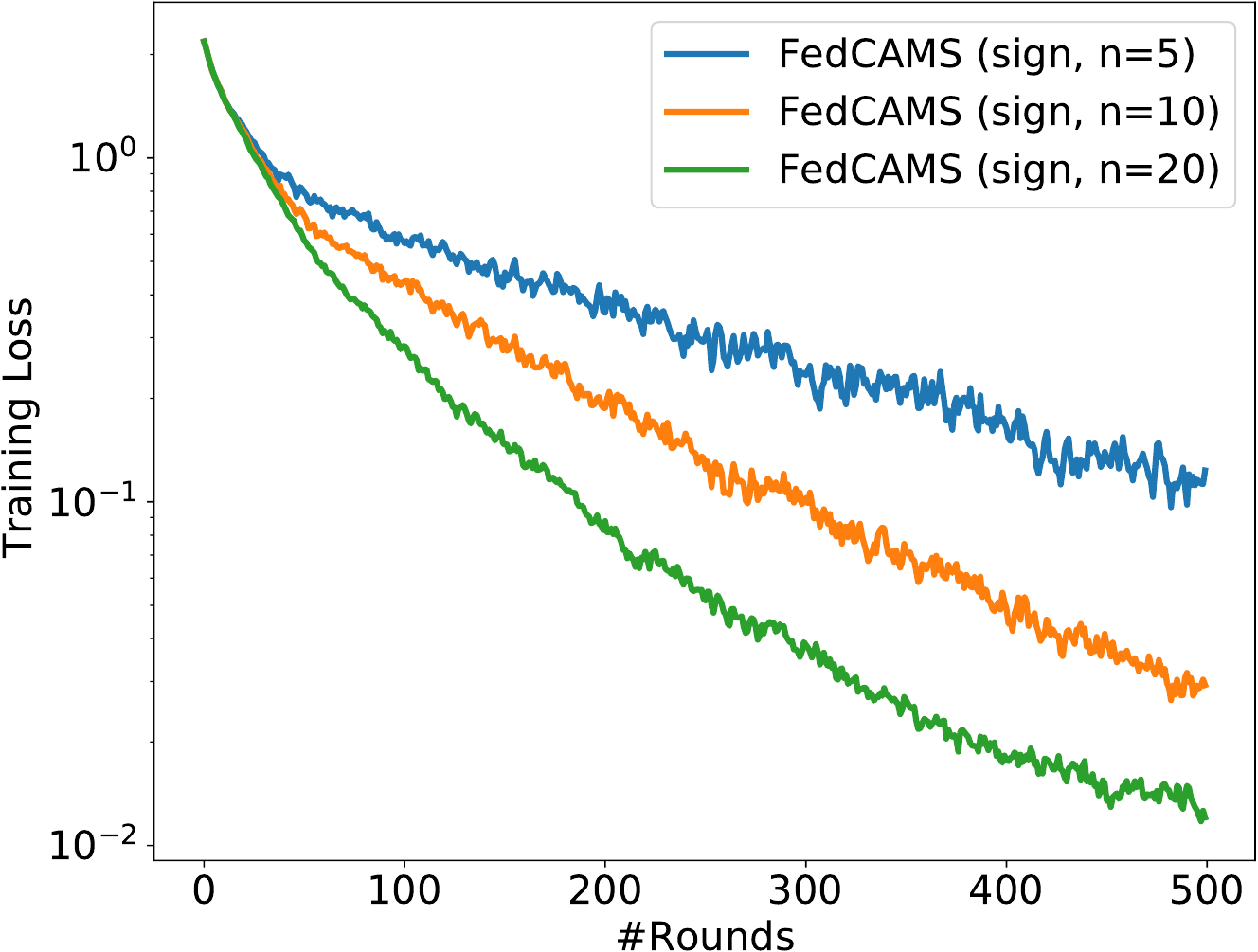}}
\caption{The learning curves for FedCAMS with different participating number of clients $n$ in training CIFAR-10 data on ResNet-18 and ConvMixer-256-8 models.}
\label{fig:cifar10_ef_kn_resnet}
\end{figure}

Figure \ref{fig:cifar100} shows the convergence result of FedAMS and other federated learning baselines on training CIFAR-100 dataset with the ResNet-18 model and the ConvMixer-256-8 model. We compare the training loss and test accuracy against the global rounds for each model. For the ResNet-18 model, FedAMS and FedYogi achieve significantly better performance comparing with other three baselines. In particular, FedYogi has a fast convergence rate at the beginning status, while FedAMS performs the best in terms of the final training loss and test accuracy. FedAvg achieves a slightly better training loss to FedAdam and FedAMSGrad but much higher test accuracy which is close to FedYogi and FedAMS. 

\begin{figure}[ht!]
\centering
\subfigure[ResNet-18]{\includegraphics[width=0.23\textwidth]{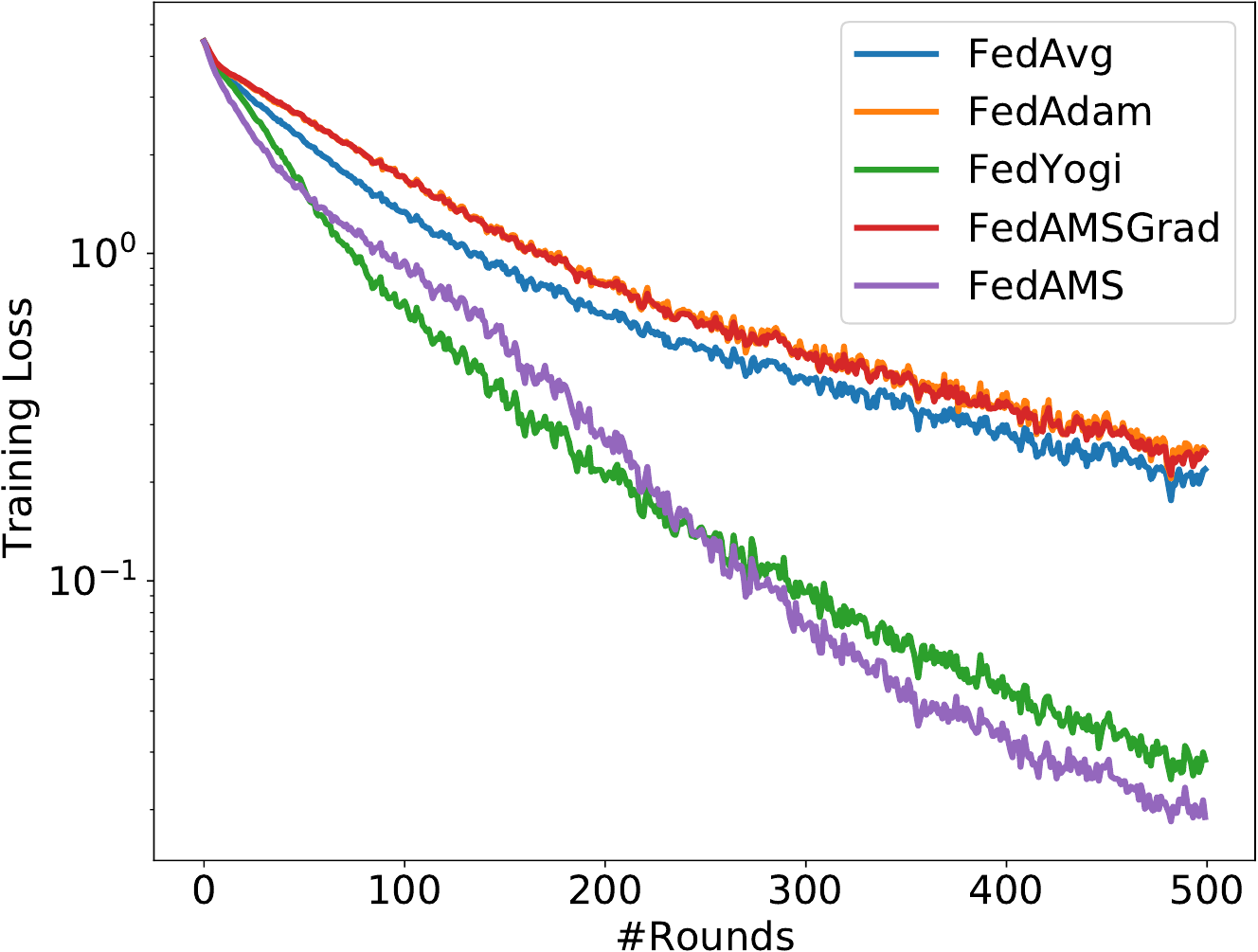}}
\subfigure[ResNet-18]{\includegraphics[width=0.23\textwidth]{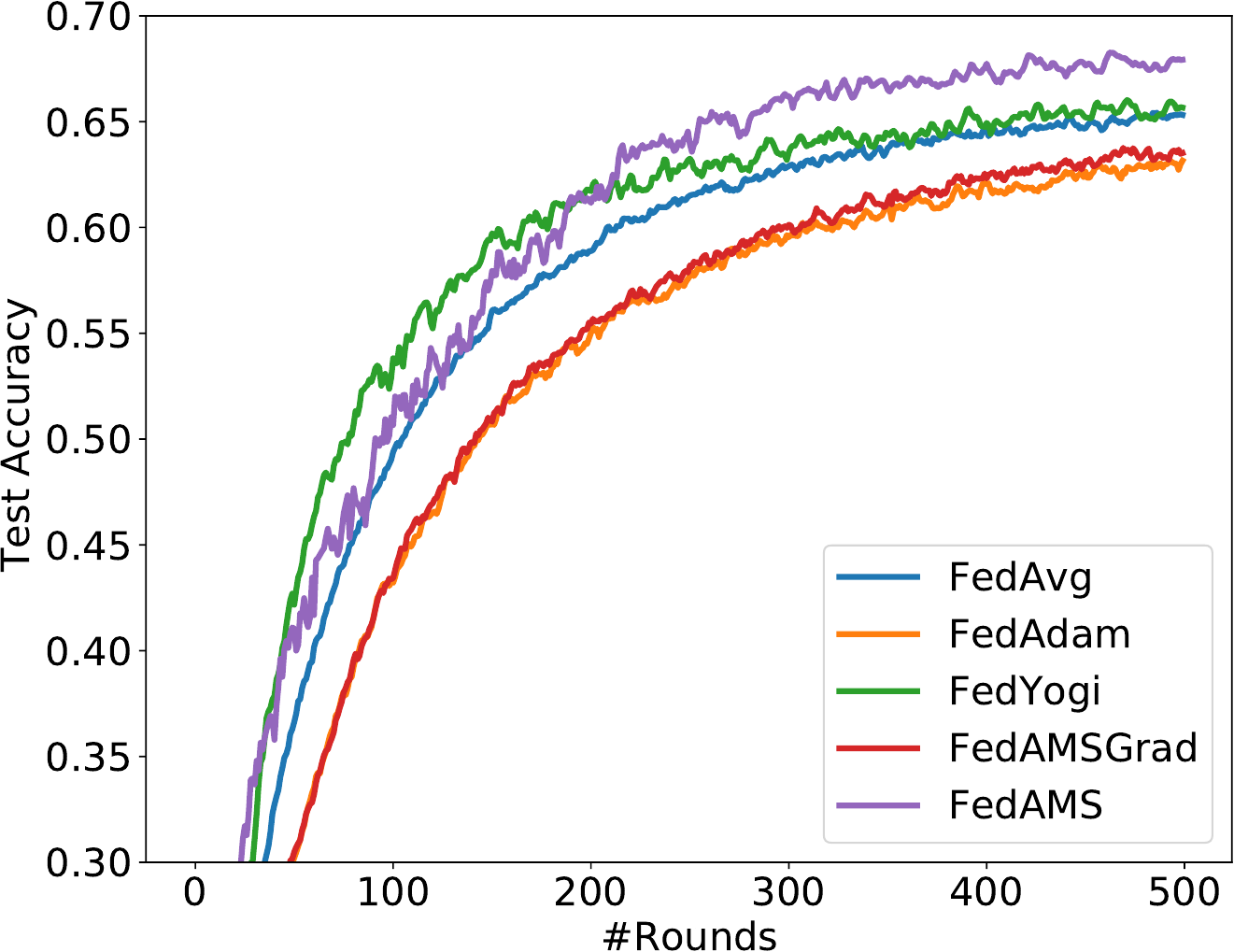}}
\subfigure[ConvMixer-256-8]{\includegraphics[width=0.23\textwidth]{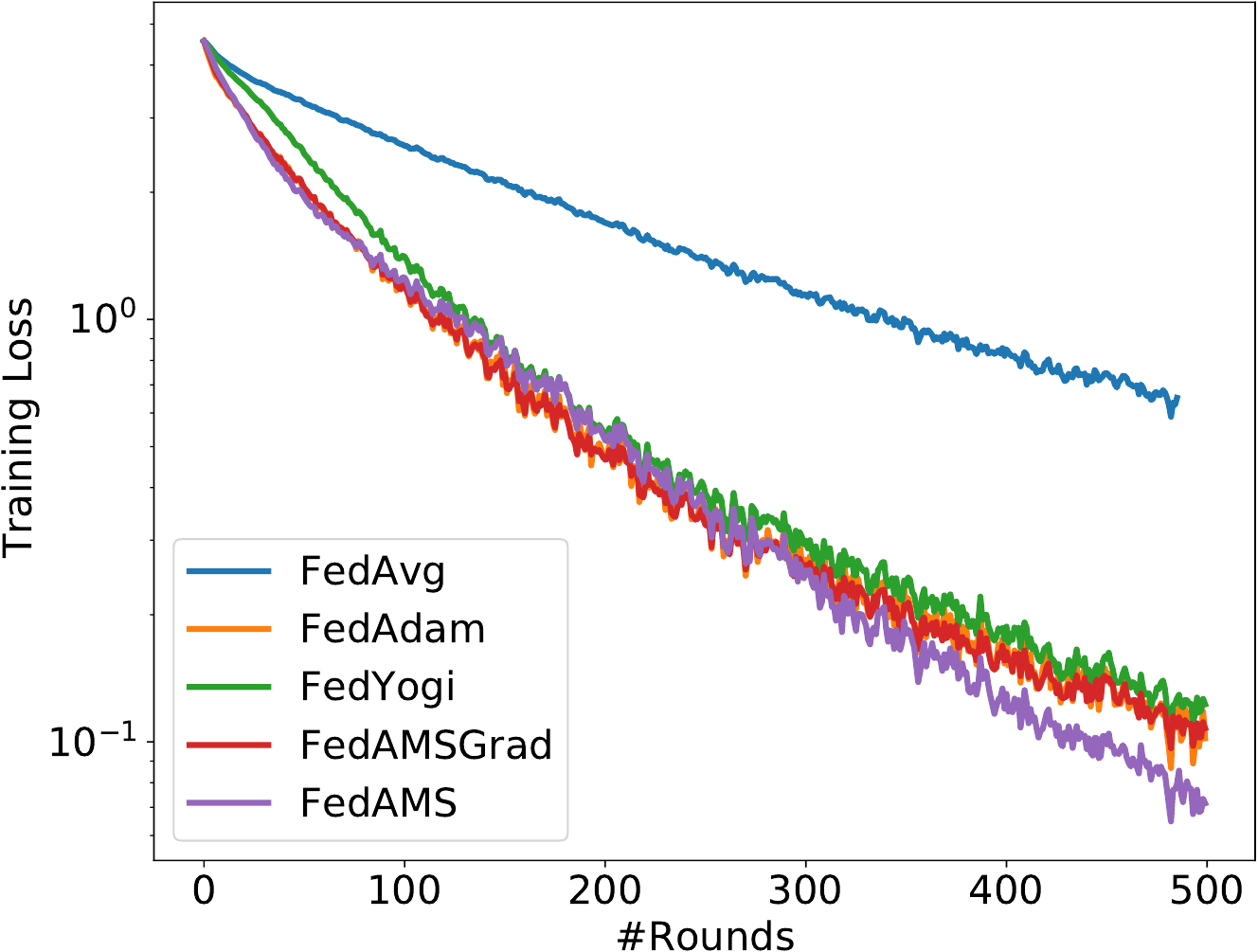}}
\subfigure[ConvMixer-256-8]{\includegraphics[width=0.23\textwidth]{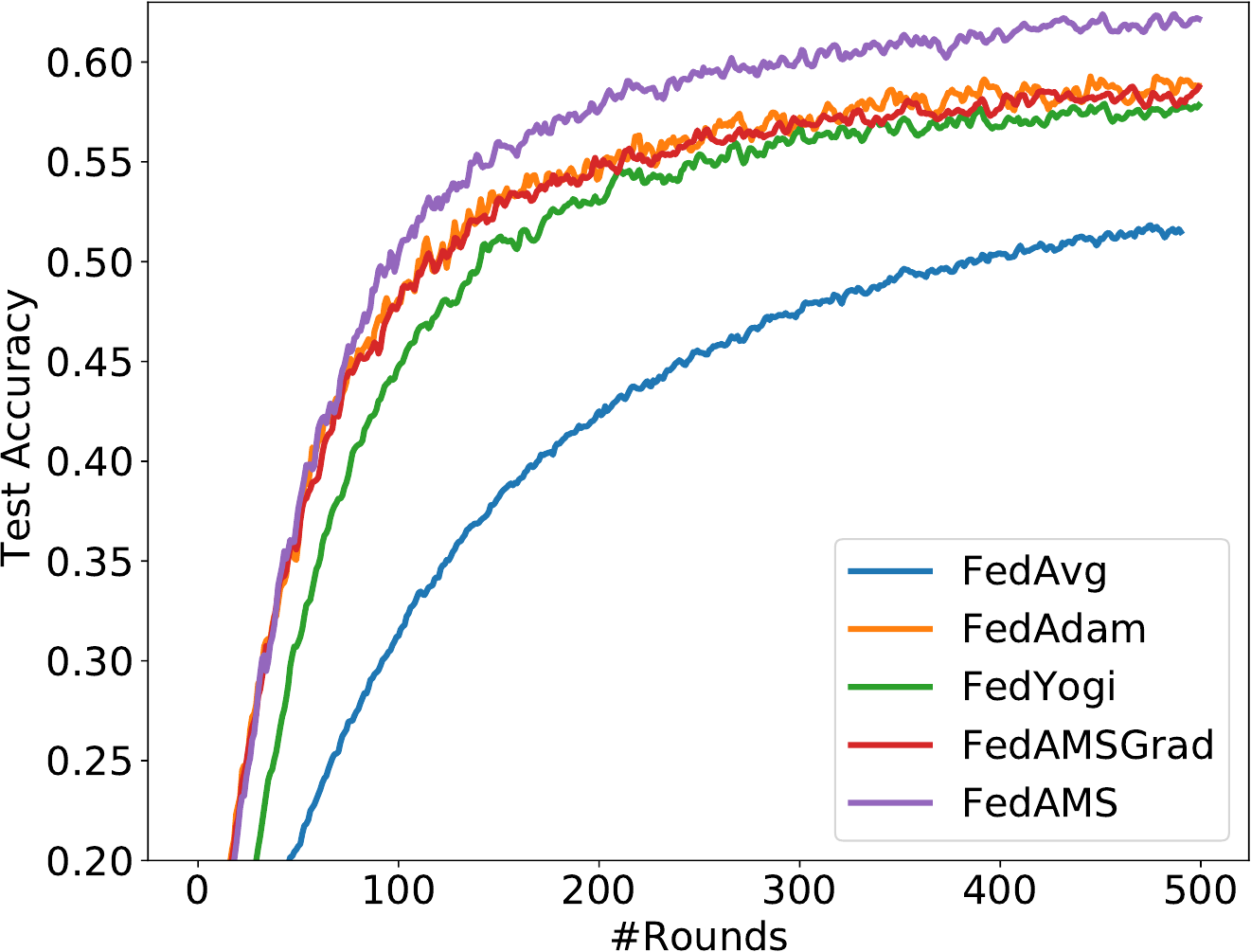}}
\caption{The learning curves for FedAMS and other federated learning baselines on training CIFAR-100 data (a)(b) show the results for ResNet-18 model and (c)(d) show the results for ConvMixer-256-8 model. }
\label{fig:cifar100}
\vspace{-10pt}
\end{figure}

For the ConvMixer-256-8 model, which is typically trained via adaptive gradient method, we observe that all adaptive federated optimization methods (FedAdam, FedYogi, FedAMSGrad and FedAMS) achieve much better performance in terms of both training loss and test accuracy compared with FedAvg. In details, FedAMS again achieves a significantly better result than other baselines in terms of training loss and test accuracy. Other adaptive methods, including FedAdam, FedYogi, and FedAMSGrad, have similar convergence behaviour when training the ConvMixer-256-8 model. Such results empirically show the effectiveness of our proposed FedAMS method.

\subsection{Additional Ablation Study}

\noindent\textbf{The ablation on $\epsilon$:} We conduct an ablation study with $\epsilon \in \{10^{-1}, 10^{-2}, 10^{-3}, 10^{-4}, 10^{-6}, 10^{-8}\}$ on CIFAR-10 in Table \ref{tab:abla-eps} and our $\epsilon$ value in experiments is chosen by its relatively higher test accuracy. 

\begin{table}[ht!]
    \centering
    \begin{tabular}{c|cccccc}
    \toprule
        $\epsilon$ & $10^{-1}$ & $10^{-2}$ & $10^{-3}$ & $10^{-4}$ & $10^{-6}$ & $10^{-8}$ \\
        \midrule
        Test acc (\%) & 90.45 & 90.51 & \textbf{90.94} & 90.72 & 90.49 & 90.30 \\
        \bottomrule
    \end{tabular}
    \caption{Ablation study on $\epsilon$.}
    \label{tab:abla-eps}
\end{table}

\end{document}